\documentclass[11pt]{article}
\usepackage{amssymb}
\usepackage{amsmath,amsfonts,amsthm}
\usepackage{subcaption}
\usepackage{multirow}
\usepackage{graphicx}
\usepackage{float}
\usepackage{natbib}
\usepackage{bbm}
\usepackage{amssymb}
\usepackage{natbib}
\usepackage{booktabs}
\usepackage{url}
\usepackage[colorlinks,linkcolor=red,citecolor=blue,urlcolor=magenta]{hyperref}
\usepackage{mathtools}  
\mathtoolsset{showonlyrefs=true}

\usepackage{algorithm}
\usepackage{algorithmic}

\usepackage[mathscr]{eucal}
\usepackage{mathrsfs}

\newcommand{\be}{\begin{equation}}
\newcommand{\ee}{\end{equation}}
\newtheorem{lemma}{Lemma}
\newtheorem{theorem}{Theorem}

\newtheorem{definition}{Definition}

\newtheorem{proposition}{Proposition}

\newtheorem{assumption}{Assumption}

\newtheorem{remark}{Remark}

\newtheorem{assumptionalt}{Assumption}[assumption]
\newenvironment{assumptionp}[1]{
	\renewcommand\theassumptionalt{#1}
	\assumptionalt
}{\endassumptionalt}

\usepackage{enumerate}
\usepackage{amsmath}
\newcommand{\bm}[1]{\boldsymbol{#1}}

\usepackage{natbib}
\bibpunct[, ]{(}{)}{,}{a}{}{,}%
\def\bibsep{\smallskipamount}%
\newcommand{\bProof}{\begin{proof}{Proof.}}
\newcommand{\eProof}{\hfill\Halmos\\  \end{proof}}

\newcommand{\bpi}{\boldsymbol{\pi}}
\newcommand{\bPi}{\boldsymbol{\Pi}}

\newcommand{\bea}{\begin{equation*}}
\newcommand{\eea}{\end{equation*}}

\DeclareMathOperator*{\argmax}{arg\,max}

\usepackage[usenames,dvipsnames]{xcolor}

\newcommand{\rev}[1]{\textcolor{black}{{#1}}} 

\usepackage{comment}

\allowdisplaybreaks[4]

\newcommand{\gao}[1]{\textcolor{black}{{#1}}}

\title{Reinforcement Learning for Jump-Diffusions, with Financial Applications}

\author{
	Xuefeng Gao\thanks{Department of Systems Engineering and Engineering Management, The Chinese University of Hong Kong, Hong Kong, China. E-mail: xfgao@se.cuhk.edu.hk}
	\and
	Lingfei Li\thanks{Department of Systems Engineering and Engineering Management, The Chinese University of Hong Kong, Hong Kong, China. E-mail: lfli@se.cuhk.edu.hk}
	\and
	Xun Yu Zhou\thanks{Department of Industrial Engineering and Operations Research and The Data Science Institute, Columbia University, New York, NY 10027, USA. Email: xz2574@columbia.edu}
}

\begin{document}
\maketitle

\begin{abstract}
We study continuous-time reinforcement learning (RL) for stochastic control in which system dynamics are governed by jump-diffusion processes. We formulate an entropy-regularized exploratory control problem with stochastic policies to capture
the exploration--exploitation balance essential for RL. Unlike the pure diffusion case initially studied by \cite{wang2020reinforcement}, the derivation of the exploratory dynamics under jump-diffusions calls for a careful formulation of the jump part. Through a theoretical analysis, we find that one can simply use the same policy evaluation and q-learning algorithms
in \cite{jia2022policy, jia2022q}, originally developed for controlled diffusions,
without needing to check a priori whether the underlying data come from a pure diffusion
or a jump-diffusion. However, we show that the presence of jumps ought to affect parameterizations of actors and critics in general. We investigate as an application the mean--variance portfolio selection problem with stock price modelled as a jump-diffusion, and show that both RL algorithms and parameterizations are invariant with respect to jumps. Finally, we present a detailed study on applying the general theory to option hedging.

\medskip


\end{abstract}

\section{Introduction}

Recently there is an upsurge of interest in continuous-time reinforcement learning (RL) with continuous state spaces and possibly continuous action spaces. Continuous RL problems are important because: 1) many if not most practical problems are naturally continuous in time (and in space), such as autonomous driving, robot navigation,  video game play and high frequency trading; 2) while one can discretize time upfront and turn a continuous-time problem into a discrete-time MDP, it has been known, indeed shown experimentally in e.g., \cite{munos2006policy}, \cite{tallec2019making} and \cite{park2021time},  that this approach is very sensitive to time discretization and performs poorly with small time steps; 3) there are more analytical tools available for the continuous setting that enable a rigorous and thorough analysis leading to interpretable (instead of black-box) and general (instead of ad hoc) RL algorithms.

Compared with the vast literature of RL for MDPs, continuous-time RL research is still in its infancy with the latest study focusing on establishing a rigorous mathematical theory and devising resulting RL algorithms. This strand of research starts with \cite{wang2020reinforcement} that introduces  a mathematical formulation to capture the essence of RL -- the exploration--exploitation tradeoff -- in the continuous setting, followed by a ``trilogy" (\citealt{jia2022policy, jia2022policyb, jia2022q}) that develops intertwining theories on policy evaluation, policy gradient and q-learning, respectively. The common underpinning of the entire theory is the martingale property of certain stochastic processes, the enforcement of which naturally leads to various temporal difference algorithms to train and learn q-functions, value functions and optimal (stochastic) policies. The research is characterized by carrying out all the analysis in the continuous setting, and discretizing time only at the final, implementation stage for approximating the integrated rewards and the temporal difference. The theory has been adapted and extended in different directions; see e.g. \cite{RZ2021}, \cite{guo2022entropy}, \cite{dai2023learning}, as well as employed for applications; see e.g. \cite{wang2020continuous}, \cite{huang2022achieving}, \cite{gao2022state}, \cite{wang2023reinforcement}, and \cite{wu2024reinforcement}.

The study so far has been predominantly on pure diffusion processes, namely the state processes are governed by controlled stochastic differential equations (SDEs) with a drift part and a diffusion one. While it is reasonable to model the underlying data generating processes as diffusions within a short period of time, sudden and drastic changes can and do happen
over time. An example is a stock price process: Although it is approximately a diffusion over a sufficiently short period, it may respond dramatically to a surprisingly good or bad earning report. Other examples include neuron dynamics (\citealt{GS1997}),
stochastic resonance (\citealt{GHJM1998})
and climate data (\citealt{Goswami2018}).
It is therefore natural and necessary to extend the continuous RL theory and algorithms to the case when jumps are present.
This is particularly important for decision makings in financial markets, where it has been well recognized that using jumps to capture large sudden movements provides a more realistic way to model market dynamics; see the discussions in Chapter 1 of \cite{tankov2004financial}. The financial modeling literature with jumps dates back to the seminal work of \cite{merton1976option}, who extends the classical Black--Scholes model by introducing a compound Poisson process with normally distributed jumps in the log returns. Since then alternative jump size distributions have been proposed in e.g. \cite{kou2002jump} and \cite{cai2011option}. Empirical success of jump-diffusion models has been documented for many asset classes; see \cite{bates1991crash}, \cite{andersen2002empirical}, and \cite{ait2012testing} for stocks and stock indices, \cite{bates1996jumps} for exchange rates, \cite{das2002surprise} for interest rates, and \cite{li2014time} for commodities, among many others.


This paper makes two major contributions. {The first is several mathematical results that form the foundation of the RL theory for jump-diffusions, including a suitable exploratory formulation, a well-posedness proof of the resulting exploratory SDE, and the convergence of the value functions of the grid sample state processes to the exploratory value function.} \cite{wang2020reinforcement} apply the classical stochastic relaxed control to model the exploration or randomization prevalent in RL, and derive an exploratory state equation that dictates the dynamics of the ``average" of infinitely many state processes generated by repeatedly sampling from the same exploratory, stochastic policy. The drift and variance coefficients of the exploratory SDE are the means of those coefficients against the given stochastic policy (which is a probability distribution) respectively. The derivation therein is based on a law of large
number (LLN) argument to the first two moments of the diffusion process. That argument fails for jump-diffusions which are not uniquely determined by the first two moments.
We overcome this difficulty by analyzing instead the infinitesimal behavior of the grid sample state process, which is obtained by sampling the stochastic policy on a discrete time grid. Based on this analysis, we identify the dynamic of the exploratory state process. Inspired by
\cite{kushner2000jump} who studies relaxed control for jump-diffusions, we formulate the exploratory SDE by extending the original Poisson random measures for jumps to capture the effect of random exploration. It should be noted that, like almost all the earlier works on relaxed control,  \cite{kushner2000jump} is motivated by answering  the theoretical question of whether an optimal control exists, as randomization convexifies the universe of control strategies. In comparison, our formulation is guided by the practical motivation of exploration for learning. There is also another subtle but important difference. We consider stochastic \emph{feedback} policies while \cite{kushner2000jump} does not. This in turn creates technical issues in studying the well-posedness of the exploratory SDE in our framework. Although the exploratory state process plays an important role in theoretical analysis, it is unobservable and hence cannot be used in the implementation of any RL algorithm. Instead, we consider the grid sample state processes, which is what one obtains from the environment by following the stochastic policy on a discrete time grid. We prove that the value functions of grid sample state processes converge to that of the exploratory state process at first order as the grid size shrinks to zero, which extends a major result in \cite{jia2025accuracy} from diffusions to jump-diffusions.

The second main contribution is several implications regarding the impact of jumps on RL algorithm design.
Thanks to the established exploratory formulation, we can define the Hamiltonian that, compared with the pure diffusion counterpart,  has to include an additional term corresponding to the jumps. The resulting HJB equation -- called the exploratory HJB -- is now a partial integro-differential equation (PIDE) instead of a PDE due to that additional term. However, when expressed {\it in terms of the Hamiltonian}, the exploratory HJB equation has exactly the same form as that in the diffusion case. This leads to several completely identical {\it statements} of important results, including the optimality of the Gibbs exploration, definition of a q-function, and martingale characterizations of value functions and q-functions. Here by ``identical" we mean in terms of the Hamiltonian; in other words, these statements differ between diffusions and jump-diffusions entirely because the Hamiltonian is defined differently (which also causes some differences in the proofs of the results concerned). Most important of all, in the resulting RL algorithms, the Hamiltonian (or equivalently the q-function) can be computed using temporal difference of the value function by virtue of the It\^o lemma; as a result the algorithms are completely identical no matter whether or not there are jumps. This has a significant practical implication: we can just use the same RL algorithms  without the need of checking in advance whether the underlying data come from a pure diffusion or a jump-diffusion. It is significant for the following reason. In practice, data are always observed or sampled at discrete times, no matter how frequent they arrive. Thus, we encounter successive discontinuities along the sample trajectory even when the data actually come from a diffusion process. There are some criteria that can be used to check whether the underlying process is a diffusion or a jump-diffusion, e.g. \cite{ait2012testing, wang2022testing}. But these methods typically require data with very high frequency to be effective, which may not always be available. In addition, noises must be properly handled for them to work.

Even though we can apply the same RL algorithms irrespective of the presence of jumps, the {\it parametrization} of the policy and value function may still depend on it, if we try to exploit certain special structure of the problem instead of using general neural networks for parameterization. Indeed, we give an example in which the optimal exploratory policy is Gaussian when there are no jumps, whereas an optimal policy either does not exist or becomes non-Gaussian when there are jumps. However, in the mean--variance (MV) portfolio selection problem we present as a concrete application, the optimal Gibbs exploration measure again reduces to Gaussian and the value function is quadratic as in \cite{wang2020continuous}, both owing to the inherent linear--quadratic (LQ) structure of the problem. Hence in this particular case jumps do not even affect the parametrization of the policy and value function/q-function for learning.

We also consider MV hedging of options as another application. This is a non-LQ problem and hence more difficult to solve than MV portfolio selection. The MV hedging problem has been studied in various early works, such as \cite{schweizer1996approximation,schweizer2001guided} and \cite{lim2005mean}. Here, we introduce the entropy-regularized MV hedging objective for an asset following a jump-diffusion and derive analytical representations for the optimal stochastic policy, which is again Gaussian, as well as the optimal value function. We use these representations to devise an actor--critic algorithm to learn the optimal hedging policy from data.

We compare our work with five recent related papers. (1) \cite{bender2023entropy} consider the continuous-time MV portfolio selection problem with exploration under a jump-diffusion setting. {Our paper differs from theirs in two main aspects. First, they consider the specific MV problem only while we study RL for general controlled jump-diffusions with MV as an instance of application.
Second, they do not develop algorithms based on their solution of the exploratory MV portfolio selection, which we do in this paper.} (2) {During the revision process of our paper, we became aware of the concurrent study \cite{bender2024continuous}, which independently investigates RL for jump-diffusions using a randon measure approach. Specifically, they derive a grid sample state process as an equation driven by suitable random measures and establish a limit theorem for these measures as the sampling grid size approaches zero. By contrast, our work focuses on developing q-learning theory and algorithms for jump-diffusions, with a particular emphasis on applying this general framework to two financial applications.} (3) \cite{bo2024continuous} develop q-learning for jump-diffusions by using Tsallis' entropy for regularization instead of  Shannon's entropy considered in our paper and \cite{jia2022policyb,jia2022q}.\footnote{The paper \cite{bo2024continuous} came to our attention after a previous version of our paper was completed and posted.} While this entropy presents an interesting alternative for developing RL algorithms, it may make the exploratory control problem less tractable to solve and lead to policy distributions that are inefficient to sample for exploration in certain applications. (4) \cite{guo2023reinforcement} consider continuous-time RL for linear--convex models with jumps. The scope and motivation are different from ours: They focus on the  Lipschitz stability of feedback controls for this special class of control problems where the diffusion and jump terms are not controlled, and propose a least-square model-based algorithm and obtain sublinear regret guarantees in the episodic setting. By contrast, we consider RL for general jump-diffusions and develop model-free algorithms without considering regret bounds. (5) \cite{denkert2024control} aim to unify certain types of stochastic control problems by considering the so-called randomized control formulation which leads to the same optimal value functions as those of the original problems. They develop a policy gradient representation and actor--critic algorithms for RL. The randomized control formulation is fundamentally different from the framework we are considering: therein the control is applied at a set of random time points generated by a random point process.


The remainder of the paper is organized as follows. In Section~\ref{sec:formulation}, we discuss the formulation of the control problem and the exploratory state process. In Section~\ref{sec:qlearn}, we present the theory of q-learning for jump-diffusions, followed by the discussion of q-learning algorithms in Section~\ref{sec:q-algo}. In Section~\ref{sec:mean-var}, we apply the general theory and algorithms to a mean--variance portfolio selection problem, and discuss the impact of jumps. Section \ref{sec:MV-hedge} presents the application to a mean--variance option hedging problem.
Finally, Section~\ref{sec:conclusion} concludes. The appendix contains the proofs and some results regarding the convergence of the value functions of the grid sample state processes to the value function of the exploratory state process.

\section{Problem Formulation and Preliminaries}\label{sec:formulation}
For readers' convenience, we first recall some basic concepts for one-dimensional (1D) L\'evy processes, which can be found in standard references such as \cite{sato1999levy} and \cite{applebaum2009levy}.
A 1D process $L=\{L_t: t \ge 0 \}$ is a L{\'e}vy process if it is continuous in probability, has stationary and independent increments, and $L_0=0$ almost surely. Denote the jump of $L$ at time $t$ by $\Delta L_t = L_t - L_{t-}$, and let $\pmb{\text{B}}_0$ be the collection of Borel sets of $\mathbb{R}$ whose closure does not contain $0$. The Poisson random measure (or jump measure) of $L$ is defined as
\begin{align}
	N(t, B) = \sum_{s: 0 < s \le t} 1_{B} (\Delta L_s),\ B\in \pmb{\text{B}}_0,
\end{align}
which gives the number of jumps up to time $t$ with jump size in a Borel set $B$ away from $0$. The L{\'e}vy measure $\nu$ of $L$ is defined by $\nu (B) = \mathbb{E} [ N(1, B)]$ for $B\in \pmb{\text{B}}_0$, which shows the expected number of jumps in $B$ in unit time, and $\nu (B)$ is finite. For any $B\in \pmb{\text{B}}_0$, $\{N(t, B): t\ge 0\}$ is a Poisson process with intensity given by $\nu(B)$. The differential forms of these two measures are written as $N(dt,dz)$ and $\nu(dz)$, respectively. If $\nu$ is absolutely continuous, we write $\nu(dz)=\nu(z)dz$ by using the same letter for the measure and its density function. The L\'evy measure $\nu$ must satisfy the integrability condition
\begin{equation}\label{eq:levy-integ-cond}
\int_{\mathbb{R}} \min\{z^2, 1\}\nu(dz)<\infty.
\end{equation}
However, it is not necessarily a finite measure on $\mathbb{R}/\{0\}$ but always a $\sigma$-finite measure. The L\'evy process $L$ is said to have jumps of finite (infinite activity) if $\int_{\mathbb{R}}\nu(dz)<\infty$ ($=\infty$). The number of jumps on any finite time interval is finite in the former case but infinite in the latter. For any Borel set $B$ with its closure including $0$,
$\nu(B)$ is finite in the finite activity case but infinite otherwise.
Finally, the compensated Poisson random measure is defined as $\widetilde{N}(dt, dz)=N(dt, dz)-\nu(dz)dt$. For any $B\in \pmb{\text{B}}_0$, the process $\{\widetilde{N}(t,B): t\ge 0\}$ is a martingale.

Throughout the paper, we use the following notations. By convention, all vectors are column vectors unless otherwise specified. We use $\mathbb{R}^k$ and $\mathbb{R}^{k\times \ell}$ to denote the space of $k$-dimensional vectors and $k\times \ell$ matrices, respectively. For matrix $A$, we use $A^\top$ for its transpose, $|A|$ for its Euclidean/Frobenius norm, and write $A^2\coloneqq AA^\top$. Given two matrices $A$ and $B$ of the same size, we denote by $A\circ B$ the inner product between $A$ and $B$, which is given by $\text{tr}(AB^\top)$. For a positive semidefinite matrix $A$, we write $\sqrt{A}=UD^{1/2}V^\top$, where $A=UDV^\top$ is its singular value decomposition with $U,V$ as two orthogonal matrices and $D$ as a diagonal matrix, and $D^{1/2}$ is the diagonal matrix whose entries are the square root of those of $D$.  We use $f=f(\cdot)$ to denote the function $f$, and $f(x)$ to denote the function value of $f$ at $x$. We use  $\partial_xf, \partial_{x}^2f$ for the first and second partial derivatives of a function $f$ with respect to $x$. We write the minimum of two values $a$ and $b$ as $a\wedge b$. The notation $\mathcal{U}(B)$ denotes the uniform distribution over set $B$ while $\mathcal{N}(\mu,\Sigma)$ refers to the Gaussian distribution with mean vector $\mu$ and covariance matrix $\Sigma$.

\subsection{Classical stochastic control of jump-diffusions}
Consider a filtered probability space $(\Omega,\mathcal{F},\{\mathcal{F}_t\},\mathbb{P})$ satisfying the usual hypothesis. Assume that this space is rich enough to support $W = \{W_t : t \ge 0\}$, a standard Brownian motion in $\mathbb{R}^m$, and $\ell$ independent one-dimensional (1D) L\'evy processes $L_1,\cdots,L_{\ell}$, which are also independent of $W$. Let $N(dt,dz)=(N_1(dt,dz_1),\cdots,N_{\ell}(dt,dz_\ell))^\top$ be the vector of their Poisson random measures, and similarly define $\nu(dz)$ and $\widetilde{N}(dt,dz)$. The controlled system dynamics are governed by the following L{\'e}vy SDE \cite[Chapter 3]{oksendal2007applied}:
\begin{align} \label{eq:LevySDE}
dX_s^a = b(s, X_{s-}^a, a_s)ds + \sigma(s, X_{s-}^a, a_s)dW_s + \int_{\mathbb{R}^\ell}\gamma(s, X_{s-}^a, a_{s}, z) \widetilde{N}(ds, dz),\ s \in [0, T],
\end{align}
where
\begin{equation}
	b: [0,T]\times \mathbb{R}^d \times \mathcal{A} \rightarrow \mathbb{R}^d,\ \sigma: [0,T]\times \mathbb{R}^d \times \mathcal{A} \rightarrow \mathbb{R}^{d \times m} \ \text { and } \ \gamma: [0,T]\times \mathbb{R}^d \times \mathcal{A} \times \mathbb{R}^\ell \rightarrow \mathbb{R}^{d \times \ell},
\end{equation}
$a_s$ is the control or action at time $s$, $\mathcal{A}\subseteq\mathbb{R}^n$ is the control space, and $a=\{a_s: s\in[0,T]\}$ is the control process assumed to be predictable with respect to $\{\mathcal{F}_s: s\in [0,T]\}$. We denote the $k$-th column of the matrix $\gamma$ by $\gamma_k$.
The goal of stochastic control is, for each initial time-state pair $(t, x)$ of \eqref{eq:LevySDE}, to find the optimal control process $a$ that maximizes the expected total reward:
\begin{equation} \label{eq:FH-control}
 \mathbb{E} \left[ \int_t^T e^{-\beta(s-t) }r(s, X_s^a, a_s) ds + e^{-\beta(T-t) } h (X_T^a) \Big| X_t^a =x \right],
\end{equation}
where $\beta \ge 0$ is a discount factor that measures the time value of the payoff.

The stochastic control problem \eqref{eq:LevySDE}--\eqref{eq:FH-control} is very general; in particular, control processes affect the drift, diffusion and jump coefficients.
We now make the following assumption to ensure well-posedness of the problem. Define $\mathbb{R}^d_K\coloneqq\{x\in\mathbb{R}^d: |x|\leq K\}$.

\begin{assumption}\label{assump:SDE}
Suppose the following conditions are satisfied by the state dynamics and reward functions:
\begin{enumerate}[(i)]
	\item $b, \sigma, \gamma, r, h$ are all continuous functions in their respective arguments;
	
	\item (local Lipschitz continuity) for any $K>0$ and any $p\geq 2$, there exist positive constants $C_K$ and $C_{K,p}$ such that $\forall(t, a) \in[0, T] \times \mathcal{A}, (x,x')\in\mathbb{R}^d_K$,
	\begin{align}
		&|b(t,x,a)-b(t,x',a)|+|\sigma(t,x,a)-\sigma(t,x',a)|\le C_K|x-x'|,\\
		&\sum_{k=1}^\ell \int_{\mathbb{R}}|\gamma_k(t,x,a,z_k)-\gamma_k(t,x',a,z_k)|^p\nu_k(dz)\le C_{K,p}|x-x'|^p;
	\end{align}
	
	\item (linear growth in $x$) for any $p\geq 1$, there exist positive constants $C$ and $C_p$ such that $\forall(t, x, a) \in[0, T] \times \mathbb{R}^d \times \mathcal{A}$,
	\begin{align}
		&|b(t,x,a)|+|\sigma(t,x,a)|\le C(1+|x|+{|a|}),\\
		&\sum_{k=1}^\ell \int_{\mathbb{R}}|\gamma_k(t,x,a,z)|^p\nu_k(dz)\le C_p(1+|x|^p+{|a|^p});
	\end{align}

	\item there exists a constant $C>0$ such that
	\begin{equation}
		|r(t, x, a)| \le C\left(1+|x|^{p}+|a|^{q}\right), \ |h(x)| \le C\left(1+|x|^{p}\right),\ \forall(t, x, a) \in[0, T] \times \mathbb{R}^d \times \mathcal{A}
	\end{equation}
	for some $p\ge 2$ and some $q\ge 1$.
\end{enumerate}
\end{assumption}

A control process $a$ is admissible if $\mathbb{E}[\int_0^T|a_s|^{p} ds]$ is finite for any $p\geq 1$. Conditions (i)-(iii) guarantee the existence of a unique strong solution to the controlled L\'evy SDE \eqref{eq:LevySDE} with initial condition $X_t^a=x\in\mathbb{R}^d$ for any admissible control process. Furthermore, for any $p\ge 2$, there exists a constant $C_p>0$ such that
\begin{equation}\label{eq:moment}
	\mathbb{E}_{t,x}\Big[\sup_{t\le s\le T} |X_s^a|^{p}\Big]\le C_p(1+|x|^{p});
\end{equation}
see \cite[Theorem 3.2]{kunita2004stochastic} and \cite[Theorem 119]{rong2006theory}. With the moment estimate \eqref{eq:moment}, it follows from condition (iv) and the admissibility condition that the expected value in \eqref{eq:FH-control} is finite.

Let $\mathcal{L}^a$ be the infinitesimal generator associated with the L{\'e}vy SDE \eqref{eq:LevySDE}. Under condition (iii), we have $	\int_{\mathbb{R}}|\gamma_k(t,x,a,z)|\nu_k(dz)<\infty$ for $k=1,\cdots,\ell$.
Thus, we can write $\mathcal{L}^a$ in the following form:
\begin{align}\label{eq:gene1}
	\mathcal{L}^a f(t,x) &= \partial_t f(t,x) + b(t, x, a) \circ \partial_x f  (t,x) + \frac{1}{2} \sigma^2(t,x,a) \circ \partial^2_x f(t,x) \nonumber \\
	&+ \sum_{k=1}^\ell \int_{\mathbb{R}} \left( f(t, x+ \gamma_k(t, x ,a, z)) - f(t,x) - \gamma_k(t,x,a,z) \circ \partial_xf(t,x) \right) \nu_k (dz),
\end{align}
where $\partial_xf\in\mathbb{R}^d$ is the gradient and $\partial^2_x f\in\mathbb{R}^{d\times d}$ is the Hessian matrix.

We recall It\^o's formula, which will be frequently used in our analysis; see e.g. \cite[Theorem 1.16]{oksendal2007applied}. Let $X^a$ be the unique strong solution to  \eqref{eq:LevySDE}. For any $f\in C^{1,2}(\mathbb{R}^+\times\mathbb{R}^d)$, we have
\begin{align}
	&df(t,X^a_t)=\partial_t f(t,X^a_{t-})dt + b(t,X^a_{t-},a_t)\circ \partial_x f(t,X^a_{t-})dt + \frac{1}{2}\sigma^2(t,X^a_{t-},a_t) \circ\partial_x^2 f(t,X^a_{t-})dt\\
	+&\sum_{k=1}^\ell \int_{\mathbb{R}} \left( f(t, X^a_{t-}+ \gamma_k(t, X^a_{t-}, a_{t}, z)) - f(t,X^a_{t-}) - \gamma_k(t,X^a_{t-}, a_{t}, z) \circ  \partial_x f(t,X^a_{t-}) \right) \nu_k (dz)dt\\
	+&\partial_x f(t, X^a_{t-})\circ \sigma(t,X^a_{t-},a_t)dW_t+\sum_{k=1}^\ell \int_{\mathbb{R}} \left( f(t, X^a_{t-}+ \gamma_k(t, X^a_{t-}, a_{t}, z)) - f(t,X^a_{t-})\right) \widetilde{N}_k (dt,dz).\label{eq:Ito-JD}
\end{align}

It is known that the Hamilton--Jacobi--Bellman (HJB)  equation for the control problem \eqref{eq:LevySDE}--\eqref{eq:FH-control} is given by
\begin{align} \label{eq:hjb-pde1}
	\sup_{a \in \mathcal{A}} \{ r(t, x,a) + \mathcal{L}^a V(t,x) \}  - \beta V(t,x) & =0, \quad (t,x) \in [0, T)\times\mathbb{R}^d,\\
	V(T,x) &=h (x), \nonumber
\end{align}
where $\mathcal{L}^a$ is given in \eqref{eq:gene1}.
Under proper conditions, the solution to the above equation is the optimal value function $V^*$ for control problem \eqref{eq:FH-control}. Moreover, the following function, which maps a time--state pair to an action:
\begin{align}\label{vt}
	a^*(t,x) = \argmax_{a \in \mathcal{A}} \left\{ r(t, x,a) + \mathcal{L}^a V^*(t,x)  \right\}
\end{align}
is the optimal feedback control policy of the problem, a result called the verification theorem.

Given a smooth function $V(t,x) \in C^{1,2}([0, T] \times \mathbb{R}^d),$
we define the Hamiltonian $H$ by
\begin{align}
	&H\big(t,x, a, \partial_x V, \partial_x^2 V, V\big)=r(t, x, a)  +  b(t, x, a) \circ \partial_x V(t,x) + \frac{1}{2} \sigma^2(t,x,a) \circ \partial^2_x V(t,x) \\
	&\phantom{\quad\quad}+\sum_{k=1}^\ell \int_{\mathbb{R}^d} \left( V(t, x+  \gamma_k(t, x ,a, z)) - V(t,x) - \gamma_k(t,x,a,z) \circ  \partial_xV(t,x) \right) \nu_k (dz).\label{eq:Hal-1}
\end{align}
Then, the HJB equation \eqref{eq:hjb-pde1} can be recast as
\begin{align} \label{eq:hjb-pde2}
	\partial_t V(t, x) + \sup_{a \in \mathcal{A}} \{ H(t, x, a, \partial_xV, \partial_x^2V, V)   \} - \beta V(t, x) &= 0, \quad (t,x) \in [0, T)\times \mathbb{R}^d,\\
	V(T, x) &= h(x).\nonumber
\end{align}

What we have just outlined above is the classical {\it model-based} approach to solving the stochastic control problem \eqref{eq:LevySDE}--\eqref{eq:FH-control}, under the assumption that the model primitives $b, \sigma, \gamma, r, h$ are known functions. When they are not known, there is no HJB equation \eqref{eq:hjb-pde1} to solve and no verification condition \eqref{vt} to apply, in which case we   resort to RL.

\subsection{Randomized control and the grid sample state process}\label{sec:expl-formulation}
A key idea of RL is to explore the unknown environment by randomizing the actions. Let $\bpi: (t,x) \in [0,T] \times \mathbb{R}^d \rightarrow  \bpi (\cdot| t,x) \in \mathcal{P} (\mathcal{A})$ be a given {\it stochastic} feedback policy, where $\mathcal{P} (\mathcal{A})$ is the set of probability density functions defined on $\mathcal{A}$.  \gao{ An action $a\in \mathcal{A}$ sampled from $\bpi (\cdot| t,x)$ can be expressed as $a = G^{\bpi}(t,x,U)$ for some measurable function $G^{\bpi}$ and some $U$ following the uniform distribution over $[0,1]^n$ ($n$ is the dimension of the control space). In the following, we will also interchangeably use the more compact notations $\bpi_{t,x}(\cdot)$ for $\bpi (\cdot| t,x)$ and  $G^{\bpi_{t,x}}(\cdot)$ for $G^{\bpi}(t,x,U)$. }

When it comes to randomizing actions in {\it continuous} time, a subtle measurability issue arises from continuum (uncountable in particular) and independent samplings from the stochastic policy; see, e.g. \cite{szpruch2024optimal} and \cite{bender2024continuous} for detailed discussions. To avoid such issues, we follow \cite{jia2025accuracy} to only consider randomization on a time grid and use piecewise constant actions.

Let $\mathbb{S}= \{0 = t_0 < t_1 < \ldots < t_N = T\}$ be a time grid of $[0, T]$ with $N \in \mathbb{N}$ and $|\mathbb{S}| = \max_{i=1, \ldots,N}(t_i - t_{i-1})$. Given a stochastic feedback policy $\bpi$, we follow a {\it grid sampling scheme} where action randomization occurs at the grid points only. Specifically, to define the state dynamic under such a scheme, we introduce the grid sampling process:
\begin{align} \label{eq:grid-sampling}
U_s^{\mathbb{S}} = U_1  \cdot 1_{[0, t_1]}(s) + \sum_{i=2}^N U_i \cdot 1_{(t_{i-1}, t_i]}(s), \quad s \in [0,T].
\end{align}
Here $(U_i)$ are i.i.d random vectors uniformly distributed over $[0,1]^{n}$, which are independent of $W$ and $L_1,\cdots,L_\ell$, assuming that the original probability space $(\Omega, \mathcal{F}, \mathbb{P})$ is rich enough to support these random elements. This process essentially describes sampling from $U_i$ at time $t_{i-1}$. We now enlarge the original filtration to include the additional randomness from sampling actions. Denote by $\mathcal{G}^{\mathbb{S}} = (\mathcal{G}^{\mathbb{S}}_t)_{t \in [0, T]}$ the right-continuous, augmented version of the filtration generated by $W$, $L_1,\cdots,L_\ell$ and the grid sampling process $U^{\mathbb{S}}$.  The new filtered probability space is now $(\Omega,\mathcal{F}, \mathcal{G}^{\mathbb{S}}, \mathbb{P})$.


Fix an initial system state $x_0$ at time 0, a time grid $\mathbb{S}$ of $[0,T]$ and a stochastic feedback policy $\bpi$. Under a grid sampling scheme, the state dynamic $X^{\bpi, \mathbb{S}} = \{X^{\bpi, \mathbb{S}}_s: 0 \le s \le T\}$, which will be henceforth referred to as the \emph{grid sample state process}, satisfies the following SDE: for all $i = 0, . . . , n - 1$ and all $s \in (t_{i}, t_{i+1}]$,
\begin{align}
d X^{\bpi, \mathbb{S}}_s &= b(s, X^{\bpi, \mathbb{S}}_{s-} ,  G^{\bpi}(t_i, X^{\bpi, \mathbb{S}}_{t_i}, U_s^{\mathbb{S}} ) ) ds + \sigma (s, X^{\bpi, \mathbb{S}}_{s-}, G^{\bpi}(t_i, X^{\bpi, \mathbb{S}}_{t_i}, U_s^{\mathbb{S}}) ) dW_s \\
& \quad + \int_{\mathbb{R}^\ell} \gamma (s, X^{\bpi, \mathbb{S}}_{s-}, G^{\bpi}(t_i, X^{\bpi, \mathbb{S}}_{t_i}, U_s^{\mathbb{S}} ), z) \widetilde N(ds, dz), \label{eq:samplestate}
\end{align}
where $U_s^{\mathbb{S}}=U_{i+1}$ for $s \in (t_{i}, t_{i+1}]$. This means a random action $G^{\bpi}(t_i, X^{\bpi, \mathbb{S}}_{t_i}, U_{i+1})$ is generated at $t_i$ and remains unchanged over $(t_{i}, t_{i+1}]$.\footnote{To define the grid sample state process, one can alternatively consider an action process given by $G^{\bpi}(s, X^{\bpi, \mathbb{S}}_{s-}, U_s^{\mathbb{S}}) = G^{\bpi}(s, X^{\bpi, \mathbb{S}}_{s-}, U_{i+1})$ for $s \in (t_{i}, t_{i+1}]$ as in \cite{bender2024continuous}. In this paper we choose to work with the piecewise constant action process given in \eqref{eq:samplestate} because it is simpler and indeed is the actual {\it implemented} control process in our RL algorithms.} Under appropriate technical assumptions on $\bpi$, we will show later that the SDE \eqref{eq:samplestate} admits a unique strong solution that has a desirable moment estimate; see Proposition~\ref{prop:SDE-sampled} for details.

\subsection{Exploratory state process}
For theoretical analysis, we need to formulate the so-called exploratory state process, which represents the averaged controlled dynamic over infinitely many randomized actions. In the case of diffusions, \cite{wang2020reinforcement} \rev{heuristically formulated such exploratory dynamics by applying an LLN (Law of Large Numbers) argument to the first two moments of the diffusion processes. We will still use the same heuristic approach based on LLN here, but the first two moments cannot determine a jump-diffusion. Instead, we study the infinitesimal behavior of the grid sample state process by averaging randomized actions, based on which we will formulate the exploratory state process.}

To this end, let $f\in C_0^{1,2}([0,T)\times\mathbb{R}^d)$ be  continuously differentiable in $t$ and twice continuously differentiable in $x$ with compact support. \rev{Let $t=t_k$, for some $k$, be a grid point of $\mathbb{S}$ and fix $x$. Consider the SDE \eqref{eq:samplestate} starting from $X^{\bpi, \mathbb{S}}_t=x$. Let $s>0$ be very small such that $[t, t +s]$ lies within $[t_k, t_{k+1}]$. We apply the random  action $G^{\bpi}(t_k, X_{t_k}^{\bpi, \mathbb{S}}, U_{k+1})$ throughout $(t,t+s]$. Let $a_1,\cdots,a_N$ be $N$ independent samples of this random action.} With a slight abuse of notations, denote by $X_{t+s}^{a_i}$ the value of the state process corresponding to ${a}_i$ at $t+s$.  We consider  $\mathbb{E}_{t,x}[f(t+s, X_{t+s}^{\bpi, \mathbb{S}})]$, where the expectation, conditional on  $X^{\bpi, \mathbb{S}}_t=x$, is taken over the randomness caused by $U_{k+1}$ as well as $W$, $L_1,\cdots,L_\ell$.  Then
\begin{align}
	&\phantom{=}\lim_{s\to 0}\frac{\mathbb{E}_{t,x}[f(t+s, X_{t+s}^{\bpi, \mathbb{S}})]-f(t,x)}{s}\\
	&=\lim_{s\to 0}\frac{\lim_{N\to\infty}\frac{1}{N}\sum_{i=1}^N \mathbb{E}_{t,x}[f(t+s,X_{t+s}^{a_i})]-f(t,x)}{s}\\
	&=\lim_{N\to\infty}\frac{1}{N}\sum_{i=1}^N\lim_{s\to 0}\frac{ \mathbb{E}_{t,x}[f(t+s,X_{t+s}^{a_i})]-f(t,x)}{s}\\
	&=\lim_{N\to\infty}\frac{1}{N}\sum_{i=1}^N \big( \partial_t f(t,x) + b(t,x,a_i)\circ\partial_x f(t,x) + \frac{1}{2}\sigma^2(t,x,a_i)\circ\partial_x^2 f(t,x)\big)\label{eq:part1}\\
	&+\lim_{N\to\infty}\frac{1}{N} \sum_{i=1}^N\sum_{k=1}^\ell \int_{\mathbb{R}} \left( f(t, x+ \gamma_k(t, x ,a_i, z)) - f(t,x) - \gamma_k(t,x,a_i,z) \circ  \partial_xf(t,x) \right) \nu_k (dz)\label{eq:part2}.
\end{align}
Using LLN, we obtain
\begin{equation}
\eqref{eq:part1}=\partial_tf(t,x) + \tilde{b}(t, x, \bpi_{t,x})\circ\partial_x f(t,x) + \frac{1}{2}\tilde{\sigma}^2(t, x, \bpi_{t,x})\circ\partial^2_x f(t,x),\label{eq:expl-gen-diff}
\end{equation}
where
\begin{equation}
	\tilde{b}(t, x, \bpi_{t,x})\coloneqq\int_{\mathcal{A}} b(t, x, a) \bpi(a|t,x)da,\quad \tilde{\sigma}(t, x, \bpi_{t,x})\coloneqq\Big(\int_{\mathcal{A}} \sigma^2(t, x, a) \bpi(a|t,x)da\Big)^{1/2}.\label{eq:diff-expl-coeff}
\end{equation}
These ``exploratory" drift and diffusion coefficients are 
consistent with those in \cite{wang2020reinforcement}. It is thus tempting to think the exploratory jump coefficient $\tilde \gamma$ is similarly the average of $\gamma$ with respect to $\bpi$; but unfortunately it is generally not true. This in turn is one of the main distinctive features in studying RL for jump-diffusions.

We approach the problem by analyzing the integrals in  \eqref{eq:part2}.  Using the second-order Taylor expansion, the boundedness of $\partial^2_x f(t,x)$ for $x\in\mathbb{R}^d$ and condition (iii) of Assumption \ref{assump:SDE}, we obtain that for fixed $(t,x)$ and each $k$,
\begin{align}
	&\Big|\int_{\mathbb{R}} \big( f(t, x+ \gamma_k(t, x ,a, z)) - f(t,x) - \gamma_k(t,x,a,z) \circ \partial_xf(t,x) \big) \nu_k (dz)\Big|\\
	&\leq C \int_{\mathbb{R}}|\gamma_k(t,x,a,z)|^2\nu_k(dz)\leq C(1+|x|^2) \label{eq:bound-for-Fubini}
\end{align}
for some constant $C>0$, which is independent of $a$. It follows that
\begin{equation}
	\eqref{eq:part2}=\sum_{k=1}^\ell \int_{\mathcal{A}}\int_{\mathbb{R}} \big( f(t, x+ \gamma_k(t, x ,a, z)) - f(t,x) - \gamma_k(t,x,a,z) \circ  \partial_xf (t,x) \big) \nu_k (dz)\bpi(a | t,x)da.\label{eq:expl-gen-jump}
\end{equation}

Combining \eqref{eq:expl-gen-diff} and \eqref{eq:expl-gen-jump}, \rev{the infinitesimal behavior of the grid sample state process (at grid points) can be characterized by the operator $\mathcal{L}^{\bpi}$, which is given by the probability weighted average of the generator $\mathcal{L}^a$ of the classical controlled jump-diffusion, i.e.,
\begin{equation}\label{eq:expl-gen-1}
	\mathcal{L}^{\bpi}f(t,x) := \lim_{s\to 0}\frac{\mathbb{E}_{t,x}[f(t+s, X_{t+s}^{\bpi, \mathbb{S}})]-f(t,x)}{s} = \int_{\mathcal{A}}\mathcal{L}^a f(t,x)\bpi(a|t,x)da.
\end{equation}}

We now reformulate the integrals in \eqref{eq:part2} to convert them to the same form as \eqref{eq:gene1}. Note that
\begin{align}
	&\int_{\mathcal{A}}\int_{\mathbb{R}} \big( f(t, x+ \gamma_k(t, x ,a, z)) - f(t,x) - \gamma_k(t,x,a,z) \circ  \partial_x f(t,x) \big) \nu_k (dz)\bpi(a | t,x)da\\
	=&\int_{\mathbb{R}\times[0,1]^n} \big( f\left(t, x+ \gamma_k\left(t, x , G^{\bpi_{t,x}}(u), z\right)\right) - f(t,x) - \gamma_k\left(t,x,G^{\bpi_{t,x}}(u),z\right) \circ \partial_x f(t,x) \big) \nu_k (dz)du.
\end{align}
It follows that $\mathcal{L}^{\bpi}$ can be written as
\begin{align}
	&\mathcal{L}^{\bpi}f(t,x)= \partial_t f(t,x) + \tilde{b}(s, x, \bpi_{t,x})\circ\partial_x f(t,x) + \frac{1}{2}\tilde{\sigma}^2(s, x, \bpi_{t,x})\circ\partial_x^2 f(t,x)\\
	&+\sum_{k=1}^\ell\int_{\mathbb{R}\times[0,1]^n} \big(f\left(t, x+ \gamma_k\left(t, x , G^{\bpi_{t,x}}(u), z\right)\right) - f(t,x) - \gamma_k\left(t,x,G^{\bpi_{t,x}}(u),z\right) \circ \partial_x f(t,x) \big) \nu_k (dz)du. \label{eq:expl-gen-2}
\end{align}

For each $k=1,\cdots,\ell$, we can view $\nu_k(dz)du$ as the intensity measure of a new Poisson random measure, denoted by $N'_k(dt,dz,du)$, defined on the product space $[0,T]\times\mathbb{R}\times[0,1]^n$. We let $N'(dt,dz,du)=(N'_1(dt,dz_1,du),\cdots,N'_{\ell}(dt,dz_\ell,du))^\top$.\footnote{\rev{We assume the original probability space $(\Omega,\mathcal{F},\mathbb{P})$ is rich enough to support $N'$.}}


Based on \eqref{eq:expl-gen-2}, we formulate the exploratory state process as the solution to the following SDE:
\begin{align}
	d X^{\bpi}_s =&\ \tilde{b}(s, X^{\bpi}_{s-}, \bpi(\cdot|s,X^{\bpi}_{s-})) ds + \tilde{\sigma}(s, X^{\bpi}_{s-}, \bpi(\cdot|s,X^{\bpi}_{s-})) dW_s \\
	& + \int_{\mathbb{R}\times[0,1]^n} \gamma\left(s, X^{\bpi}_{s-}, G^{\bpi}(s,X^{\bpi}_{s-},u), z\right) \widetilde N'(ds, dz, du),\ X^{\bpi}_t=x,\ s \in [t, T],\label{eq:SDE-expl}
\end{align}
which we call the {\it exploratory L\'evy SDE}. The solution process, should it exist, is denoted by $\tilde{X}^{\bpi}$ and called the exploratory (state) process. \rev{It is a semimartingale by \eqref{eq:SDE-expl} and determined by three characteristics: the drift, the quadratic variation of the continuous local martingale, and the compensator of the random measure associated with the process's jumps; see \cite{jacod2013limit} for detailed discussions of semimartingales and their characteristics. Given $\tilde{X}_{s-}^{\bpi}=x$, the semimartingale characteristics of $\tilde{X}^{\bpi}$ over an infinitesimally small time interval $[s,s+ds]$ are given by the triplet $(\tilde{b}(s, x, \bpi_{s,x})ds, \tilde{\sigma}^2(s, x, \bpi_{s,x})ds, \sum_{k=1}^{\ell}\int_{[0,1]^n}\gamma_k(s, x, G^{\bpi}(s,x,u), z)du\cdot\nu_k(dz)ds)$, where the third characteristic is obtained by calculating $\mathbb{E}[\sum_{k=1}^\ell\gamma_k\left(s, x, G^{\bpi}(s,x,u), z\right)N_k(ds, dz, du)]$ for L\'evy jumps with size from $[z,z+dz]$. Using \eqref{eq:diff-expl-coeff}, we have
\begin{align}
	&\tilde{b}(s, x, \bpi_{s,x})ds = \int_{\mathcal{A}}b(s,x,a)\bpi(a|s,x)dads,\ \ \tilde{\sigma}^2(s, x, \bpi_{s,x})ds = \int_{\mathcal{A}}\sigma^2(s,x,a)\bpi(a|s,x)dads,\\
	&\sum_{k=1}^\ell\int_{[0,1]^n}\gamma_k(s, x, G^{\bpi}(s,x,u), z)du\cdot\nu_k(dz)ds = \sum_{k=1}^\ell\int_{\mathcal{A}}\gamma_k\left(s, x, a, z\right)\nu_k(dz)ds\cdot\bpi(a|s,x)da.
\end{align}
Thus, the semimartingale characteristics of the exploratory state process are the averages of those of the original controlled state process \eqref{eq:LevySDE} over action randomization.}

A technical yet foundational question is the well-posedness (i.e. existence and uniqueness of solution) of the exploratory SDE \eqref{eq:SDE-expl}, which we now address. For that we first specify the class of admissible strategies, which is the same as that considered in \cite{jia2022q} for pure diffusions.

\begin{definition}\label{def:policy}
	A policy $\bpi = \bpi(\cdot| \cdot, \cdot)$ is called admissible, if
	\begin{enumerate}[(i)]
		\item $\bpi(\cdot| t,x) \in \mathcal{P} (\mathcal{A})$,  $\text{supp}\bpi(\cdot| t,x) = \mathcal{A}$ for every $(t,x) \in [0, T] \times \mathbb{R}^d$ and $\boldsymbol{\pi} (a| t,x): (t,x,a) \in [0, T] \times   \mathbb{R}^d \times \mathcal{A} \rightarrow \mathbb{R}$ is measurable;
		
		\item $\bpi(a|t,x)$ is continuous in $(t,x)$, i.e.,  $\int_{\mathcal{A}}\left|\bpi(a|t,x)-\bpi(a|t',x')\right|da\to 0$ as $(t',x')\to (t,x)$. Furthermore, for any $K>0$, there is a constant $C_K>0$ independent of $(t,a)$ such that
		\begin{equation}
			\int_{\mathcal{A}}\left|\bpi(a|t,x)-\bpi(a|t,x')\right|da\leq C_K|x-x'|,\ \forall x,x'\in\mathbb{R}^d_K;
		\end{equation}
	
		\item $\forall (t,x)$, $\int_{\mathcal{A}} |\log \boldsymbol{\pi}(a|t,x)| \boldsymbol{\pi}(a |t,x) da \le C (1+ |x|^p)$ for some $p\geq 2$ and some positive constant $C$. \rev{Moreover, $\int_{\mathcal{A}} |a|^p\bpi(a|t,x) da \le C_p (1+ |x|^p)$ for any $p\geq 1$ and some positive constant $C_p$}.
	\end{enumerate}
\end{definition}

Next, we establish the well-posedness of \eqref{eq:SDE-expl} under any given admissible policy. The result of the next lemma regarding $\tilde b$ and $\tilde \sigma$ is provided in the proof of Lemma 2 in \cite{jia2022policyb}, which uses property (ii) of admissibility.

\begin{lemma}\label{lem:expl-diff-coeffs}
	Under Assumption \ref{assump:SDE}, for any admissible policy $\bpi$, the functions $\tilde{b}(t, x, \bpi_{t,x})$ and $\tilde{\sigma}(t, x, \bpi_{t,x})$ have the following properties:
	\begin{enumerate}[(i)]
		\item (local Lipschitz continuity) for $K>0$, there exists a constant $C_K>0$ such that $\forall t\in[0, T], (x,x')\in\mathbb{R}^d_K$,
		\begin{equation}
			|\tilde{b}(t,x,\bpi_{t,x})-\tilde{b}(t,x',\bpi_{t,x'})|+|\tilde{\sigma}(t,x,\bpi_{t,x})-\tilde{\sigma}(t,x',\bpi_{t,x'})|\leq C_K|x-x'|\ ;
		\end{equation}
		
		\item (linear growth in $x$) there exists a constant $C>0$ such that $\forall(t, x) \in[0, T] \times \mathbb{R}^d$,
		\begin{equation}
			|\tilde{b}(t,x,\bpi_{t,x})|+|\tilde{\sigma}(t,x,\bpi_{t,x})|\leq C(1+|x|).
		\end{equation}
	\end{enumerate}
\end{lemma}

We now establish analogous  properties for $\gamma(t,x,G^{\bpi_{t,x}}(u),z)$ in the following lemmas, whose proofs are relegated to the appendix.

\begin{lemma}[linear growth in $x$]\label{lem:expl-jump-LG}
Under Assumption \ref{assump:SDE}, for any admissible $\bpi$ and any $p\geq 2$, there exists a constant $C_p>0$ that can depend on $p$ such that $\forall(t, x) \in[0, T] \times \mathbb{R}^d$,
\begin{equation}
	\sum_{k=1}^\ell \int_{\mathbb{R}\times[0,1]^n}\left|\gamma_k\left(t,x,G^{\bpi_{t,x}}(u),z\right)\right|^p\nu_k(dz)du\leq C_p(1+|x|^p).
\end{equation}
\end{lemma}

For the local Lipschitz continuity of $\gamma_k(t,x,G^{\bpi_{t,x}}(u),z)$, we make an additional  assumption.
\begin{assumption}\label{assump:jump}
	For $k=1,\cdots,\ell$, the following conditions hold.
	\begin{enumerate}[(i)]
		\item For any $K>0$ and any $p\geq 2$, there exists a constant $C_{K,p}>0$ that can depend on $K$ and $p$ such that
		\begin{equation}
			\int_{\mathbb{R}}\left|\gamma_{k}(t,x,a,z)-\gamma_{k}(t,x,a',z)\right|^p\nu_k(dz)\leq C_{K,p}|a-a'|^p,\quad \forall t\in[0,T], a,a'\in\mathcal{A}, x\in \mathbb{R}^d_K, z\in\mathbb{R}.
		\end{equation}
		
		\item For any $K>0$ and any $p\geq 2$, there exists a constant $C_{K,p}>0$ that can depend on $K$ and $p$ such that
		\begin{equation}\label{eq:G-LC}
			\int_{[0,1]^n}\left|G^{\bpi}(t,x,u)-G^{\bpi}(t,x',u)\right|^pdu\le C_{K,p}|x-x'|^p,\quad \forall t\in[0,T], {x, x' \in \mathbb{R}^d_K.}
		\end{equation}
	\end{enumerate}
\end{assumption}

For a stochastic feedback policy $\bpi_{t,x}\sim\mathcal{N}(\mu(t,x), A(t,x)A(t,x)^\top)$, we have $G^{\bpi}(t,x,u)=\mu(t,x) + A(t,x)\Phi^{-1}(u)$. Clearly, Assumption \ref{assump:jump}-(ii) holds provided that $\mu(t,x)$ and $A(t,x)$ are locally Lipschitz continuous in $x$.

\begin{lemma}[local Lipschitz continuity]\label{lem:expl-jump-LC} Under Assumptions \ref{assump:SDE} and \ref{assump:jump}, for any admissible policy $\bpi$, any $K>0$, and any $p\geq 2$, there exists a constant $C_{K,p}>0$ that can depend on $K$ and $p$ such that $\forall t\in[0, T], (x,x')\in\mathbb{R}^d_K$,
	\begin{equation}
		\sum_{k=1}^\ell \int_{\mathbb{R}\times[0,1]^n}\left|\gamma_k\left(t,x,G^{\bpi_{t,x}}(u),z\right)-\gamma_k\left(t,x',G^{\bpi_{t,x'}}(u),z\right)\right|^p\nu_k(dz)du\le C_{K,p}|x-x'|^p.
	\end{equation}
\end{lemma}

With Lemmas \ref{lem:expl-diff-coeffs} to \ref{lem:expl-jump-LC}, we can now apply \cite[Theorem 3.2]{kunita2004stochastic} and \cite[Theorem 119]{rong2006theory} to obtain the well-posedness of \eqref{eq:SDE-expl} along with the moment estimate of its solution.
\begin{proposition}\label{prop:expl-SDE-wellposed}
	Under Assumptions \ref{assump:SDE} and \ref{assump:jump}, for any admissible policy $\bpi$, there exists a unique strong solution $\{\tilde{X}_t^{\bpi}, 0\leq t\leq T\}$ to the exploratory L\'evy SDE \eqref{eq:SDE-expl}. Furthermore, for any $p\geq 2$, there exists a constant $C_p>0$ such that
	\begin{equation}\label{eq:moment-expl}
		\mathbb{E}_{t,x}\Big[\sup_{t\le s\le T} |\tilde{X}_s^{\bpi}|^{p}\Big]\le C_p(1+|x|^{p}).
	\end{equation}
\end{proposition}

It should be noted that the conditions imposed in Assumptions \ref{assump:SDE} and \ref{assump:jump} are sufficient but not necessary for obtaining the well-posedness and moment estimate of the exploratory L\'evy SDE \eqref{eq:SDE-expl}. For a specific problem, weaker conditions may suffice for these results if we exploit special structures of the problem.

\gao{
\begin{remark}
   \cite{bender2024continuous} formulate the grid sample state process as an equation driven by certain random measures, and prove a limit theorem  for these random measures as the size of the sampling grid goes to zero.
 This limit theorem in turn motivates a limiting formulation of the grid sample state process, which is referred to as the grid-sampling limit SDE in their paper. Based on the discussion in Section 6.2 of \cite{bender2024continuous}, this limiting SDE turns out to have the same probability law as our exploratory SDE, even though it has a different representation.
\end{remark}
}



\subsection{Exploratory HJB equation}
\rev{Given the exploratory dynamic \eqref{eq:SDE-expl}, we now define the (exploratory) value function for any given admissible stochastic policy $\bpi$. To encourage exploration, we follow \cite{wang2020reinforcement} to add an entropy regularizer to the running reward of the orginal control problem, and set
\begin{align}
	J(t, x; \bpi) & = \mathbb{E}_{t,x} \bigg[  \int_t^T  e^{-\beta (s-t)} \int_{\mathcal{A}}\big( r(s, \tilde X_s^{\bpi}, a) - \theta \log\bpi(a|s, \tilde X_{s-}^{\bpi})\big)\bpi(a|s, \tilde X_{s-}^{\bpi})da ds\\
	&\quad\quad\quad~ +  e^{-\beta (T-t) } h( \tilde X_T^{\bpi} )\bigg], \label{eq:J2}
\end{align}
where the expectation is conditioned on $\tilde X_t^{\bpi} =x$. Here $\theta>0$ is the {\it temperature parameter} that controls the level of exploration. The function $J(\cdot, \cdot; \bpi)$ is called the (exploratory) value function of the policy $\bpi$. The goal of RL is to find the policy that maximizes the value function among admissible policies specified in Definition \ref{def:policy}. }

Under Assumption \ref{assump:SDE} and using the admissibility of $\bpi$ and \eqref{eq:moment-expl}, it is easy to see that $J$ has polynomial growth in $x$. We provide the Feynman--Kac formula for this function in Lemma \ref{lem:PK} by working with the representation \eqref{eq:J2}. In the proof, we consider the finite and infinite jump activity cases separately because special care is needed in the latter. We revise Assumption \ref{assump:SDE} by adding one more condition for this case.

\begin{assumptionp}{\ref*{assump:SDE}$'$}\label{assump:SDE-1}
	Conditions (i) to (iv) in Assumption \ref{assump:SDE} hold. We further assume condition (v): if $\int_{\mathbb{R}}\nu_k(dz)=\infty$, then $|\gamma_k(t,x,a,z)|$ is bounded for any $|z|\leq 1$, $t\in[0,T]$, $x\in\mathbb{R}_K^d$ and $a\in\mathcal{A}$.
\end{assumptionp}

\textit{For Lemma \ref{lem:PK}, Lemma \ref{lem:expHJB} and Theorem \ref{thm:PI}, we impose Assumption \ref{assump:SDE-1} and assume that the exploratory SDE \eqref{eq:SDE-expl} is well-posed with the moment estimate \eqref{eq:moment-expl}}. For simplicity, we do not explicitly mention these assumptions in the statement of the results.

\begin{lemma}\label{lem:PK}
		Given an admissible stochastic policy $\bpi$, suppose there exists a solution $ \phi\in C^{1,2}([0,T)\times\mathbb{R}^d)\cap C([0,T]\times\mathbb{R}^d)$ to the following partial integro-differential equation (PIDE): $\forall (t,x) \in [0, T)\times\mathbb{R}^d$,
	\begin{align}
		\partial_t \phi(t,x) +   \int_{\mathcal{A}} \big(H(t,x, a, \partial_x\phi, \partial_x^2\phi, \phi) - \theta \log \bpi(a | t, x)\big) \bpi(a | t, x) da  - \beta \phi(t, x) = 0, \label{eq:FK}
	\end{align}
	with terminal condition $\phi(T, x) = h (x)$, $x\in \mathbb{R}^d$. Moreover, for some $p\geq 2$, $\phi$ satisfies
	\begin{equation}
		|\phi(t, x)|\leq C(1+|x|^p),\ \forall (t,x)\in[0,T]\times\mathbb{R}^d.\label{eq:VF-PolyG}
	\end{equation}
	Then $\phi$ is the value function of the policy $\boldsymbol{\pi}$, i.e. $J(t, x; \boldsymbol{\pi})=\phi(t,x)$.
\end{lemma}

For ease of  presentation, we henceforth assume the value function $J(\cdot, \cdot; \bpi) \in C^{1,2}([0,T)\times\mathbb{R}^d)\cap C([0,T]\times\mathbb{R}^d)$ for any admissible stochastic policy $\bpi$.

\begin{remark}
The conclusion in Lemma \ref{lem:PK} still holds if we assume $\partial_x\phi(t,x)$ has polynomial growth in $x$ instead of Assumption~\ref{assump:SDE-1}-(v).
\end{remark}

Next, we consider the {\it optimal} value function defined by
\begin{align} \label{eq:Jstar}
J^*(t, x) = \sup_{\bpi\in\bPi} J(t, x; \bpi),
\end{align}
where $\bPi$ is the class of admissible strategies. The following result characterizes $J^*$ and the optimal stochastic policy through the so-called {\it exploratory HJB equation}.

\begin{lemma} \label{lem:expHJB}
Suppose there exists a solution $ \psi \in C^{1,2}([0,T)\times\mathbb{R}^d)\cap C([0,T]\times\mathbb{R}^d)$ to the exploratory HJB equation: 
\begin{align} \label{eq:HJB-explore}
	\partial_t \psi(t,x)   +  \sup_{\bpi\in \mathcal{P} (\mathcal{A})} \int_{\mathcal{A}} \big( H(t,x, a, \partial_x\psi,  \partial_x^2\psi, \psi) - \theta \log \boldsymbol{\pi}(a|t,x)\big) \boldsymbol{\pi}(a|t,x) da  - \beta \psi (t, x) = 0,
\end{align}
with the terminal condition $\psi(T, x)  = h (x)$,  where $H$ is the Hamiltonian defined in \eqref{eq:Hal-1}. Moreover, for some $p\geq 2$, $\psi$ satisfies
\begin{equation}
	|\psi(t, x)|\leq C(1+|x|^p),\ \forall (t,x)\in[0,T]\times\mathbb{R}^d,\label{eq:OVF-PolyG}
\end{equation}
and it holds that
\begin{equation}
	\int_{\mathcal{A}} \exp \Big(\frac{1}{\theta} H(t,x, a, \partial_x\psi,  \partial_x^2\psi, \psi)\Big) da<\infty.
\end{equation}
Then, the Gibbs measure or Boltzman distribution
\begin{align}\label{eq:opt-pol}
	{\bpi}^*(a| t, x) \propto \exp \Big(\frac{1}{\theta} H(t,x, a, \partial_x\psi,  \partial_x^2\psi, \psi)\Big)
\end{align}
 is the optimal stochastic policy and $J^*(t,x)=\psi(t,x)$ provided that $\boldsymbol{\pi}^*$ is admissible.
\end{lemma}

Plugging the optimal stochastic policy \eqref{eq:opt-pol} into Eq.~\eqref{eq:HJB-explore} to remove the supremum operator, we obtain the following nonlinear PIDE for the optimal value function $J^*$: $\forall (t,x) \in [0, T)\times\mathbb{R}^d$,
\begin{equation} \label{eq:HJB-explore2}
\partial_t J^*(t,x)   + \theta \log \bigg(\int_{\mathcal{A}} \exp \Big(\frac{1}{\theta} H(t,x, a, \partial_xJ^*,  \partial_x^2J^*, J^*)\Big)  da \bigg)  - \beta J^* (t, x) = 0,
\end{equation}
with terminal condition $J^*(T,x)=h(x)$, $x\in\mathbb{R}^d$.

\gao{
\begin{remark}
The optimal value function $ J^*$ in \eqref{eq:HJB-explore2} for the exploratory control problem is expected to converge to the optimal value function $V^*$ for the classical control problem \eqref{eq:FH-control} when $\theta \rightarrow 0$ under certain technical conditions. In the case of pure diffusions without jumps, this has been proved in \cite{tang2022exploratory} for infinite-horizon problems. Establishing such a result for general jump-diffusions necessitates a careful analysis of the dependency of PIDE \eqref{eq:HJB-explore2} on $\theta$, which falls beyond the scope of this paper.
 However, we will see that in each of the two applications considered in this paper (see Sections~\ref{sec:mean-var} and \ref{sec:MV-hedge}), the optimal value function $J^*$ of the exploratory problem indeed converges to the classical optimal value function $V^*$  as $\theta \rightarrow 0$,  due to the convergence of the optimal stochastic feedback policy to the optimal deterministic feedback policy.
\end{remark}
}
\section{q-Learning Theory} \label{sec:qlearn}
\subsection{q-function and policy improvement}

We present the q-learning theory for jump-diffusions that includes both policy evaluation and policy improvement, now that the exploratory formulation has been set up. The theory can be developed similarly to \cite{jia2022q,jia2025erratum}; so we will just highlight the main differences in the analysis, skipping the parts that are similar.\footnote{The q-learning theory for pure diffusions is established originally in \cite{jia2022q}. In the wake of the observations of the aforementioned measurability issue, an erratum that works with grid sample state processes is put forth in \cite{jia2025erratum}.}

\begin{definition}\label{eq：q-def}
The q-function of the problem \eqref{eq:LevySDE}--\eqref{eq:FH-control} associated with a given policy $\boldsymbol{\pi} \in \boldsymbol{\Pi}$ is defined by
\begin{align}
	q(t,x, a; \boldsymbol{\pi}) & =
	\partial_t J(t, x; \boldsymbol{\pi}) +   H(t,x, a, \partial_xJ,  \partial_x^2J, J) - \beta J(t, x; \boldsymbol{\pi}),\;\;(t,x,a) \in [0, T] \times \mathbb{R}^d \times \mathcal{A},
\end{align}
where $J$ is given in \eqref{eq:J2} and the Hamiltonian function $H$ is defined in \eqref{eq:Hal-1}.
\end{definition}

It is an immediate consequence of Lemma~\ref{lem:PK} that the q-function satisfies
\begin{align}\label{eq:q-eq}
\int_{\mathcal{A}} \big(q(t,x, a; \boldsymbol{\pi}) - \theta \log \boldsymbol{\pi}(a |t,x)\big)\boldsymbol{\pi}(a |t,x) da =0,\;\;\forall(t,x) \in [0, T] \times \mathbb{R}^d.
\end{align}

The following  policy improvement theorem can be proved similarly as in \cite[Theorem 2]{jia2022q} by using the arguments in the proof of Lemma \ref{lem:PK}.
\begin{theorem} [Policy Improvement] \label{thm:PI}
For any given $\boldsymbol{\pi} \in \boldsymbol{\Pi}$, define $$\boldsymbol{\pi}'(\cdot | t, x) \propto \exp \Big(\frac{1}{\gamma} H\big(t,x, \cdot, \partial_xJ(t,x; \boldsymbol{\pi} ),  \partial_x^2J(t,x; \boldsymbol{\pi}), J(t, \cdot; \boldsymbol{\pi})\big) \Big)  \propto \exp \Big(\frac{1}{\gamma} q(t,x, \cdot; \boldsymbol{\pi})\Big).$$ If $\boldsymbol{\pi}' \in \boldsymbol{\Pi},$ then
\begin{align}\label{eq:PI}
	J(t,x, \boldsymbol{\pi}' ) \ge J(t,x, \boldsymbol{\pi} ).
\end{align}
Moreover, if the following map
\begin{align}
	\mathcal{I} ( \boldsymbol{\pi} ) & = \frac{\exp(\frac{1}{\theta} H(t,x, \cdot, \partial_xJ(t,x; \boldsymbol{\pi} ),  \partial_{x}^2J(t,x; \boldsymbol{\pi}), J(t, \cdot; \boldsymbol{\pi})) ) }{\int_{\mathcal{A}} \exp (\frac{1}{\theta} H(t,x, a, \partial_xJ(t,x; \boldsymbol{\pi} ),  \partial_x^2J(t,x; \boldsymbol{\pi}), J(t, \cdot; \boldsymbol{\pi}))) da}, \quad \boldsymbol{\pi} \in \boldsymbol{\Pi} \\
	& = \frac{\exp (\frac{1}{\theta}  q(t,x, \cdot; \boldsymbol{\pi})) }{\int_{\mathcal{A}} \exp(\frac{1}{\theta} q(t,x, a; \boldsymbol{\pi})) da}
\end{align}
has a fixed point $\boldsymbol{\pi}^*$, then $\boldsymbol{\pi}^*$ is an optimal policy.
\end{theorem}

\subsection{Martingale characterization of the q-function}
\rev{In this section, we derive the martingale characterization of the q-function associated with a policy $\boldsymbol{\pi} \in \boldsymbol{\Pi}$, which will be used for on-policy and off-policy learning.  We cannot work with the exploratory process directly because it is unobservable (i.e. it does not constitute {\it data}). We consider the grid sample state process instead, which is what one obtains from the environment by following a stochastic policy.}

\rev{For any $t \in [0, T)$, in the following we consider a grid on $[t, T]$, which will still be denoted by $\mathbb{S}$ for notational simplicity (though it depends on $t$). That is, $\mathbb{S}= \{t = t_0 < t_1 < \ldots < t_N = T\}$ for some $N$.
The grid sampling process on $[t, T]$ can be defined analogously as in \eqref{eq:grid-sampling}:
\begin{align} \label{eq:grid-sampling-t}
U_s^{\mathbb{S}} = U_1  \cdot 1_{[t_0, t_1]}(s) + \sum_{i=2}^N U_i \cdot 1_{(t_{i-1}, t_i]}(s), \quad s \in [t,T],
\end{align}
where the first uniform random vector $U_1$ is sampled at $t = t_0.$
The grid sample state process $X^{\boldsymbol{\pi, \mathbb{S} }} = \{X^{\boldsymbol{\pi, \mathbb{S}}}_s: t \le s \le T\}$ starts from $t$ and satisfies the SDE \eqref{eq:samplestate}. The action process $a^{\boldsymbol{\pi, \mathbb{S} }} = \{a^{\boldsymbol{\pi, \mathbb{S}}}_s: t \le s \le T\}$ satisfies $a^{\bpi,\mathbb{S}}_{s}=a_{t_n}^{\bpi,\mathbb{S}}$ for $s \in (t_n,t_{n+1}]$, where $a_{t_n}^{\bpi,\mathbb{S}}\coloneqq G^{\bpi}(t_n, X^{\bpi, \mathbb{S}}_{t_n}, U_{n+1})$ is applied over $(t_n,t_{n+1}]$.
With a slight abuse of notations,
we still denote by $\mathcal{G}^{\mathbb{S}}$ the right-continuous, augmented version of the filtration generated by the Brownian motion $W$,  L\'evy processes $L_1,\cdots,L_{\ell}$, and the grid sampling process on $[t, T]$. }

\gao{We need the following result for the grid sample state process $X^{\boldsymbol{\pi, \mathbb{S} }} = \{X^{\boldsymbol{\pi, \mathbb{S}}}_s: t \le s \le T\}$.
\begin{proposition} \label{prop:SDE-sampled}
	Fix any $(t,x)\in[0,T]\times\mathbb{R}^d$. Under Assumption \ref{assump:SDE}, for any admissible policy $\bpi$ and any time grid $\mathbb{S}$ of $[t, T]$, there exists a unique strong solution $\{{X}_s^{\bpi, \mathbb{S}}, t \leq s\leq T\}$ to the grid sample state SDE \eqref{eq:samplestate} with the initial condition ${X}_t^{\bpi, \mathbb{S}}=x$. Furthermore, for any $p\geq 2$, there exists a constant $C_p$ that is independent of the grid $\mathbb{S}$ such that
	\begin{equation}\label{eq:moment-grid}
		\mathbb{E}_{t,x}\Big[\sup_{t \le s\le T}|X_s^{\bpi, \mathbb{S}}|^p\Big] \le C_p(1+|x|^{p}).
	\end{equation}
\end{proposition} }

\rev{We now state in Theorem~\ref{thm:mart1} the martingale characterization of the q-function associated with a policy $\boldsymbol{\pi} \in \boldsymbol{\Pi}$, assuming that its value function has already been learned and known.} Note that for Theorems \ref{thm:mart1}--\ref{thm:opt-mart} below, we impose Assumption \ref{assump:SDE} without explicitly mentioning it in the theorem statements .

\begin{theorem}\label{thm:mart1}
Let a policy $\boldsymbol{\pi} \in \boldsymbol{\Pi},$ its value function $J(\cdot, \cdot; \boldsymbol{\pi}) \in C^{1,2}([0,T)\times\mathbb{R}^d)\cap C([0,T]\times\mathbb{R}^d)$ and a continuous function $\hat q: [0, T] \times \mathbb{R}^d \times \mathcal{A} \rightarrow \mathbb{R}$ be given. Assume that $J(t,x;\bpi)$ and $J_x(t,x;\bpi)$ both have polynomial growth in $x$. Then the following results hold.
\begin{enumerate}
	
\item [(i)] \rev{$\hat q(t,x, a) = q(t,x, a; \boldsymbol{\pi}) $ for all $(t,x,a)$ if and only if for any $(t,x)$ and any grid $\mathbb{S}$ of $[t,T]$, the following process
	\begin{align} \label{eq:mart1}
		e^{-\beta s}  J(s, X_s^{ \boldsymbol{\pi, \mathbb{S}}}; \boldsymbol{\pi}) +    \int_t^s e^{-\beta \tau } \big( r( \tau, X_{\tau}^{\boldsymbol{\pi}, \mathbb{S}}, a_{\tau}^{\boldsymbol{\pi},  \mathbb{S}}) - \hat q ( \tau, X_{\tau}^{\boldsymbol{\pi, \mathbb{S}}}, a_{\tau}^{\boldsymbol{\pi, \mathbb{S}}})\big) d\tau
	\end{align}
	is a $\mathcal{G}^{\mathbb{S}}$-martingale, where  $X^{\boldsymbol{\pi, \mathbb{S} }} = \{X^{\boldsymbol{\pi, \mathbb{S}}}_s: t \le s \le T\}$ is the grid sample state process defined in \eqref{eq:samplestate} with $X^{\boldsymbol{\pi}, \mathbb{S}}_t =x$.}

	\item [(ii)] \rev{If $\hat q(t,x, a) = q(t,x, a; \boldsymbol{\pi}) $ for all $(t,x,a)$, then for any $\boldsymbol{\pi}' \in \boldsymbol{\Pi}$, any $(t,x)$ and any grid $\mathbb{S}$ of $[t,T]$,  the following process
	\begin{align} \label{eq:mart2}
		e^{-\beta s}  J(t, X_s^{ \boldsymbol{\pi}', \mathbb{S}}; \boldsymbol{\pi}) +    \int_t^s e^{-\beta\tau}\big(r( \tau, X_{\tau}^{\boldsymbol{\pi}', \mathbb{S}}, a_{\tau}^{\boldsymbol{\pi}', \mathbb{S}}) - \hat q(\tau, X_{\tau}^{\boldsymbol{\pi}', \mathbb{S}}, a_{\tau}^{\boldsymbol{\pi}', \mathbb{S}})\big) d\tau
	\end{align}
	is a $\mathcal{G}^{\mathbb{S}}$-martingale, where  $ \{X^{\boldsymbol{\pi}', \mathbb{S}}_s: t \le s \le T\}$ is the solution to  \eqref{eq:samplestate} under $\boldsymbol{\pi}'$ with initial condition $X^{\boldsymbol{\pi}', \mathbb{S}}_t =x$.}

	\item [(iii)] \rev{If there exists $\boldsymbol{\pi}' \in \boldsymbol{\Pi}$ such that for all $(t,x)$ and any grid $\mathbb{S}$ of $[t,T]$, the process \eqref{eq:mart2} is a $\mathcal{G}^{\mathbb{S}}$-martingale where $X^{\boldsymbol{\pi}', \mathbb{S}}_t =x$, then we have $\hat q(t,x, a) = q(t,x, a; \boldsymbol{\pi}) $ for all $(t,x,a)$. }
	
\end{enumerate}
Moreover, in any of the three cases above, the q-function satisfies
\begin{align}
	\int_{\mathcal{A}} \big( q(t,x, a; \boldsymbol{\pi}) - \theta \log \boldsymbol{\pi}(a| t,x)\big) \boldsymbol{\pi}(a|t,x) da =0, \quad \text{for all $(t,x) \in [0, T] \times \mathbb{R}^d $}.
\end{align}
\end{theorem}

\begin{remark}
Similar to \cite{jia2022q,jia2025erratum}, Theorem~\ref{thm:mart1}-(i) facilitates on-policy learning, where learning the q-function of the given target policy $\boldsymbol{\pi}$ is based on data $\{(s, X_s^{ \boldsymbol{\pi}, \mathbb{S}}, a_s^{ \boldsymbol{\pi}, \mathbb{S}}), t \le s \le T \}$ generated by $\boldsymbol{\pi}$. On the other hand, Theorem~\ref{thm:mart1}-(ii) and -(iii) are for off-policy learning, where learning the q-function of $\boldsymbol{\pi}$ is based on data generated by a different, called behavior policy, $\boldsymbol{\pi}'$.
\end{remark}

Next, we extend Theorem 7 in \cite{jia2022q,jia2025erratum} and obtain a martingale characterization of the value function and the q-function simultaneously.

\begin{theorem}\label{thm:mart2}
Let a policy $\boldsymbol{\pi} \in \boldsymbol{\Pi},$ a function $\hat J \in C^{1,2}([0,T)\times\mathbb{R}^d)\cap C([0,T]\times\mathbb{R}^d)$ and a continuous function $\hat q: [0, T] \times \mathbb{R}^d \times \mathcal{A} \rightarrow \mathbb{R}$ be given satisfying
\begin{align}\label{eq:constraints}
	\hat J(T, x) = h (x), \quad  \int_{\mathcal{A}} \big(\hat q(t,x, a) - \theta \log \boldsymbol{\pi}(a| t,x)\big) \boldsymbol{\pi}(a|t,x) da =0, \quad \forall(t,x) \in [0, T] \times \mathbb{R}^d.
\end{align}
\rev{Assume that $\hat J$ and $\hat J_x$ both have polynomial growth in $x$.}
Then
\begin{enumerate}
	\item [(i)] \rev{$\hat J$ and $\hat q$ are respectively the value function and the q-function associated with $\boldsymbol{\pi}$ if and only if for any
	$(t,x) \in [0, T] \times \mathbb{R}^d $ and any grid $\mathbb{S}$ of $[t,T]$, the following process
	\begin{align}
		e^{-\beta s} \hat J(s, X_s^{ \boldsymbol{\pi} , \mathbb{S}}) +      \int_t^s e^{-\beta \tau } \big(r( \tau, X_{\tau}^{\boldsymbol{\pi} , \mathbb{S}}, a_{\tau}^{\boldsymbol{\pi},  \mathbb{S}}) - \hat q ( \tau, X_{\tau}^{\boldsymbol{\pi}, \mathbb{S}}, a_{\tau}^{\boldsymbol{\pi, \mathbb{S}}})\big) d\tau
	\end{align}
	is a $\mathcal{G}^{\mathbb{S}}$-martingale, where $X^{\boldsymbol{\pi}, \mathbb{S}} = \{X^{\boldsymbol{\pi}, \mathbb{S}}_s: t \le s \le T\}$ satisfies \eqref{eq:samplestate}  with $X^{\boldsymbol{\pi}, \mathbb{S}}_t =x$.}
	
	\item [(ii)] \rev{ If $\hat J$ and $\hat q$ are respectively the value function and the q-function associated with $\boldsymbol{\pi}$, then for any $\boldsymbol{\pi}' \in \boldsymbol{\Pi}$, any
	$(t,x) \in [0, T] \times \mathbb{R}^d $ and any grid $\mathbb{S}$ of $[t,T]$, the following process
	\begin{align} \label{eq:mart3}
		e^{-\beta s} \hat J(s, X_s^{ \boldsymbol{\pi}', \mathbb{S}}) +    \int_t^s e^{-\beta \tau} \big( r(\tau, X_{\tau}^{\boldsymbol{\pi}', \mathbb{S}}, a_{\tau}^{\boldsymbol{\pi}', \mathbb{S}}) - \hat q (\tau, X_{\tau}^{\boldsymbol{\pi}', \mathbb{S}}, a_{\tau}^{\boldsymbol{\pi}', \mathbb{S}})\big) d\tau
	\end{align}
	is a $\mathcal{G}^{\mathbb{S}}$-martingale, where  $ \{X^{\boldsymbol{\pi}', \mathbb{S}}_s: t \le s \le T\}$ satisfies  \eqref{eq:samplestate}  with $X^{\boldsymbol{\pi}', \mathbb{S}}_t =x$. }

	\item [(iii)]  \rev{ If there exists $\boldsymbol{\pi}' \in \boldsymbol{\Pi}$ such that for all $(t,x)$ and any grid $\mathbb{S}$ of $[t,T]$, the process \eqref{eq:mart3} is a $\mathcal{G}^{\mathbb{S}}$-martingale where $X^{\boldsymbol{\pi}', \mathbb{S}}_t =x$, then we have $\hat J(t,x) = J(t,x; \boldsymbol{\pi}) $ and $\hat q(t,x, a) = q(t,x, a; \boldsymbol{\pi}) $ for all $(t,x,a)$. }
\end{enumerate}

In any of the three cases above, if it holds that $\boldsymbol{\pi}(a|t,x) = \frac{\exp(\frac{1}{\theta} \hat q(t,x,a))} {\int_{\mathcal{A} } {\exp(\frac{1}{\theta} \hat q(t,x,a)) da}}$, then $\boldsymbol{\pi}$ is the optimal policy and $\hat J$ is the optimal value function.
\end{theorem}

\subsection{Optimal q-function}
We consider in this section the optimal q-function, i.e., the q-function associated with the optimal policy $\boldsymbol{\pi}^*$ in \eqref{eq:opt-pol}. Based on Definition~\ref{eq：q-def}, we can write it as
\begin{align}\label{eq：q-def-opt}
q^*(t,x, a) & = \partial_t J^*(t, x) +   H(t,x, a, \partial_xJ^*,  \partial_x^2J^*, J^*) - \beta J^*(t, x),
\end{align}
where $J^*$ is the optimal value function that solves the exploratory HJB equation in \eqref{eq:HJB-explore}.

The following is  the martingale condition that characterizes the optimal value function $J^*$ and the optimal q-function.
\begin{theorem}\label{thm:opt-mart}
Let a function $\hat J^* \in C^{1,2}([0,T)\times\mathbb{R}^d)\cap C([0,T]\times\mathbb{R}^d)$ and a continuous function $\hat q^*: [0, T] \times \mathbb{R}^d \times \mathcal{A} \rightarrow \mathbb{R}$ be given satisfying
\begin{align} \label{eq:boundary}
	\hat J^*(T, x) = h (x), \quad \int_{\mathcal{A} } {\exp \Big( \frac{1}{\theta} \hat q^*(t,x,a) \Big) da} = 1, \quad \text{for all $(t,x) \in [0, T] \times \mathbb{R}^d $}.
\end{align}
Assume that $\hat J^*(t,x)$ and $\hat J^*_x(t,x)$ both have polynomial growth in $x$.
Then
\begin{enumerate}
	\item [(i)] \rev{If $\hat J^*$ and $\hat q^*$ are respectively the optimal value function and the optimal q-function,  then for any $\boldsymbol{\pi} \in \boldsymbol{\Pi}$ , for all
	$(t,x) \in [0, T] \times \mathbb{R}^d $ and any grid $\mathbb{S}$ of $[t,T]$, the following process
	\begin{align} \label{eq:mart-opt}
		e^{-\beta s} \hat J^*(s, X_s^{ \boldsymbol{\pi}, \mathbb{S}}) +    \int_t^s e^{-\beta \tau } \big( r(\tau, X_{\tau}^{\boldsymbol{\pi}, \mathbb{S}}, a_{\tau}^{\boldsymbol{\pi}, \mathbb{S}}) - \hat q^* (\tau, X_{\tau}^{\boldsymbol{\pi} , \mathbb{S}}, a_{\tau}^{\boldsymbol{\pi}, \mathbb{S}})\big) d\tau
	\end{align}
	is a $\mathcal{G}^{\mathbb{S}}$-martingale, where $X^{\boldsymbol{\pi}, \mathbb{S}} = \{X^{\boldsymbol{\pi}, \mathbb{S}}_s: t \le s \le T\}$ satisfies \eqref{eq:samplestate}  with $X^{\boldsymbol{\pi}, \mathbb{S}}_t =x$. Moreover, in this case, $\boldsymbol{\hat \pi}^*(a|t,x) =  \exp(\frac{1}{\theta} \hat q^*(t,x,a))$ is the optimal stochastic policy.}

	\item [(ii)]  \rev{ If there exists $\boldsymbol{\pi} \in \boldsymbol{\Pi}$ such that for all $(t,x)$ and any grid $\mathbb{S}$ of $[t,T]$, the process \eqref{eq:mart-opt} is a $\mathcal{G}^{\mathbb{S}}$-martingale where $X^{\boldsymbol{\pi}, \mathbb{S}}_t=x$, then $\hat J^*$ and $\hat q^*$ are respectively the optimal value function and the optimal q-function.}
\end{enumerate}
\end{theorem}

\section{q-Learning Algorithms}\label{sec:q-algo}

In this section we present learning algorithms based on the martingale characterization of the q-function discussed in the previous section. We need to distinguish  two cases, depending on whether or not the density function of the Gibbs measure generated from the q-function can be computed and integrated explicitly.

We first discuss the case when the normalizing constant in the Gibbs measure can be computed explicitly.
We denote by $J^{\psi}$ and ${q}^{\phi}$ the parameterized function approximators for the optimal value function and optimal q-function, respectively. In view of Theorem~\ref{thm:opt-mart}, these approximators are chosen to satisfy
\begin{equation}\label{ex:constraints qv optimal}
{J}^{\psi}(T,x) = h(x),\ \int_{\mathcal{A}} \exp\Big(  \frac{1}{\theta}q^{\phi}(t,x,a)\Big) d a = 1.
\end{equation}
We can then update $(\psi,\phi)$ by enforcing the martingale condition discussed in Theorem \ref{thm:opt-mart} and applying the techniques developed  in \cite{jia2022policy}. This procedure has been discussed in details in Section 4.1 of  \cite{jia2022q}, and hence we omit the details.

For reader's convenience, we present Algorithms~\ref{algo:offline episodic} and \ref{algo:online incremental}, which  summarize the offline and online q-learning algorithms respectively. \rev{In the algorithmic implementation, we execute stochastic policies on a uniform time grid with step size $\Delta t$. That is, the grid $\mathbb{S}= \{0= t_0 < t_1 < \ldots < t_K= T\}$, where $t_{i+1} - t_i = \Delta t$ for all $i$. }
Algorithms~\ref{algo:offline episodic} and \ref{algo:online incremental} are based on the so-called martingale orthogonality condition in \cite{jia2022policy}, where the typical choices of test functions in these algorithms are $\xi_t = \partial_{\psi} J^{\psi}(t,X_t^{\bm\pi^{\phi}})$, and $\zeta_t = \partial_{\phi} q^{\phi}(t,X_t^{\bm\pi^{\phi}},a_t^{\bm\pi^{\phi}})$. \rev{Here, $\bpi^{\phi}$ is the policy generated by $q^\phi$, and we use $X_t^{\bm\pi^{\phi}}$ to denote the grid sample state process $X_t^{\bm\pi^{\phi}, \mathbb{S}}$ for notational simplicity.}
Note that these two algorithms are {\it identical} to Algorithms~2 and 3 in \cite{jia2022q}.

\begin{algorithm}[hbtp]
\caption{Offline--Episodic q-Learning Algorithm}
\textbf{Inputs}: initial state $x_0$,  horizon $T$, time step $\Delta t$, number of episodes $N$, number of mesh grids $K$, initial learning rates $\alpha_{\psi},\alpha_{\phi}$ and a learning rate schedule function $l(\cdot)$ (a function of the number of episodes), functional forms  of parameterized  value function $J^{\psi}(\cdot,\cdot)$ and  q-function $q^{\phi}(\cdot,\cdot,\cdot)$ satisfying \eqref{ex:constraints qv optimal}, functional forms of test functions $\bm{\xi}(t,x_{\cdot \wedge t},a_{\cdot \wedge t})$ and $\bm{\zeta}(t,x_{\cdot \wedge t},a_{\cdot \wedge t})$, and temperature parameter $\theta$.

\textbf{Required program (on-policy)}: environment simulator $(x',r) = \textit{Environment}_{\Delta t}(t,x,a)$ that takes current time--state pair $(t,x)$ and action $a$ as inputs and generates state $x'$ at time $t+\Delta t$ and  instantaneous reward $r$ at time $t$ as outputs. Policy $\bm\pi^{\phi}(a|t,x) = \exp(\frac{1}{\gamma}q^{\phi}(t,x,a))$.

\textbf{Required program (off-policy)}: observations $ \{a_{t_k}, r_{t_{k}}, x_{t_{k+1}}\}_{k = 0,\cdots, K-1}\cup \{ x_{t_K}, h(x_{t_K})\} = \textit{Observation}(\Delta t)$ including the observed actions, rewards, and state trajectories under the given behavior policy  at the sampling time grid with step size $\Delta t$.

\textbf{Learning procedure}:
\begin{algorithmic}
\STATE Initialize $\psi,\phi$.
\FOR{episode $j=1$ \TO $N$} \STATE{Initialize $k = 0$. Observe  initial state $x_0$ and store $x_{t_k} \leftarrow  x_0$.

\COMMENT{\textbf{On-policy case}

\WHILE{$k < K$} \STATE{
Generate action $a_{t_k}\sim \bm{\pi}^{\phi}(\cdot|t_k,x_{t_k})$.

Apply $a_{t_k}$ to environment simulator $(x,r) = Environment_{\Delta t}(t_k, x_{t_k}, a_{t_k})$, and observe new state $x$ and reward $r$ as outputs. Store $x_{t_{k+1}} \leftarrow x$ and $r_{t_k} \leftarrow r$.

Update $k \leftarrow k + 1$.
}
\ENDWHILE	

}

\COMMENT{\textbf{Off-policy case}
Obtain one observation $\{a_{t_k}, r_{t_{k}}, x_{t_{k+1}}\}_{k = 0,\cdots, K-1}\cup \{ x_{t_K}, h(x_{t_K})\} = \textit{Observation}(\Delta t)$.

}

For every $k = 0,1,\cdots,K-1$, compute and store test functions $\xi_{t_k} = \bm{\xi}(t_k, x_{t_0},\cdots, x_{t_k},a_{t_0},\cdots, a_{t_k})$, $\zeta_{t_k} = \bm{\zeta}(t_k, x_{t_0},\cdots, x_{t_k},a_{t_0},\cdots, a_{t_k})$.	

Compute
\[ \Delta\psi = \sum_{i=0}^{K-1} \xi_{t_i} \big(J^{\psi}(t_{i+1},x_{t_{i+1}}) - J^{\psi}(t_{i},x_{t_{i}}) + r_{t_i}\Delta t -q^{\phi}(t_{i},x_{t_{i}},a_{t_i})\Delta t - \beta J^{\psi}(t_{i},x_{t_{i}}) \Delta t  \big), \]
\[
\Delta\phi =   \sum_{i=0}^{K-1}\zeta_{t_i}\big( J^{\psi}(t_{i+1},x_{t_{i+1}}) - J^{\psi}(t_{i},x_{t_{i}}) + r_{t_i}\Delta t - q^{\phi}(t_{i},x_{t_{i}},a_{t_i})\Delta t - \beta J^{\psi}(t_{i},x_{t_{i}}) \Delta t\big) .
\]

Update $\psi$ and $\phi$ by
\[ \psi \leftarrow \psi + l(j)\alpha_{\psi} \Delta\psi .\]
\[ \phi \leftarrow \phi + l(j)\alpha_{\phi} \Delta\phi .  \]

}
\ENDFOR
\end{algorithmic}
\label{algo:offline episodic}
\end{algorithm}

\begin{algorithm}[hbtp]
\caption{Online-Incremental q-Learning Algorithm}
\textbf{Inputs}: initial state $x_0$, horizon $T$, time step $\Delta t$, number of mesh grids $K$, initial learning rates $\alpha_{\psi},\alpha_{\phi}$ and learning rate schedule function $l(\cdot)$ (a function of the number of episodes), functional forms  of parameterized  value function $J^{\psi}(\cdot,\cdot)$ and  q-function $q^{\phi}(\cdot,\cdot,\cdot)$ satisfying \eqref{ex:constraints qv optimal}, functional forms of test functions $\bm{\xi}(t,x_{\cdot \wedge t},a_{\cdot \wedge t})$ and $\bm{\zeta}(t,x_{\cdot \wedge t},a_{\cdot \wedge t})$, and temperature parameter $\theta$.

\textbf{Required program (on-policy)}: environment simulator $(x',r) = \textit{Environment}_{\Delta t}(t,x,a)$ that takes current time--state pair $(t,x)$ and action $a$ as inputs and generates state $x'$ at time $t+\Delta t$ and  instantaneous reward $r$ at time $t$ as outputs. Policy $\bm\pi^{\phi}(a|t,x) = \exp(\frac{1}{\theta}q^{\phi}(t,x,a))$.

\textbf{Required program (off-policy)}: observations $ \{a, r, x'\} = \textit{Observation}(t, x;\Delta t)$ including the observed actions, rewards, and state when the current time-state pair is $(t,x)$ under the given behavior policy at the sampling time grid with step size $\Delta t$.

\textbf{Learning procedure}:
\begin{algorithmic}
\STATE Initialize $\psi,\phi$.
\FOR{episode $j=1$ \TO $\infty$} \STATE{Initialize $k = 0$. Observe  initial state $x_0$ and store $x_{t_k} \leftarrow  x_0$.
\WHILE{$k < K$} \STATE{
\COMMENT{\textbf{On-policy case}

Generate action $a_{t_k}\sim \bm{\pi}^{\phi}(\cdot|t_k,x_{t_k})$.

Apply $a_{t_k}$ to  environment simulator $(x,r) = Environment_{\Delta t}(t_k, x_{t_k}, a_{t_k})$, and observe new state $x$ and reward $r$ as outputs. Store $x_{t_{k+1}} \leftarrow x$ and $r_{t_k} \leftarrow r$.

}

\COMMENT{\textbf{Off-policy case}

Obtain one observation $a_{t_k}, r_{t_k}, x_{t_{k+1}} = \textit{Observation}(t_k, x_{t_k};\Delta t)$.

}

Compute test functions $\xi_{t_k} = \bm{\xi}(t_k, x_{t_0},\cdots, x_{t_k},a_{t_0},\cdots, a_{t_k})$, $\zeta_{t_k} = \bm{\zeta}(t_k, x_{t_0},\cdots, x_{t_k},a_{t_0},\cdots, a_{t_k})$.	

Compute
\[ \begin{aligned}
& \delta = J^{\psi}(t_{k+1},x_{t_{k+1}}) - J^{\psi}(t_{k},x_{t_{k}}) + r_{t_k}\Delta t -q^{\phi}(t_{k},x_{t_{k}},a_{t_k})\Delta t - \beta J^{\psi}(t_{k},x_{t_{k}}) \Delta t ,\\
& \Delta\psi = \xi_{t_k} \delta, \\
& \Delta\phi = \zeta_{t_k} \delta.
\end{aligned} \]

Update $\psi$ and $\phi$ by
\[ \psi \leftarrow \psi + l(j)\alpha_{\psi} \Delta\psi.\]
\[ \phi \leftarrow \phi + l(j)\alpha_{\phi} \Delta\phi .  \]

Update $k \leftarrow k + 1$
}
\ENDWHILE	

}
\ENDFOR
\end{algorithmic}
\label{algo:online incremental}
\end{algorithm}

When the normalizing constant in the Gibbs measure is not available, we take the same approach as in \cite{jia2022q} to develop learning algorithms. Specifically, we consider $\{\boldsymbol{\pi}^{\phi}(\cdot|t,x)\}_{\phi \in \Phi}$, which is a family of density functions of some tractable distributions, e.g. multivariate normal distributions. Starting from a stochastic policy $\boldsymbol{\pi}^{\phi}$ in this family, we update the policy by considering the optimization problem
\begin{align*}
\min_{\phi' \in \Phi} \text{KL} \bigg( \boldsymbol{\pi}^{\phi'}(\cdot|t,x)  \Big| \exp \Big(\frac{1}{\theta} q(t,x, \cdot; \boldsymbol{\pi}^{\phi}) \Big) \bigg).
\end{align*}
Specifically, using gradient descent, we can update $\phi$ as in  \cite{jia2022q}, by
\begin{align}\label{eq:update-un}
\phi \leftarrow \phi - \theta \alpha_{\phi} dt \Big(\log \boldsymbol{\pi}^{\phi}( a_t^{ \boldsymbol{\pi}^{\phi} }|t, X_{t-}^{ \boldsymbol{\pi}^{\phi} } ) -   \frac{1}{\theta}  q(t,X_t^{ \boldsymbol{\pi}^{\phi} }, a_t^{ \boldsymbol{\pi}^{\phi} }; \boldsymbol{\pi}^{\phi}) \Big) \partial_\phi\log \boldsymbol{\pi}^{\phi}( a_t^{ \boldsymbol{\pi}^{\phi} }|t, X_{t-}^{ \boldsymbol{\pi}^{\phi} }).
\end{align}
In the above updating rule, we need only the {\it values} of the q-function along the trajectory -- the ``data" -- $\{(t,X_t^{ \boldsymbol{\pi}^{\phi} }, a_t^{ \boldsymbol{\pi}^{\phi} }); 0 \le t \le T\}$, instead of its full functional form. These values  can be learned through the ``temporal difference" of the value function along the data. To see this, applying It\^{o}'s formula \eqref{eq:Ito-JD} to $J(\cdot, \cdot; \boldsymbol{\pi}^{\phi} )$, we have
\begin{align}
q(t, X_t^{ \boldsymbol{\pi}^{\phi} }, a_t^{ \boldsymbol{\pi}^{\phi} }; \boldsymbol{\pi}^{\phi}) dt & = dJ(t, X_t^{ \boldsymbol{\pi}^{\phi} }; \boldsymbol{\pi}^{\phi}) + \big(r(t, X_t^{ \boldsymbol{\pi}^{\phi} }, a_t^{ \boldsymbol{\pi}^{\phi} }) -\beta  J(t, X_t^{ \boldsymbol{\pi}^{\phi} }; \boldsymbol{\pi}^{\phi}) \big)dt  \\
	& \quad +     \partial_xJ(t, X_{t-}^{ \boldsymbol{\pi}^{\phi} }; \boldsymbol{\pi}^{\phi})\circ{\sigma} (t, X_{t-}^{ \boldsymbol{\pi}^{\phi} }, a_t^{ \boldsymbol{\pi}^{\phi} })dW_t \nonumber  \\
	& \quad +   \int_{\mathbb{R}^d} \big( J(t,   X^{\boldsymbol{\pi}^{\phi}}_{t-} + {\gamma}( t,X^{\boldsymbol{\pi}^{\phi}}_{t-} , a_t^{\boldsymbol{\pi}^{\phi}},z)) -  J(t,X^{\boldsymbol{\pi}^{\phi} }_{t-}; \boldsymbol{\pi}^{\phi}) \big)  \widetilde{ N}(dt, dz) ,
\end{align}
\rev{where we also use $a_t^{ \boldsymbol{\pi}^{\phi}}$ to denote the action process $a_t^{ \boldsymbol{\pi}^{\phi}, \mathbb{S}}$ for notational simplicity.}
We may ignore the $dW_t$ and $\widetilde{ N}(dt, dz)$ terms which are martingale differences with mean zero, and then the updating rule in \eqref{eq:update-un} becomes
\begin{align}
&\phi \leftarrow \phi  +  \alpha_{\phi} \Big(dJ(t, X_t^{ \boldsymbol{\pi}^{\phi} }; \boldsymbol{\pi}^{\phi}) + \big(r(t, X_t^{ \boldsymbol{\pi}^{\phi} }, a_t^{ \boldsymbol{\pi}^{\phi} }) -\beta  J(t, X_t^{ \boldsymbol{\pi}^{\phi} }; \boldsymbol{\pi}^{\phi})\big)dt -\theta \log \boldsymbol{\pi}^{\phi}( a_t^{ \boldsymbol{\pi}^{\phi}}|t, X_{t-}^{ \boldsymbol{\pi}^{\phi} })dt \Big)\\
& \qquad \qquad \qquad \cdot \partial_{\phi}\log \boldsymbol{\pi}^{\phi}( a_t^{ \boldsymbol{\pi}^{\phi} }|t, X_{t-}^{ \boldsymbol{\pi}^{\phi} }).
\end{align}
Using $J^{\psi}(\cdot, \cdot)$ as the parameterized function approximator for $J(\cdot, \cdot ; \boldsymbol{\pi}^{\phi})$, we arrive at the updating rule for the policy parameter $\phi$:
\begin{align}
&\phi \leftarrow \phi  +  \alpha_{\phi}\Big(dJ^{\psi}(t, X_t^{ \boldsymbol{\pi}^{\phi} }) + \big(r(t, X_t^{ \boldsymbol{\pi}^{\phi} }, a_t^{ \boldsymbol{\pi}^{\phi} }) -\beta  J^{\psi}(t, X_t^{ \boldsymbol{\pi}^{\phi} }) \big)dt-\theta\log\boldsymbol{\pi}^{\phi}( a_t^{\boldsymbol{\pi}^{\phi} }|t, X_{t-}^{ \boldsymbol{\pi}^{\phi} })dt\Big) \\
& \qquad \qquad \qquad \cdot \partial_\phi\log \boldsymbol{\pi}^{\phi}( a_t^{ \boldsymbol{\pi}^{\phi} }|t, X_{t-}^{ \boldsymbol{\pi}^{\phi} }). \label{eq:update-un2}
\end{align}
Therefore, we can update $\psi$ using the PE methods in \cite{jia2022policy}, and update $\phi$ using the above rule, leading to actor--critic type of algorithms.

To conclude, we are able to use the {\it same} RL algorithms to learn the optimal
policy and optimal value function, {\it without having to know a priori whether the unknown environment entails a pure diffusion or a jump-diffusion.} This important conclusion is based on the theoretical analysis carried out in the previous sections.

\section{Application: Mean--Variance Portfolio Selection}\label{sec:mean-var}
We now present an applied example of the general theory and algorithms derived. Consider investing in a market where there are a risk-free asset and a risky asset (e.g. a stock or an index). The risk-free rate is given by a constant $r_f$ and the risky asset price process follows
\begin{equation}
	dS_t=S_{t-}\bigg(\mu dt+ \sigma dW_t + \int_{\mathbb{R}} (\exp(z)-1) \widetilde{N}(dt, dz)\bigg).\label{eq:stock-SDE}
\end{equation}
Let $X_t$ be the discounted wealth value at time $t$, and $a_t$ is the discounted dollar value of the investment in the risky asset. The self-financing discounted wealth process follows
\begin{equation}\label{eq:wealth-det}
	dX^{a}_t = a_t\sigma\rho dt + a_t\sigma dW_t + a_t \int_{\mathbb{R}} (\exp(z)-1)\widetilde{N}(dt, dz),
\end{equation}
where $\rho$ is the Sharpe ratio of the risky asset and given by
\begin{equation}\label{eq:rho}
\rho=\frac{\mu-r_f}{\sigma}.
\end{equation}
We assume
\begin{align}
	\int_{|z|>1}\exp(2z)\nu(dz)<\infty. \label{eq:exp-nu}
\end{align}
Condition \eqref{eq:exp-nu} implies that $\mathbb{E}[S_t]$ and $\mathbb{E}[S_t^2]$ are finite for every $t\geq 0$; see \cite[Proposition 3.14]{tankov2004financial}. We set
\begin{equation}
\sigma_J^2\coloneqq\int_{\mathbb{R}} (\exp(z)-1)^2\nu(dz),\label{eq:sigma-J}	
\end{equation}
which is finite by conditions \eqref{eq:levy-integ-cond} and \eqref{eq:exp-nu}.

Fix the investment horizon as $[0,T]$. The MV portfolio selection problem considers
\begin{equation}
	\min_{a}\operatorname{Var}\left[X_{T}^{a}\right]\quad \text { subject to } \mathbb{E}\left[X_{T}^{a}\right]=z.
\end{equation}
We seek the optimal pre-committed strategy for the MV problem as in \cite{zhou2000continuous}. We can transform the above constrained problem into an unconstrained one by introducing a Lagrange multiplier, which yields
\begin{equation}
	\min_{a} \mathbb{E}[(X_{T}^{a})^{2}]-z^{2}-2 \omega(\mathbb{E}[X_{T}^{a}]-z)=\min_{a} \mathbb{E}[(X_{T}^{a}-\omega)^{2}]-(\omega-z)^{2}.
	\label{eq:unconstrained}
\end{equation}
Note that the optimal solution to the unconstrained minimization problem depends on $\omega$, and we can obtain the optimal multiplier $\omega^*$ by solving $\mathbb{E}\left[X_{T}^{a^{*}}(\omega)\right]=z$, \rev{which is given by
\begin{equation}
	\omega^* = \bigg(z\exp\bigg(\frac{\rho^2\sigma^2}{\sigma^2+\sigma_J^2}T\bigg) - x_0\bigg)\bigg/\bigg(\exp\bigg(\frac{\rho^2\sigma^2}{\sigma^2+\sigma_J^2}T\bigg) -1\bigg), \label{eq:opt-Lagrange}
\end{equation}
where $x_0$ is the initial wealth. The optimal deterministic control policy is given by
\begin{equation}\label{eq:MV-port-opt-det-policy}
	a^*(t,x)=-\frac{\sigma\rho}{\sigma^2+\sigma_J^2}(x-\omega^*).
\end{equation}
}

The exploratory formulation of the control problem is
\begin{align} \label{eq:MV-expl}
	\min_{\bpi}\mathbb{E}\bigg[(\tilde{X}_{T}^{\bpi}-\omega)^{2} + \theta \int_0^T \int_{\mathcal{A}}\log \bpi(a|t, \tilde{X}_{t-}^{\bpi}) \bpi(a|t, \tilde{X}_{t-}^{\bpi})dadt \bigg]-(\omega-z)^{2},
\end{align}
where the exploratory discounted wealth under a stochastic policy $\bpi$ follows \rev{
\begin{align}
	d\tilde{X}^{\bpi}_t &= \sigma\rho \int_{\mathcal{A}}a\bpi(a|t, \tilde{X}^{\bpi}_{t-})da dt + \sigma\sqrt{\int_{\mathcal{A}}a^2\bpi(a|t, \tilde{X}^{\bpi}_{t-})da} dW_t\\
	& \phantom{=} + \int_{\mathbb{R}\times[0,1]^n} G^{\bpi}(t,\tilde{X}^{\bpi}_{t-},u)(\exp(z)-1)\widetilde N'(dt, dz, du),\quad t\in[0,T].\label{eq:wealth-expl}
\end{align}
As will be explained shortly, we only need to consider stochastic feedback policies of the form $\mathcal{N}(\cdot|\alpha(t)(x-\beta(t)),\gamma(t))$, where $\alpha(t), \beta(t)$ and $\gamma(t)$ are continuous functions of $t$. Such a policy $\bpi$ clearly fulfills the requirements in Definition \ref{def:policy} and hence is admissible. It also satisfies condition (ii) in Assumption \ref{assump:jump}. }

\rev{
We next discuss the well-posedness of the exploratory SDE \eqref{eq:wealth-expl}. While Assumptions \ref{assump:SDE} and \ref{assump:jump} are imposed to ensure the moment estimate \eqref{eq:moment-expl} holds for any $p\geq 2$ in the general setting, we only need such an estimate for the second moment of the exploratory state in the particular MV application. To have such a property, we only need to require the conditions involving $p$ in Assumptions \ref{assump:SDE} and \ref{assump:jump} to hold up to $p=2$, which are clearly satisfied in the problem under \eqref{eq:exp-nu}. The other requirements in Assumption \ref{assump:SDE} are also fulfilled.  Thus, the exploratory SDE \eqref{eq:wealth-expl} is well-posed for the $\bpi$ mentioned above and Eq. \eqref{eq:moment-expl} holds for $p=2$. }

\subsection{Solution of the exploratory control problem}
We consider the HJB equation for problem \eqref{eq:MV-expl}:
\begin{align}
	\partial_t V(t,x) +   \inf_{ \boldsymbol{\pi} \in \mathcal{P} (\mathbb{R})}\int_{\mathbb{R}} \big( H(t,x, a, \partial_xV, \partial_{x}^2V, V) + \theta \log \bpi(a | t, x)\big) \bpi(a | t, x) da = 0 \label{eq:HJB-MV},
\end{align}
with the terminal condition $V(T, x) = (x-\omega)^2 - (\omega-z)^2$. Note that  supremum becomes infimum and the sign before $\theta\log\bpi(a | t, x)$ flips compared with \eqref{eq:HJB-explore} because we consider minimization here. The Hamiltonian of the problem is given by
\begin{align}
	H(t,x, a, \partial_xV,  \partial_{x}^2V, V)&= a\sigma\rho \partial_xV(t,x) + \frac{1}{2} a^2\sigma^2 \partial_x^2V(t,x)\\
	&\phantom{=~}+\int_{\mathbb{R}} \big( V(t, x+  \gamma(a, z)) - V(t,x) - \gamma(a,z)\partial_xV(t,x) \big) \nu(dz),\label{eq:Hal-MV}
\end{align}
where $\gamma(a, z)= a(e^z-1).$
We take the following ansatz for the solution of the HJB equation \eqref{eq:HJB-MV}:
\begin{equation}\label{eq:ansatz-MV}
	V(t,x)=(x-\omega)^2 f(t) + g(t) - (\omega-z)^2.
\end{equation}
As $V(t,x)$ is quadratic in $x$, we can easily calculate the integral term in the Hamiltonian and obtain
\begin{equation}
	H(t,x, a, \partial_xV,  \partial_x^2V, V)= a\sigma\rho \partial_xV(t,x) + \frac{1}{2} a^2(\sigma^2 + \sigma_J^2) \partial_x^2V(t,x).\label{eq:Hal-MV-2}
\end{equation}
The probability density function that minimizes the integral in \eqref{eq:HJB-MV} is given by
\begin{equation}
	\bpi_c(\cdot|t,x)\propto \exp \Big(-\frac{1}{\theta} H(t,x, a, \partial_xV,  \partial_{x}^2V, V)  \Big),
\end{equation}
which is a candidate for the optimal stochastic policy.
From \eqref{eq:Hal-MV-2}, we obtain
\begin{equation}
\bpi_c(\cdot|t,x)\sim\mathcal{N}\Big(\cdot \Big| -\frac{\sigma\rho \partial_xV}{(\sigma^2+\sigma_J^2)\partial_x^2V}, \frac{\theta}{(\sigma^2+\sigma_J^2)\partial_x^2V}\Big).
\end{equation}
Substituting it back to the HJB equation \eqref{eq:HJB-MV}, we obtain a nonlinear PDE as
\begin{align}
	&\partial_tV - \frac{\rho^2\sigma^2}{2(\sigma^2+\sigma_J^2)}\frac{(\partial_xV)^2}{\partial_x^2V}-\frac{\theta}{2}\ln\frac{2\pi\theta}{(\sigma^2+\sigma_J^2)\partial_x^2V}=0,\quad (t,x)\in[0,T)\times\mathbb{R},\\
	&V(T,x)=(x-\omega)^2-(\omega-z)^2.
\end{align}
We plug in the ansatz \eqref{eq:ansatz-MV} to the above PDE and obtain that $f(t)$ satisfies
\begin{align}
f'(t)-\frac{\rho^2\sigma^2}{\sigma^2+\sigma_J^2}f(t)=0,\ f(T)=1,
\end{align}
and $g(t)$ satisfies
\begin{align}
	g'(t)-\frac{\theta}{2}\ln\frac{\pi\theta}{(\sigma^2+\sigma_J^2)f(t)}=0,\ g(T)=0.
\end{align}
These two ordinary differential equations can be solved analytically, and we obtain
\begin{align}
	V(t,x)=&(x-\omega)^2\exp\bigg(-\frac{\rho^2\sigma^2}{\sigma^2+\sigma_J^2}(T-t)\bigg) + \frac{\theta\rho^2\sigma^2}{4(\sigma^2+\sigma_J^2)}(T^2-t^2)\\
	&-\frac{\theta}{2}\bigg(\frac{\rho^2\sigma^2}{\sigma^2+\sigma_J^2}T+\ln\frac{\pi\theta}{\sigma^2+\sigma_J^2}\bigg)(T-t) - (\omega-z)^2.\label{eq:J-optimal-MV}
\end{align}
It follows that
\begin{equation}\label{eq:policy-optimal-MV}
	\bpi_c(\cdot|t,x)\sim\mathcal{N}\bigg(\cdot\ \Big| -\frac{\sigma\rho}{\sigma^2+\sigma_J^2}(x-\omega), \frac{\theta}{2(\sigma^2+\sigma_J^2)}\exp\bigg(\frac{\rho^2\sigma^2}{\sigma^2+\sigma_J^2}(T-t)\bigg)\bigg),
\end{equation}
and it is admissible. Furthermore, $V$ solves the HJB equation \eqref{eq:HJB-MV} and shows quadratic growth. Therefore, by Lemma \ref{lem:expHJB}, we have the following conclusion.
\begin{proposition}\label{prop:MV-opt}
	For the unconstrained MV problem \eqref{eq:unconstrained}, the optimal value function $J^*(t,x)=V(t,x)$ and the optimal stochastic policy $\bpi^*=\bpi_c$.
\end{proposition}

When there is no jump, we have $\sigma_J^2=0$ and thus recover the expressions of the optimal value function and optimal policy  derived in \cite{wang2020continuous} for the unconstrained MV problem in the pure diffusion setting.

\subsection{Parametrizations for q-learning}
It is important to observe that the optimal value function, optimal policy and the Hamiltonian \eqref{eq:Hal-MV-2} take the same {\it structural} forms regardless of the presence of jumps, while the only differences are the constant coefficients in those functions. However, those coefficients are unknown anyway and will be parameterized in the implementation of our RL algorithms.  Consequently,  we can use the same parameterizations for the optimal value function and optimal q-function for learning as in the diffusion setting of \cite{jia2022q}. This important insight, concluded only {\it after} a rigorous theoretical analysis, shows that the continuous-time RL algorithms are robust to the presence of jumps and essentially model-free, at least for the MV portfolio selection problem.

Following \cite{jia2022q}, we parametrize the value function as
\begin{equation}
	J^\psi(t, x ; \omega)=(x-\omega)^2 e^{-\psi_3(T-t)}+\psi_2(t^2-T^2)+\psi_1(t-T)-(\omega-z)^2,\label{eq:MV-vf}
\end{equation}
and the q-function as
\begin{equation}
	q^\phi(t, x, a ;\omega)=-\frac{e^{-\phi_2-\phi_3(T-t)}}{2}\big(a+\phi_1(x-\omega)\big)^2-\frac{\theta}{2}\big(\log(2 \pi \theta)+\phi_2+\phi_3(T-t)\big).\label{eq:MV-qf}
\end{equation}
Let $\psi=\left(\psi_1, \psi_2, \psi_3\right)^{\top}$ and $\phi=\left(\phi_1, \phi_2, \phi_3\right)^{\top}$.
The policy associated with the parametric q-function is ${\bpi}^\phi(\cdot \mid t, x ; \omega)=\mathcal{N}(-\phi_1(x-\omega), \theta e^{\phi_2+\phi_3(T-t)})$. In addition to $\psi$ and $\phi$, we learn the Lagrange multiplier $\omega$ in the same way as in \cite{jia2022q} by the stochastic approximation algorithm that updates $\omega$ with a learning rate after a fixed number of iterations.

\subsection{Simulation study}
We assess the effect of jumps on the convergence behavior of our algorithms via a simulation study. We assume the risk-free rate is zero and use the same basic setting as in \cite{jia2022q}: $x_0 = 1$, $z = 1.4$, $T = 1$ year, $\Delta t = 1/252$ years (corresponding to one trading day), and a chosen temperature parameter $\theta = 0.1$. We consider two market simulators: one is given by the Black--Scholes (BS) model and the other is Merton's jump-diffusion (MJD) model in which the L\'evy density is a scaled Gaussian density, i.e.,
\begin{equation}\label{eq:Levy-den-Gauss}
	\nu(z)=\lambda\frac{1}{\sqrt{2\pi\delta^2}}\exp\bigg(-\frac{(z-m)^2}{2\delta^2}\bigg),\ \lambda>0, \delta>0, m\in\mathbb{R},
\end{equation}
where $\lambda$ is the arrival rate of the Poisson jumps. The Gaussian assumption is a common choice in the finance literature for modeling the jump-size distribution (see e.g. \citealt{merton1976option}, \citealt{bates1991crash}, \citealt{das2002surprise}), partly due to its tractability for statistical estimation and partly because heavy-tailed distributions may not be easily identified from real data when the number of jumps is limited (see \citealt{heyde2004controversy}).

Under the latter model, we have
\begin{equation}
	\sigma_J^2=\lambda\bigg(\exp(2m+2\delta^2)-2\exp\bigg(m+\frac{1}{2}\delta^2\bigg)+1\bigg).
\end{equation}
To mimic the real market, we set the parameters of these two simulators by estimating them from the daily data of the S\&P 500 index using maximum likelihood estimation (MLE). Our estimation data cover a long period from the beginning of 2000 to the end of 2023. In Table \ref{tab:MVPS-params}, we summarize the estimated parameter values used for the simulators.

\begin{table}[h]
	\centering
	\begin{tabular}{ccc}
		\hline
		\textbf{Simulator} & \textbf{Parameters} \\ \hline
		BS & $\mu = 0.0690, \sigma = 0.1965$  \\ \hline
		MJD & $\mu = 0.0636, \sigma = 0.1347, \lambda = 28.4910, m = -0.0039, \delta = 0.0275$ \\\hline
	\end{tabular}
	\caption{Parameters used in the two simulators}\label{tab:MVPS-params}
\end{table}

For offline learning (Algorithm 1), the Lagrange multiplier $\omega$ is updated after every 10 iterations, and the parameter vectors $\psi$ and $\phi$ are initialized as zero vectors. \rev{For the learning rates, we set $\alpha_\omega=0.05$, $\alpha_\psi=0.001$, $\alpha_\phi=(0.12,1.5,0.15)$ for the BS simulator, and $\alpha_\omega=0.06$, $\alpha_\psi=0.001$, $\alpha_\phi=(0.06,1.3,0.11)$ for the MJD one}. We apply a decay rate $l(j)=j^{-0.51}$ to the learning rates, where $j$ is the iteration index. In each iteration, we generate $32$ independent $T$-year trajectories to update the parameters. We train the model for $N=20000$ iterations.

We also consider online learning (Algorithm 2) with $\Delta t$ equal to one trading day. We select a batch size of $128$ trading days and update the parameters after having obtained this number of observations. We update $\omega$ after collecting 252 trading days of observations and initialize $\psi$ and $\phi$ as zero vectors. \rev{For the learning rates, we set $\alpha_\omega=0.005$, $\alpha_\psi=0.001$, $\alpha_\phi=(0.12,2.9,0.16)$ for the BS simulator, and $\alpha_\omega=0.01$, $\alpha_\psi=0.001$, $\alpha_\phi=(0.027,1.05,0.45)$ for the MJD one.} We apply the same decay scheme as the offline case. The model is again trained for $N=20000$ iterations.

\rev{We display in Table \ref{tab:phi_convergence} the learned results for the parameters $\phi_1,\phi_2$ and $\phi_3$ as well as the true values of these parameters calculated by plugging the BS/MJD model parameters into \eqref{eq:policy-optimal-MV}. For both offline and online learning, the learned results of the four parameters are close to their true values. In practice, learning $\phi_1$ and $\omega$ (parameters in the policy's mean) is more important than $\phi_2$ and $\phi_3$ (parameters in the policy's variance). Although we use a stochastic policy to interact with the environment during training to update the policy parameters, for actual execution of portfolio selection we apply a deterministic policy which is the \emph{mean} part of the optimal stochastic policy after it has been learned. This renders an  {\it off-policy} learning. An advantage of doing so among others is to reduce the variance of the final wealth; see \cite{huang2022achieving} for a discussion. Furthermore, the deterministic policy given by \eqref{eq:MV-port-opt-det-policy} is the one that optimizes the original MV objective (without entropy regularization). Hence, it makes sense to apply the mean part of the learned stochastic policy for execution, which theoretically is identical to the optimal deterministic policy (compare \eqref{eq:MV-port-opt-det-policy} and \eqref{eq:policy-optimal-MV}). }

\begin{table}[htbp!]
	\centering
	\begin{tabular}{lcccc}
	\toprule
	& $\phi_1$ & $\phi_2$ & $\phi_3$ & $\omega$ \\
	\midrule
	BS simulator & & & & \\
	\midrule
	True & 1.7869 & 2.5610 & 0.1233 & 4.4481 \\
	RL-Offline & 1.7613 & 2.5651 & 0.1241 & 4.4886 \\
	RL-Online & 1.7693 & 2.6158 & 0.1217 & 4.3864 \\
	\midrule
	MJD simulator & & & & \\
	\midrule
	True & 1.5940 & 2.5282 & 0.1013 & 5.1489 \\
	RL-Offline & 1.5871 & 2.5451 & 0.1043 & 5.1776 \\
	RL-Online & 1.5873 & 2.5050 & 0.0998 & 5.1840 \\
	\bottomrule
	\end{tabular}
	\caption{Learned parameters after 20000 iterations from the offline and online q-learning algorithms under two market simulators}
	\label{tab:phi_convergence}
\end{table}

\rev{Figure \ref{fig:qlearning_updated} plots the learning curves of $\phi_1$ for both  offline and online learning and under both simulators (one with jumps and one without). All the curves eventually converge, whether jumps are present or not. This demonstrates that convergence of the offline and online q-learning algorithms proposed in \cite{jia2022q} under the diffusion setting is robust to the presence of jumps for mean--variance portfolio selection.} 

\begin{figure}[h]
	\centering
	\includegraphics[width=1\textwidth]{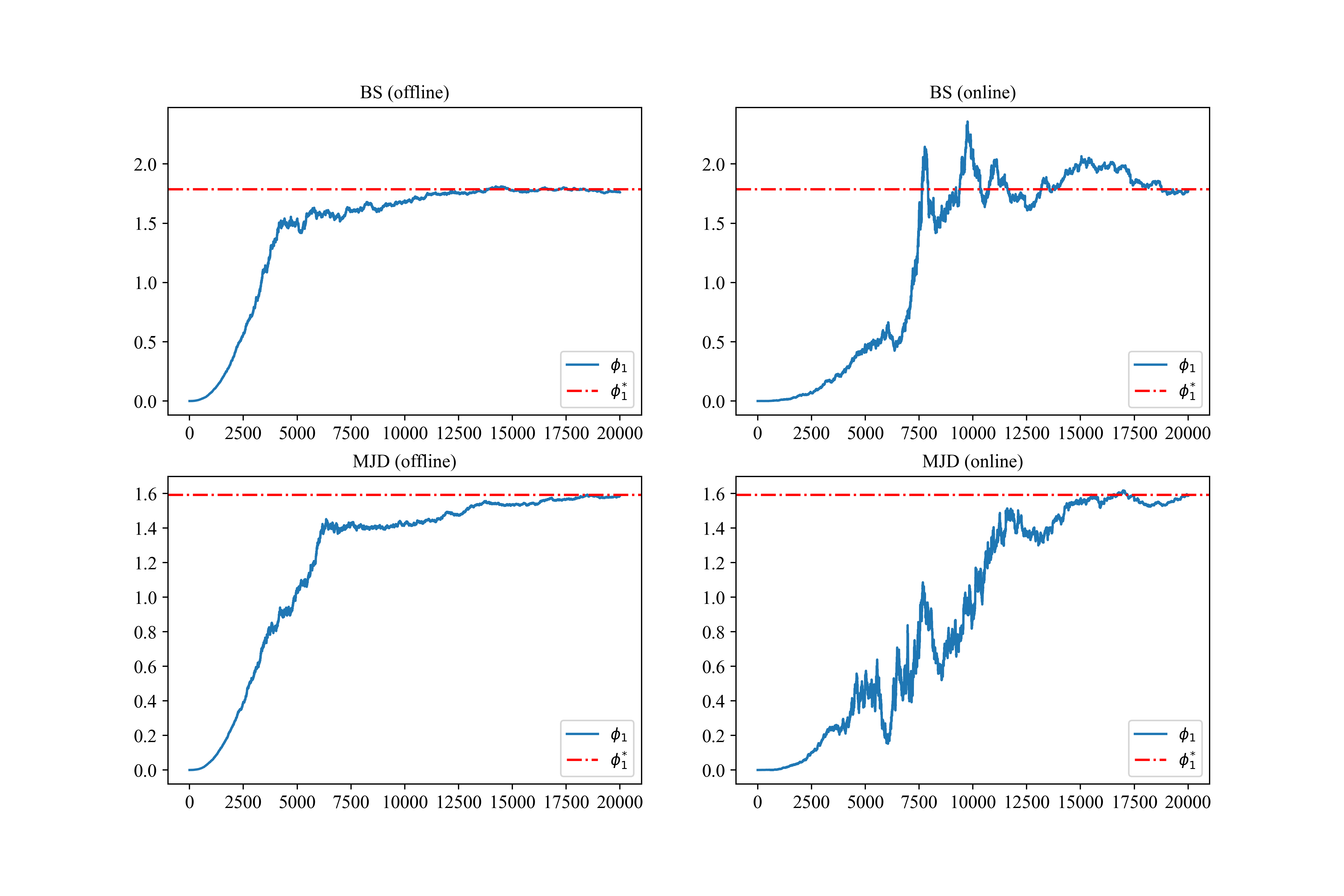}
	\caption{Convergence of the offline and online q-learning algorithms under two market simulators for the policy parameter $\phi_1$ (iteration index on the $x$-axis)}
	\label{fig:qlearning_updated}
\end{figure}

\subsection{Empirical Study}\label{sec:MV-Port-Emp}
In this section we report an empirical study on investing in the  S\&P 500 index for one year.  We obtain daily observations of the index value from the beginning of 2000 to the end of 2023, and split them into two periods: the training period (the first 20 years) and the test period (the last 4 years). Recall that the test period includes the U.S. stock market meltdown at the onset of the
pandemic in 2020, the 2022 bear period as well as various bull subperiods.

We assume the risk-free rate is zero and perform daily rebalancing. We compare our offline and online RL algorithms with the plug-in approach based on MLE estimates. In the latter, we consider both BS and MJD models and fit them to daily returns of the index in the training period. The MLE estimates are plugged into \eqref{eq:opt-Lagrange} to obtain the respective optimal Lagrange multipliers and into \eqref{eq:MV-port-opt-det-policy} to obtain the corresponding optimal deterministic policies, which are then applied on the test data. To generate episodes for training the two RL algorithms, we bootstrap index returns in the training period to create one-year episodes of index prices. We initialize the policy parameter $\phi_1$ by using the MLE estimates of the MJD model parameters in the formula $\sigma\rho/(\sigma^2+\sigma_J^2)$. The implementation of RL follows the same practice in our simulation study and the mean part of the learned stochastic policy is applied on the test data. To create the test data, we generate a set of 5000 one-year episodes of index prices by bootstrapping index returns in the test period, which is used for testing the RL algorithms and the plug-in method.

Table \ref{tab:testSPY} compares the out-of-sample test performances of the four methods with different levels of the target wealth. As a natural benchmark, we also report the annualized mean return, volatility and Sharpe ratio of the index in the 4-year test period, which corresponds to a buy-and-hold strategy. It is clear
that in terms of the Sharpe ratio, all the methods beat the index yet the two  RL algorithms outperform the MLE method significantly. 
A closer examination shows that this outperformance arises from RL's decisive lower volatilities in all the cases. Indeed, in the case when
$z=1.1$ representing a (reasonably targeted) 10\% expected annual return, both RL strategies' volatilities are smaller than that of the index, and the online one is less than half of those of the MLE!

\color{black}

\begin{table}[htbp!]
	\centering
	\begin{tabular}{lccccccc}
		\toprule
		& $\phi_1$  & $\omega$& MR & Vol & SR \\
		\midrule
		$z=1.4$ & & & &  \\
		\cmidrule(lr){1-1}
		RL-Offline &  3.0507  & 2.8560  & 0.6823  & 0.7765  & 0.8788  \\
		RL-Online   &2.7079  & 2.5421  & 0.5062  & 0.6011  & 0.8422 \\
		MLE-BS  & 1.6203  & 5.4748  & 0.8938  & 1.2038  & 0.7425 \\
		MLE-MJD & 1.3877  & 6.7992  & 0.9954  & 1.3752  & 0.7238\\
		\midrule
		$z=1.3$ & & & &  \\
		\cmidrule(lr){1-1}
		RL-Offline & 2.8040  & 2.5218  & 0.5165  & 0.6061  & 0.8522\\
		RL-Online  &2.6593  & 2.1760  & 0.3794  & 0.4532  & 0.8373  \\
		MLE-BS & 1.6203  & 4.3561  & 0.6704  & 0.9029  & 0.7425\\
		MLE-MJD &1.3877  & 5.3494  & 0.7466  & 1.0314  & 0.7238 \\
		\midrule
		$z=1.2$ & & & &  \\
		\cmidrule(lr){1-1}
		RL-Offline& 3.0599  & 1.9130  & 0.3366  & 0.3826  & 0.8798   \\
		RL-Online  & 2.4275  & 1.8672  & 0.2564  & 0.3148  & 0.8143\\
		MLE-BS & 1.6203  & 3.2374  & 0.4469  & 0.6019  & 0.7425\\
		MLE-MJD  & 1.3877  & 3.8996  & 0.4977  & 0.6876  & 0.7238 \\
		\midrule
		$z=1.1$ & & & &  \\
		\cmidrule(lr){1-1}
		RL-Offline& 2.9130  & 1.4729  & 0.1664  & 0.1926  & 0.8637\\
		RL-Online & 2.5111  & 1.4031  & 0.1231  & 0.1497  & 0.8224  \\
		MLE-BS &1.6203  & 2.1187  & 0.2234  & 0.3010  & 0.7425\\
		MLE-MJD & 1.3877  & 2.4498  & 0.2489  & 0.3438  & 0.7238 \\
		& & & & & \\
		Buy-and-hold & N.A. & N.A. & 0.0984 & 0.2312 & 0.4256 \\
		\bottomrule
	\end{tabular}
	\caption{Mean return (MR), volatility (Vol), and Sharpe ratio (SR) of investment portfolios obtained by four methods on $5000$ test episodes. We show the learned $\phi_1$ and $\omega$ under each approach. For MLE-BS and MLE-MJD, $\phi_1$ and $\omega$ are calculated by plugging the MLE estimates to $\phi_1=\sigma\rho/(\sigma^2+\sigma_J^2)$ and ~\protect\eqref{eq:opt-Lagrange}, respectively. } \label{tab:testSPY}
\end{table}

\subsection{Effects of jumps} \label{sec:jump-effect}
The theoretical analysis so far in this section shows that, {\it for the MV portfolio selection problem}, one does not need to know in advance whether or not the stock prices have jumps in order to carry out the RL task, because the optimal stochastic policy is Gaussian and the corresponding value function and q-function have the same structures for parameterization irrespective of the presence of jumps. However, we stress that this is rather an exception than a rule. Here we give a counterexample.

Consider a modification of the MV portfolio selection problem where the controlled state follows
\begin{align} \label{eq:exdy2}
	dX_t^a = a_t \sigma\rho dt + a_t\sigma dB_t + \int_{\mathbb{R}} \gamma(a_t,z)\widetilde{N}(dt, dz),
\end{align}
with
\begin{align} \label{eq:gamma-a2}
	\gamma(a,z) = a^2.
\end{align}
The exploratory value function is given by
\begin{align} \label{eq:exploratory-MV2}
	J^*(t, x; w) =  \min_{\bpi}\mathbb{E}_{t,x}\bigg[ (\tilde{X}_{T}^{\bpi} -w)^2  + \theta  \int_t^T\int_{\mathcal{A}}\log\bpi(a|s, \tilde{X}_{s-}^{\bpi})\bpi(a|s, \tilde{X}_{s-}^{\bpi})dads\bigg] - (\omega-z)^{2}.
\end{align}
Note that this is {\it not} an MV portfolio selection problem because \eqref{eq:exdy2} does not correspond to a self-financed wealth equation with a reasonably modelled stock price process.

The Hamiltonian is given by
\begin{align}\label{eq:HalMV2}
	H(t,x, a,\partial_xv,  \partial_x^2v, v)&= a\sigma\rho \partial_xv(t,x) + \frac{1}{2} a^2\sigma^2 \partial_x^2v(t,x)\\ & \quad +\int_{\mathbb{R}} \big( v(t, x+  a^2) - v(t,x) - a^2 \cdot \partial_xv(t,x) \big) \nu(dz).
\end{align}
If an optimal stochastic policy exists, then it must be
\begin{align} \label{eq:pi-star-MV}
	\boldsymbol{\pi}^*(a| t, x) \propto \exp\bigg(-\frac{1}{\theta} H(t,x, a, \partial_xJ^*,  \partial_x^2J^*, J^*)\bigg).
\end{align}
We show by contradiction that the optimal stochastic policy can not be Gaussian in this case. Note that if there is no optimal stochastic policy, then it would already demonstrate that jumps matter because the optimal stochastic policy for the case of no jumps exists and is Gaussian.

\begin{remark}
The existence of optimal stochastic policy in \eqref{eq:pi-star-MV} is equivalent to the integrability of the quantity $\exp(\frac{1}{\theta} H(t,x, a, \partial_xJ^*,  \partial_x^2J^*, J^*))$ over $a \in \mathcal{A} = (-\infty, \infty)$. This integrability depends on the tail behavior of the Hamiltonian and, in particular, the behavior of $J^*(t, x+ a^2)$ when $a^2$ is large.
\end{remark}

Suppose the optimal stochastic policy $\boldsymbol{\pi}^*(\cdot|t,x)$ is Gaussian for all $(t,x)$, implying that  the Hamilitonian $H(t,x, a, \partial_xJ^*,  \partial_x^2J^*, J^*)$ is a quadratic function of $a$.
It then follows from \eqref{eq:HalMV2} that  there exist functions $h_1(t,x)$ and $h_2(t,x)$ such that
\begin{align}\label{eq:quad}
	J^*(t, x+  a^2 ) - J^*(t,x) - a^2 \partial_xJ^*(t,x) = a^2 \cdot h_1(t,x) + a \cdot h_2(t,x), \quad \text{for all $(t,x,a)$}.
\end{align}
We do not put a term independent of $a$ on the right-hand side because the left-hand side is zero when $a=0$.
Taking derivative with respect to $a$, we obtain
\begin{align}
	\partial_xJ^*(t, x+  a^2 ) \cdot 2a - 2a \partial_xJ^*(t,x) =  2a \cdot h_1(t,x) + h_2(t,x), \quad \text{for all $(t,x,a)$}.
\end{align}
Setting $a=0$, we get $h_2 = 0.$ It follows that
\begin{align}
	a \cdot \left[ \partial_xJ^*(t, x+  a^2 )  -  \partial_xJ^*(t,x) -   h_1(t,x)  \right] = 0  \quad \text{for all $(t,x, a)$}.
\end{align}
Hence, we have $h_1(t,x) = \partial_xJ^*(t, x+  a^2 )  -  \partial_xJ^*(t,x) $ for any $a \ne 0.$  Sending $a$ to zero yields $h_1(t,x) =0$ for all $(t,x)$. Therefore, we obtain from \eqref{eq:quad} that
\begin{align}
	J^*(t, x+  a^2 ) - J^*(t,x) - a^2 \partial_xJ^*(t,x) = 0, \quad \text{for all $(t,x,a)$}.
\end{align}
Taking derivative in $a$ in the above we have
\begin{align}
	\partial_xJ^*(t, x+  a^2 ) -  \partial_xJ^*(t,x) = 0, \quad \text{for all $(t,x,a)$}.
\end{align}
Thus, $\partial_xJ^*$ is constant in $x$ or $J^*$ is affine in $x$, leading to
$J^*(t,x) = g_1(t) x + g_2(t)$ for some functions $g_1(t)$ and $g_2(t)$. The resulting Hamiltonian becomes
\begin{align}
	H(t,x, a, \partial_xJ^*,  \partial_x^2J^*, J^*)&= a\sigma\rho g_1(t).
\end{align}
This is linear in $a \in \mathcal{A} = (-\infty, \infty)$ and hence the integral $\int_{\mathcal{A}}\exp(-\frac{1}{\theta} H(t,x, a, \partial_xJ^*,  \partial_x^2J^*, J^*)) da$ does not exist. It follows that $\boldsymbol{\pi}^*(\cdot| t, x)$ does not exist, which is a contradiction. Therefore, we have shown  that under \eqref{eq:gamma-a2}, the optimal stochastic policy either does not exist or is not Gaussian when it exists.

\begin{remark}
The argument above works for $\gamma(a,z) = a^m$ for any $m >1$. 
\end{remark}

\section{Application: Mean--Variance Hedging of Options}\label{sec:MV-hedge}
The MV portfolio selection problem considered in Section \ref{sec:mean-var} is an LQ problem. In this section, we present another application that is non-LQ. Consider an option seller who needs to hedge a short position in a European-style option that expires at time $T$. The option is written on a risky asset whose price process $S$ is described by the SDE \eqref{eq:stock-SDE} with condition \eqref{eq:exp-nu} satisfied. At the terminal time $T$, the seller pays $G(S_T)$ to the option holder. We assume that the seller's hedging activity will not affect the risky asset price.

To hedge the random payoff, the seller constructs a portfolio consisting of the underlying risky asset and cash. We consider discounted quantities in the problem: the discounted risky asset price $\hat{S}_t\coloneqq e^{-r_ft}S_t$ and discounted payoff $\hat{G}(\hat{S}_T):=e^{-r_fT}G(S_T)$, where $r_f$ is again the constant risk-free interest rate. As an example, for a put option with strike price $\mathcal{K}$, $G(S_T)=(\mathcal{K}-S_T)^+$ ($x^+:=\max(x,0)$), and
$\hat{G}(\hat{S}_T)=(\hat{\mathcal{K}} - \hat{S}_T)^+$, where $\hat{\mathcal{K}}:=e^{-r_fT}\mathcal{K}$. Here
$\hat{S}$ follows the SDE
\begin{equation}
	d\hat{S}_t=\hat{S}_{t-}\bigg(\rho\sigma dt+ \sigma dW_t + \int_{\mathbb{R}} (\exp(z)-1) \widetilde{N}(dt, dz)\bigg),\label{eq:stock-SDE-discount}
\end{equation}
where $\rho$ is defined in \eqref{eq:rho}.

We denote the discounted dollar value in the risky asset and the discounted value of the hedging portfolio at time $t$ by $a_t$ and $X_t$, respectively. As in the MV portfolio selection problem, $a$ is the control variable and $X$ follows the SDE \eqref{eq:wealth-det}. The seller seeks a hedging policy to minimize the deviation from the terminal payoff. A popular formulation is mean--variance hedging (also known as quadratic hedging), which considers the objective
\begin{equation}
	\min_{a}\mathbb{E}\Big[\big(X_{T}^{a}-\hat{G}(\hat{S}_T)\big)^{2}\Big].\label{eq:MV-hedge}
\end{equation}
Note that this expectation is taken under the real-world probability measure, rather than under any martingale measure for option pricing. An advantage of this formulation is that relevant data are observable in the real world (but not necessarily in a risk-neutral world). Although objective \eqref{eq:MV-hedge} looks similar to that of the MV portfolio selection problem, the target the hedging portfolio tries to meet is now random instead of constant. Furthermore, the payoff is generally nonlinear in $\hat{S}_T$. Thus, MV hedging is not an LQ problem. Also, both $\hat{S}_t$ and $X_t^{a}$ matter for the hedging decision at time $t$, whereas $\hat{S}_t$ is irrelevant for decision in the MV portfolio selection problem. This increase in the state dimension makes the hedging problem much more difficult to solve.

We consider the following exploratory formulation to encourage exploration for learning:
\begin{align} \label{eq:MV-hedge-expl}
	\min_{\bpi}\mathbb{E}\bigg[\big(\tilde{X}_{T}^{\bpi}-\hat{G}(\hat{S}_T)\big)^{2} + \theta \int_0^T  \int_{\mathcal{A}}\log \bpi(a|t, X_{t-}^{\bpi}) \bpi(a|t, X_{t-}^{\bpi})dadt\bigg],
\end{align}
where the discounted value of the hedging portfolio under a stochastic policy $\bpi$ follows \eqref{eq:wealth-expl}.

\subsection{Solution of the exploratory control problem}
We consider the HJB equation for problem \eqref{eq:MV-hedge-expl}:
\begin{equation}
	\partial_tV +   \inf_{ \boldsymbol{\pi} \in \mathcal{P} (\mathbb{R})}\int_{\mathbb{R}} \big(    H(t,S,x,a,\partial_SV,\partial_S^2V,\partial_xV,\partial_x^2V,\partial_S\partial_xV,V) + \theta \log \bpi(a | t,S,x)\big) \bpi(a | t,S,x) da = 0, \label{eq:HJB-MV-hedge}
\end{equation}
with terminal condition $V(T,S,x) = (x-\hat{G}(S))^2$. The Hamiltonian of the problem is given by
\begin{align}
	&\phantom{=~}H(t,S,x,a,\partial_SV,\partial_S^2V,\partial_xV,\partial_x^2V,\partial_S\partial_xV,V)\\
	&= \rho\sigma S \partial_SV + \frac{1}{2}\sigma^2S^2 \partial_S^2V + a\rho\sigma \partial_xV + \frac{1}{2}a^2\sigma^2 \partial_x^2V + a\sigma^2S\partial_S\partial_xV \\
	& + \int_{\mathbb{R}} \big( V(t,S+\tilde{\gamma}(S,z),x+\gamma(a, z)) - V(t,S,x) -\tilde{\gamma}(S,z)\partial_SV(t,S,x)-\gamma(a,z)\partial_xV(t,S,x) \big) \nu(dz)\label{eq:Hal-QH}
\end{align}
with $\gamma(a,z)=a(\exp(z)-1)$ and $\tilde{\gamma}(S,z)=S(\exp(z)-1)$.

We make the following ansatz for the solution of the HJB equation \eqref{eq:HJB-MV-hedge}:
\begin{equation}\label{eq:ansatz-MV-hedge}
	V(t,S,x)=(x-h(t,S))^2 f(t) + g(t,S).
\end{equation}
With this, we can simplify the integral term in the Hamiltonian and obtain
\begin{align}
	&\phantom{=~}H(t,S,x,a,\partial_SV,\partial_S^2V,\partial_xV,\partial_x^2V,\partial_S\partial_xV,V)\\
	&= \rho\sigma S \partial_SV + \frac{1}{2}\sigma^2S^2 \partial_S^2V + \frac{1}{2}a^2(\sigma^2+\sigma_J^2) \partial_x^2V \\
	&\phantom{=~}+ a\bigg(\rho\sigma \partial_xV + \sigma^2 S \partial_S\partial_xV - \int_{\mathbb{R}}(\exp(z)-1)\big(h(t,S\exp(z))-h(t,S)\big)\nu(dz)\bigg)\\
	&\phantom{=~}+\int_{\mathbb{R}} \big( V(t,S\exp(z),x) - V(t,S,x)- S(\exp(z)-1)\partial_SV(t,S,x)\big) \nu(dz).\label{eq:Hal-MV-hedge}
\end{align}
The probability density function that minimizes the integral in \eqref{eq:HJB-MV-hedge} is given by
\begin{equation}
	\bpi_c(\cdot|t,S,x)\propto \exp \bigg(-\frac{1}{\theta} H(t,S,x,a,\partial_SV,\partial_S^2V,\partial_xV,\partial_x^2V,\partial_S\partial_xV,V)\bigg),
\end{equation}
which is a candidate for the optimal stochastic policy.
From \eqref{eq:Hal-MV-hedge}, we obtain that $\bpi_c(\cdot|t,S,x)$ is given by
\begin{equation}
\mathcal{N}\bigg(\cdot\ \bigg| -\frac{\rho\sigma \partial_xV + \sigma^2 S \partial_S\partial_xV - \int_{\mathbb{R}}(\exp(z)-1)(h(t,S\exp(z))-h(t,S))\nu(dz)\partial_x^2V}{(\sigma^2+\sigma_J^2)\partial_x^2V}, \frac{\theta}{(\sigma^2+\sigma_J^2)\partial_x^2V}\bigg).
\end{equation}
Substituting it back to the HJB equation \eqref{eq:HJB-MV-hedge}, we obtain
a nonlinear PIDE
\begin{align}
	&\phantom{+~}\partial_tV + \rho\sigma S\partial_SV + \frac{1}{2}\sigma^2S^2\partial_S^2V\\
	&+ \int_{\mathbb{R}} \big( V(t,S\exp(z),x) - V(t,S,x)- S(\exp(z)-1)\partial_SV(t,S,x)\big) \nu(dz)\\
	&-\frac{(\rho\sigma \partial_XV + \sigma^2 S\partial_S\partial_xV-\int_{\mathbb{R}}(\exp(z)-1)(h(t,S\exp(z))-h(t,S))\nu(dz)\partial_x^2V)^2}{2(\sigma^2+\sigma_J^2)\partial_x^2V}\\
	&-\frac{\theta}{2}\ln\frac{2\pi\theta}{(\sigma^2+\sigma_J^2)\partial_x^2V}=0,\quad (t,S,x)\in[0,T)\times\mathbb{R}_+\times\mathbb{R},\ V(T,x)=(x-\hat{G}(S))^2.
\end{align}
We plug in the ansatz \eqref{eq:ansatz-MV-hedge} to the above PIDE. After some lengthy calculations, we can collect similar terms and obtain the following equations satisfied by $f$, $h$ and $g$:
\begin{align}
	& f'(t)-\frac{\rho^2\sigma^2}{\sigma^2+\sigma_J^2}f(t)=0,\ f(T)=1, \label{eq:f-ode-MV-hedge}\\
	&\partial_th + \int_\mathbb{R} \big( h(t,S\exp(z)) - h(t,S)- S(\exp(z)-1)\partial_Sh(t,S)\big) \bigg(1-\frac{\rho\sigma}{\sigma^2+\sigma_J^2}(\exp(z)-1)\bigg)\nu(dz)\\
	& + \frac{1}{2}\sigma^2 S^2 \partial_S^2h = 0, \quad (t,S)\in[0,T)\times\mathbb{R}_+,\ h(T,S)=\hat{G}(S), \label{eq:h-pde-MV-hedge}\\
	&\partial_tg + \rho\sigma S \partial_Sg + \frac{1}{2}\sigma^2 S^2 \partial_S^2g + \int_\mathbb{R} \big(g(t,S\exp(z)) - g(t,S)- S(\exp(z)-1)\partial_Sg(t,S)\big)\nu(dz) \\
	& + \sigma^2 S^2(\partial_Sh)^2f(t) + \int_{\mathbb{R}}\big(h(t,S\exp(z))-h(t,S)\big)^2f(t)\nu(dz) \\
	& - \frac{f(t)}{\sigma^2+\sigma_J^2}\bigg(\sigma^2S\partial_Sh + \int_{\mathbb{R}}(\exp(z)-1)\big(h(t,S\exp(z))-h(t,S)\big)\nu(dz)\bigg)^2 \\	
	& - \frac{\theta}{2}\ln\frac{\pi\theta}{(\sigma^2+\sigma_J^2)f(t)} = 0, \quad (t,S)\in[0,T)\times\mathbb{R}_+,\ g(T,S)=0. \label{eq:g-pde-MV-hedge}
\end{align}
The function $f$ is given by
\begin{equation}\label{eq:ft}
f(t)=\exp\bigg(-\frac{\rho^2\sigma^2}{\sigma^2+\sigma_J^2}(T-t)\bigg).	
\end{equation}
However, the two PIDEs cannot be solved in closed form in general. Below we first present some properties of $h$ and $g$.
\begin{lemma}\label{lem:pde-MV-hedge}
	Assume $\sigma>0$ and $\hat{G}(S)$ is Lipschitz continuous in $S$. The PIDE \eqref{eq:h-pde-MV-hedge} has a unique solution in $C^{1,2}([0,T)\times\mathbb{R}_{+})\cap C([0,T]\times\mathbb{R}_{+})$ that is Lipschitz continuous in $S$. The PIDE \eqref{eq:g-pde-MV-hedge} also has a unique solution in $C^{1,2}([0,T)\times\mathbb{R}_{+})\cap C([0,T]\times\mathbb{R}_{+})$. Moreover, there exists a constant $C>0$ such that
	\begin{equation}
		|h(t,S)|\leq C(1+S),\quad |g(t,S)|\leq C(1+S^2)\ \ \text{for any}\ (t,S)\in[0,T]\times\mathbb{R}_+.
	\end{equation}
\end{lemma}
Next, we provide stochastic representations for the solutions to PIDEs \eqref{eq:h-pde-MV-hedge} and \eqref{eq:g-pde-MV-hedge}, which will be subsequently exploited for our RL study. Construct a new probability measure $\mathbb{Q}$ from the given probability measure $\mathbb{P}$ with the Radon--Nikodym density process $\Lambda_s:=\frac{d\mathbb{Q}}{d\mathbb{P}}|_{\mathcal{F}_s}$ defined by
\begin{equation}
	d\Lambda_s = -\frac{\rho\sigma}{\sigma^2+\sigma_J^2}\Lambda_{s-} \bigg(\sigma dW_s + \int_{\mathbb{R}} (\exp(z)-1) \widetilde{N}(ds, dz)\bigg),\ \Lambda_t = 1.
\end{equation}
Condition \eqref{eq:exp-nu} guarantees that $\{\Lambda_s: t\leq s\leq T\}$ is a true martingale with unit expectation. The  standard measure change results yield that, under $\mathbb{Q}$,
\begin{equation}\label{eq:underQ}
	W'_s := W_s + \frac{\rho\sigma^2}{\sigma^2+\sigma_J^2}(s-t),\ \widetilde{N}'(ds,dz) := N(ds,dz) - \bigg(1-\frac{\rho\sigma}{\sigma^2+\sigma_J^2}(\exp(z)-1)\bigg)\nu(dz)ds
\end{equation}
are standard Brownian motion and compensated Poisson random measure, respectively. Using them we can rewrite the SDE for $\hat{S}$ as
\begin{equation}
	d\hat{S}_s = \hat{S}_{s-}\bigg(\sigma dW'_s + \int_{\mathbb{R}} (\exp(z)-1) \widetilde{N}'(ds, dz)\bigg).\label{eq:stock-MV-hedge-Q}
\end{equation}
Let $Y_t:=1+\sigma W_t + \int_0^t\int_{\mathbb{R}} (\exp(z)-1) \widetilde{N}'(ds, dz)$. Clearly, it is a L\'evy process and also a martingale under condition \eqref{eq:exp-nu}. Process $\hat{S}$ is the stochastic exponential of $Y$ and it follows from \cite[Proposition 8.23]{tankov2004financial} that $\hat{S}$ is also a martingale under $\mathbb{Q}$. The Feynman--Kac theorem gives the representation of the solution $h$ to PIDE \eqref{eq:h-pde-MV-hedge} as
\begin{equation}
	h(t,S) = \mathbb{E}^{\mathbb{Q}}[\hat{G}(\hat{S}_T)\big|\hat{S}_t=S],\quad t\in[0,T]. \label{eq:h-rep}
\end{equation}
We can view $\mathbb{Q}$ as the martingale measure for option valuation and interpret $h(t,S)$ as the option price at time $t$ and underlying price  $S$. For PIDE \eqref{eq:g-pde-MV-hedge}, again by the Feynman--Kac theorem we obtain that its solution $g$ is given by
\begin{align}
	g(t,S) &= g_e(t,S)-\int_t^T\frac{\theta}{2}\ln\frac{\pi\theta}{(\sigma^2+\sigma_J^2)f(s)}ds\\
	& = g_e(t,S) - \frac{\theta}{2}\bigg(\ln\bigg(\frac{\pi\theta}{\sigma^2 + \sigma_J^2}\bigg) (T - t) + \frac{\rho^2\sigma^2}{2(\sigma^2 + \sigma_J^2)}(T-t)^2\bigg), \label{eq:g-rep}
\end{align}
where
\begin{align}
	&g_e(t,S):=\mathbb{E}^{\mathbb{P}}\bigg[\int_t^T f(s)\bigg(\sigma^2 \hat{S}_s^2(\partial_Sh(s,\hat{S}_s))^2 + \int_{\mathbb{R}}\big(h(s,\hat{S}_s\exp(z))-h(s,\hat{S}_s)\big)^2\nu(dz)\bigg)ds\\
	& - \int_t^T\frac{f(s)}{\sigma^2+\sigma_J^2}\bigg(\sigma^2\hat{S}_s\partial_Sh(s,\hat{S}_s) + \int_{\mathbb{R}}(\exp(z)-1)\big(h(s,\hat{S}_s\exp(z))-h(s,\hat{S}_s)\big)\nu(dz)\bigg)^2ds\bigg|\hat{S}_t=S\bigg].\label{eq:g_e}
\end{align}

Using the expression for $V(t,S,x)$, we obtain $\bpi_c(\cdot|t,S,x)$ as
\begin{equation}\label{eq:policy-optimal-MV-hedge}
	\bpi_c(\cdot|t,S,x)\sim\mathcal{N}\bigg(\cdot\ \bigg|\ \mathcal{M}^*(t,S,x), \frac{\theta}{2(\sigma^2+\sigma_J^2)}\exp\bigg(\frac{\rho^2\sigma^2}{\sigma^2+\sigma_J^2}(T-t)\bigg)\bigg),
\end{equation}
where
\begin{equation}\label{eq:mean-policy-optimal-MV-hedge}
	\mathcal{M}^*(t,S,x)=\frac{\sigma^2S\partial_Sh(t,S)+\int_{\mathbb{R}}(\exp(z)-1)(h(t,S\exp(z))-h(t,S))\nu(dz)-\rho\sigma(x-h(t,S))}{\sigma^2+\sigma_J^2}.
\end{equation}
Using Lemma \ref{lem:pde-MV-hedge}, we can directly verify that
$\bpi_c$ is admissible. It is also obvious to see that $V(t,S,x)$ given by \eqref{eq:ansatz-MV-hedge} has quadratic growth in $S$ and $x$. Therefore, by Lemma \ref{lem:expHJB}, we have the following conclusion.
\begin{proposition}\label{prop:MV-opt}
	For the MV hedging problem \eqref{eq:MV-hedge}, the optimal value function $J^*(t,S,x)=V(t,S,x)$ and the optimal stochastic policy $\bpi^*=\bpi_c$ under the assumptions that $\sigma>0$ and $\hat{G}(S)$ is Lipschitz continuous in $S$.
\end{proposition}

The mean part of $\bpi^*$, $\mathcal{M}^*(t,S,x)$,  comprises three terms. The quantity $\partial_Sh(t,S)$ is the delta of the option and hence $S\partial_Sh(t,S)$ gives the dollar amount of the risky asset as required by delta hedging, which is used to hedge continuous price movements. The integral term shows the dollar amount of the risky asset that should be held to hedge discontinuous changes. The two hedges are combined using weights $\sigma^2/(\sigma^2+\sigma_J^2)$ and $1/(\sigma^2+\sigma_J^2)$. The last term reflects the adjustment that needs to be made due to the discrepancy between the hedging portfolio value and option price. It is easy to show that $\mathcal{M}^*(t,S,x)$ is the optimal deterministic policy of problem \eqref{eq:MV-hedge}.

The variance of $\bpi^*$ is the same as that in the MV portfolio selection problem (c.f. \eqref{eq:policy-optimal-MV}), which decreases as $t$ approaches the terminal time $T$. This implies that one gradually reduces exploration over time.

\subsection{Parametrizations and actor--critic learning}
We use the previously derived solution of the exploratory problem as knowledge for RL. As we will see, the exploratory solution lends itself to natural parametrizations for the policy (``actor") and value function (``critic"), but less so for parametrizing the q-function.

To simplify the integral in $\mathcal{M}^*(t,S,x)$  defined in \eqref{eq:mean-policy-optimal-MV-hedge}, we first parametrize the L\'evy density $\nu(z)$ as in \eqref{eq:Levy-den-Gauss} where we have three parameters:
$\lambda$ is the jump arrival rate, and $m$ and $\delta$ are the mean and standard deviation of the normal distribution for the jump size.
We then use the Fourier-cosine method developed in \cite{fang2009novel} to obtain an expression for $h(t,S)$. Let $Y_t:=\log(\hat{S}_t/\hat{S}_0)$. Applying It\^o's formula,  under $\mathbb{P}$ we obtain
\begin{equation}
	d\log \hat{S}_t = \bigg(\rho\sigma - \frac{1}{2}\sigma^2 - \lambda\bigg(\exp\bigg(m +\frac{1}{2}\delta^2\bigg)-1\bigg)\bigg) dt  + \sigma dW_t + \int_{\mathbb{R}}zN(dt,dz).
\end{equation}
Noting \eqref{eq:underQ}, under $\mathbb{Q}$ we have
\begin{equation}
	d\log \hat{S}_t = \bigg(\rho\sigma\frac{\sigma_J^2}{\sigma^2+\sigma_J^2} - \frac{1}{2}\sigma^2 - \lambda\bigg(\exp\bigg(m +\frac{1}{2}\delta^2\bigg)-1\bigg)\bigg)dt  + \sigma dW'_t + \int_{\mathbb{R}}zN(dt,dz),
\end{equation}
where the L\'evy measure of $N(dt,dz)$ becomes
\begin{equation}
	\nu'(dz)=\bigg(1-\frac{\rho\sigma}{\sigma^2+\sigma_J^2}\big(\exp(z)-1\big)\bigg)\frac{\lambda}{\sqrt{2\pi\delta^2}}\exp\bigg(-\frac{(z-m)^2}{2\delta^2}\bigg)dz.
\end{equation}
The characteristic function of $Y_t$ under $\mathbb{Q}$ is given by
\begin{equation}
	\Phi_{Y_t}^{\mathbb{Q}}(u) = \exp\bigg( iut\left(\rho\sigma\frac{\sigma_J^2}{\sigma^2+\sigma_J^2}- \frac{1}{2}\sigma^2 - \lambda\left(\exp\left(m +\frac{1}{2}\delta^2\right)-1\right)\right) - \frac{1}{2}\sigma^2u^2 t  + t A(u) \bigg),
\end{equation}
where
\begin{align}
	A(u)=&\int_{\mathbb{R}}\big(\exp(iuz)-1\big)\nu'(dz)\\
	=& \lambda\bigg(\exp\bigg(ium-\frac{1}{2}\delta^2u^2\bigg) -1\bigg) - \frac{\lambda\rho\sigma}{\sigma^2+\sigma_J^2}\times\\
	&\bigg(\exp\bigg(m +\frac{1}{2}\delta^2 + iu(m+\delta^2)-\frac{1}{2}\delta^2u^2\bigg) -\exp\bigg(ium - \frac{1}{2}\delta^2u^2\bigg) -\exp\bigg(m +\frac{1}{2}\delta^2\bigg)+1\bigg).
\end{align}
The Fourier-cosine method then calculates $h(t,S)$ as
\begin{equation}\label{eq:FCos-h}
	\hat{h}(t,S) := \sideset{}{'}\sum_{k=0}^{N-1} \Re\bigg(\Phi_{Y_t}^{\mathbb{Q}}\bigg(\frac{k\pi}{b-a}\bigg)\exp\bigg(ik\pi\frac{\ln(S/\hat{\mathcal{K}})-a}{b-a}\bigg)\bigg) V_k.
\end{equation}
Here, $\Re(x)$ denotes the real part of complex number $x$, $\Sigma'$ indicates the first term in the sum is halved, $\hat{\mathcal{K}}$ is the discounted strike price of the call or put option under consideration, and $[a,b]$ is the truncation region for $Y_T$ (chosen wide enough so that the probability of $Y_T$ going outside under $\mathbb{Q}$ can be neglected). The expression of $V_k$ is given by Eqs. (24) and (25) in \cite{fang2009novel} for call and put options, respectively, and it only depends on $a, b, \hat{\mathcal{K}}$. We can also calculate $\partial_Sh(t,S)$ by differentiating $\hat{h}(t,S)$ w.r.t. $S$, which yields
\begin{equation}\label{eq:FCos-hS}
	\partial_S\hat{h}(t,S) := \sideset{}{'}\sum_{k=0}^{N-1}  \Re\bigg(\Phi_{Y_t}^{\mathbb{Q}}\bigg(\frac{k\pi}{b-a}\bigg)\frac{ik\pi}{(b-a)S}
\exp\bigg(ik\pi\frac{\ln(S/\hat{\mathcal{K}})-a}{b-a}\bigg)\bigg)V_k.
\end{equation}

\textbf{Actor parametrization}. We follow the form of the optimal stochastic policy given in \eqref{eq:policy-optimal-MV-hedge} to parametrize our policy. To obtain a parsimonious parametrization, we introduce $\phi_1, \phi_2, \phi_3, \phi_4, \phi_5$, which respectively represent the agent's understanding -- relative to her learning -- of drift $\mu$, volatility $\sigma$, jump arrival rate $\lambda$, mean $m$ and standard deviation $\delta$ of the jump size of the risky asset. We emphasize that in our learning algorithm these parameters are not {\it estimated} by any statistical method (and hence they do not  necessarily correspond to the true dynamics of the risky asset); rather they will be updated based on the hedging experiences (via exploration and exploitation) from the real environment. To implement \eqref{eq:policy-optimal-MV-hedge}, we also need two dependent parameters $\phi_6$ and $\phi_7$ that are calculated by
\begin{equation}\label{eq:dep-params}
	\phi_6=\frac{\phi_1-r_f}{\phi_2},\
	\phi_7=\sqrt{\phi_3\bigg(\exp(2\phi_4+2\phi_5^2)-2\exp\bigg(\phi_4+\frac{1}{2}\phi_5^2\bigg)+1\bigg)}.	
\end{equation}
These two parameters show the agent's understanding of $\rho$ and $\sigma_J$, and they will be updated in the learning process by the above equations with the update of the independent parameters $\phi_1$ to $\phi_5$. Denote $\phi=(\phi_i)_{i=1}^7$.

We use $\phi_1,\phi_2,\phi_3,\phi_4,\phi_5,\phi_6,\phi_7$ to replace their counterparts $\mu,\sigma,\lambda,m,\delta,\rho,\sigma_J$ in \eqref{eq:policy-optimal-MV-hedge} to obtain a parametrized policy.
For the mean part, we calculate the option price and its delta by plugging $\phi$ into \eqref{eq:FCos-h} and \eqref{eq:FCos-hS}, with the results denoted by $\hat{h}^{\phi}(t,S)$ and $\hat{h}_S^{\phi}(t,S)$ respectively, and approximate the integral by Gauss-Hermite (G-H) quadrature, which yields
\begin{equation}
	\mathcal{M}^{\phi}(t,S,x) = \frac{\phi_2^2S\partial_S\hat{h}^{\phi}(t,S)+\frac{\phi_3}{\sqrt{2\pi}}\sum\limits_{q=1}^Qw_q(\exp(z_q)-1)\left(\hat{h}^{\phi}(t,S\exp(z_q))-\hat{h}^{\phi}(t,S)\right)-\phi_2\phi_6(x-\hat{h}^{\phi}(t,S))}{\phi_2^2+\phi_7^2},
\end{equation}
where $z_q = \phi_4+\phi_5 u_q$ and $\{(w_q,u_q), q=1,\cdots,Q\}$ are the weights and abscissa of the $Q$-point G-H rule that are predetermined and independent of $\phi$. For the variance part, we use $\phi$ in \eqref{eq:ft} to obtain $f^\phi(t)$. Thus, our parametrized policy is given by
\begin{equation}\label{eq:param-policy-MV-hedge}
	\bpi^{\phi}(\cdot|t,S,x)\sim\mathcal{N}\bigg(\cdot\ \bigg|\ \mathcal{M}^\phi(t,S,x), \frac{\theta}{2(\phi_2^2+\phi_7^2)f^\phi(t)}\bigg).
\end{equation}

\textbf{Critic parametrization}. For the value function of $\bpi^{\phi}$, following \eqref{eq:ansatz-MV-hedge} we parametrize it as
\begin{equation}
	J^{\psi,\phi}(t,S,x) = (x-\hat{h}^{\phi}(t,S))^2 f^\phi(t)
	  + g_e^{\psi}(t,S) - \frac{\theta}{2}\bigg(\ln\bigg(\frac{\pi\theta}{\phi_2^2 + \phi_7^2}\bigg) (T - t) + \frac{\phi_2^2\phi_6^2}{2(\phi_2^2 + \phi_7^2)}(T-t)^2\bigg).
\end{equation}
We do not specify any parametric form for $g_e^{\psi}(t,S)$, but will learn it using Gaussian process regression (\citealp{williams2006gaussian}), where $\psi$ denotes its parameters.
We choose Gaussian process for our problem because it is able to generate flexible shapes and  can often achieve a reasonably good fit with a moderate amount of data.

\textbf{Actor--critic learning}.  Consider a time grid $\{t_k\}_{k=0}^K$ with $t_0=0$, $t_K=T$, and time step $\Delta t$. In an iteration, we start with a given policy $\bpi^{\phi}$ and use it to collect $M$ episodes from the environment, where the $m$-th episode is given by $\{(t_k, \hat{S}_{t_k}^{(m)}, X_{t_k}^{\bpi^\phi,(m)}, a_{t_k}^{\bpi^\phi,(m)})\}_{k=0}^K$. We then perform two steps.

In the critic step, the main task is to learn $g_e^{\psi}(t_k,S)$ as a function of $S$ for each $t_k$. The stochastic representation \eqref{eq:g_e}  shows that $g_e$ can be viewed as the value function of a stream of rewards that depend on time and the risky asset price; so the task is a policy evaluation problem. We first obtain the running reward path in every episode, where the integral involving the L\'evy measure is again calculated by the $Q$-point G-H rule. Specifically, the reward at time $t_k$ for $k=0,\cdots, K-1$ is given by
\begin{align}
&R_{t_k}^g:=f^\phi(t_k)\bigg(\phi_2^2 \hat{S}_{t_k}^2(\partial_S\hat{h}^\phi(t_k,\hat{S}_{t_k}))^2 + \frac{\phi_3}{\sqrt{2\pi}}\sum_{q=1}^Qw_q\big(\hat{h}^{\phi}(t_k,\hat{S}_{t_k}\exp(z_q))-\hat{h}^{\phi}(t_k,\hat{S}_{t_k})\big)^2\bigg)\Delta t \\
&-\frac{f^\phi(t_k)}{\phi_2^2+\phi_7^2}\bigg(\phi_2^2\hat{S}_{t_k}\partial_S\hat{h}^\phi(t_k,\hat{S}_{t_k}) + \frac{\phi_3}{\sqrt{2\pi}}\sum_{q=1}^Qw_q(\exp(z_q)-1)\big(\hat{h}^\phi(t_k,\hat{S}_{t_k}\exp(z_q))-\hat{h}^\phi(t_k,\hat{S}_{t_k})\big)\bigg)^2\Delta t,
\end{align}
where the risky asset price path $\{(\hat{S}_{t_k})\}_{k=0}^{K-1}$ is collected from the real environment.
We stress again that the above calculation uses the agent's parameters $\phi$ and does not involve any parameters of the true dynamics. We then fit a Gaussian process to the sample of the cumulative rewards $\{\sum_{\ell=k}^{K-1}R_{t_\ell}^{g,(m)}\}_{m=1}^M$ to estimate $g_e^{\psi}(t_k,S)$. Figure \ref{fig:GP-fit} illustrates the fit at three time points with the radial basis kernel. 
The Gaussian process is able to provide a good fit to the shape of $g_e(t,S)$ as a function of $S$ for each given $t$.

In the actor step, we update the policy parameter $\phi_i$ for $i=1,\cdots,5$ by
\begin{align}\label{eq:update-MV-hedge}
	\phi_i \leftarrow &~\phi_i  -  \frac{\alpha_{\phi_i}}{M}\sum_{m=1}^M\sum_{k=0}^{K-1}\Big(\theta\log \bpi^{\phi}( a_{t_k}^{\bpi^{\phi},(m)}|~t_k, \hat{S}_{t_k}^{(m)}, X_{t_k}^{\bpi^{\phi},(m)})\Delta t -\beta  J^{\psi,\phi}(t_k, X_{t_k}^{\bpi^{\phi},(m)}) \Delta t\\
	&\phantom{\phi_i  -  \frac{\alpha_{\phi_i}}{M}\sum_{m=1}^M\sum_{k=0}^{K-1}\Big(} + J^{\psi,\phi}(t_{k+1}, \hat{S}_{t_{k+1}}^{(m)}, X_{t_{k+1}}^{\bpi^{\phi},(m)}) - J^{\psi,\phi}(t_{k}, \hat{S}_{t_{k}}^{(m)}, X_{t_{k}}^{\bpi^{\phi},(m)})  \Big)\\
	&\phantom{\phi_i  -  \frac{\alpha_{\phi_i}}{M}\sum_{m=1}^M\sum_{k=0}^{K-1}\Big(}\cdot \partial_{\phi_i} \log\boldsymbol{\pi}^{\phi}( a_{t_k}^{\bpi^{\phi}, (m)}|~t_k, \hat{S}_{t_{k}}^{(m)}, X_{t_k}^{\bpi^{\phi}, (m)}).
\end{align}
This update essentially follows from \eqref{eq:update-un2};  the difference is that we do batch processing and update the parameters only after observing all the episodes in a batch. Furthermore, we flip the sign before $\alpha_{\phi_i}$ and $\theta$ as we consider minimization here while the problem in the general discussion is for maximization.

\begin{figure}[htbp!]
	\centering
	\includegraphics[width=1\textwidth]{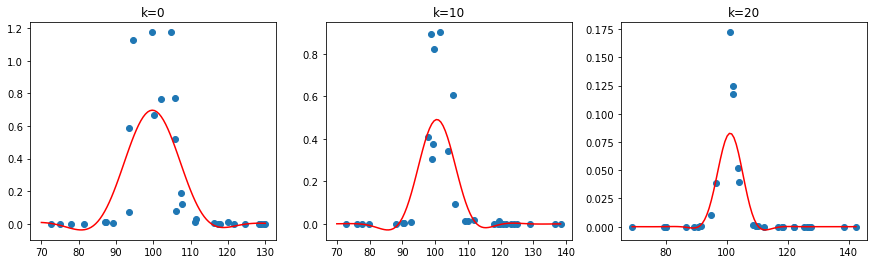}
	\caption{Fit of Gaussian process to the cumulative reward data for $g_e$ at $t_k$ ($k=0,10,20$) for a one-month put option with strike $\mathcal{K}=100$. The variable $S$ is on the $x$-axis.  We use the MLE estimates of the MJD model shown in Table \ref{tab:simul-MV-hedge} to calculate the rewards for $g_e$ from the sampled index price paths.}
	\label{fig:GP-fit}
\end{figure}

\subsection{Simulation study}
We check the convergence behavior of our actor--critic algorithm in a simulation study. We assume the environment is described by an MJD model with its true parameters shown in the upper row of Table \ref{tab:simul-MV-hedge}.\footnote{These values are chosen close to the estimates obtained from the market data of S\&P 500 in Table \ref{tab:MVPS-params} to make the MJD model better resemble the reality in the simulation study.}  We consider hedging a 4-month put option with strike price $\mathcal{K}=100$ because short-term options with maturity of a few months are typically traded much more actively than longer-term options. We assume the risk-free rate is zero and implement daily rebalancing.

We use 20 quadrature points for the G-H rule, which suffices for high accuracy. In each iteration, we sample $M=32$ episodes from the environment simulated by the MJD model. The learning rates are chosen to be $\alpha_{\phi_1}=5\mathrm{e}{-3}$, $\alpha_{\phi_2}=2\mathrm{e}{-5}$, $\alpha_{\phi_3}=15$, $\alpha_{\phi_4}=1\mathrm{e}{-6}$, $\alpha_{\phi_5}=1\mathrm{e}{-6}$, while the temperature parameter $\theta=0.1$. We keep the learning rates constant in the first 30 iterations and then they decay at rate $l(j)=j^{-0.5}$, where $j$ is the iteration index.  Before our actor--critic learning, we sample a 15-year path from the MJD model as data for MLE and obtain estimates of the parameters shown in Table \ref{tab:simul-MV-hedge}. MLE estimates all the parameters quite well except the drift $\mu$, where the estimate is more than twice the true value.\footnote{This is expected due to the well-known mean-blur problem, as 15 years of data are insufficient to obtain a reasonably accurate estimate of the mean using statistical methods.} We then use these estimates as initial values of $\phi_1$ to $\phi_5$ for learning. Figure \ref{fig:hedge-converge} clearly shows that our RL algorithm is able to converge near the true values. In particular, the estimate of $\mu$ eventually moves closely around the true level despite being distant to it initially.

\begin{table}[htbp!]
	\centering
	\begin{tabular}{cc}
		\hline
		 & \textbf{Parameters}  \\ \hline
		True & $\mu = 0.06, \sigma = 0.13, \lambda=28, m=-0.004, \delta=0.03$ \\ \hline
		MLE & $\mu = 0.1289, \sigma = 0.1331, \lambda = 26.1091, m = -0.0033, \delta = 0.0317$ \\ \hline
	\end{tabular}
	\caption{Parameters of the true MJD model and MLE estimates from 15 years of data}\label{tab:simul-MV-hedge}
\end{table}

\begin{figure}[htbp!]
	\centering
	\includegraphics[width=\textwidth]{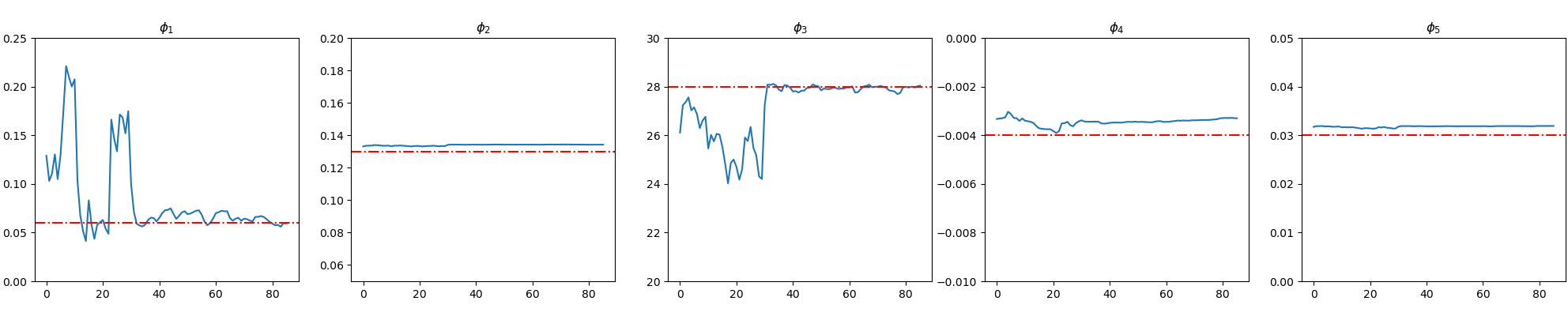}
	\caption{Training paths of five parameters (iteration index on the $x$-axis)}
	\label{fig:hedge-converge}
\end{figure}

\subsection{Empirical study}
In the empirical study we consider hedging a European put option written on the S\&P 500 index. This type of option is popular in the financial market as a tool for investors to protect their high beta investment. We assume the risk-free rate is zero, implement daily rebalancing and use the same S\&P 500 data as in Section \ref{sec:MV-Port-Emp} with the first 20 years as the training period and the last 4 years as the test period.

We use the MLE estimates from the first 20 years of observations to initialize $\phi_1$ to $\phi_5$ for learning. To sample an episode from the environment for training our RL algorithm, we bootstrap index returns from the training period. Similarly, we bootstrap index returns from the test period to generate index price paths to test the learned policy. We use the 20-point G-H rule with $M=32$ episodes in each training iteration. The learning rates are set as
$\alpha_{\phi_1}=5\mathrm{e}{-3}$, $\alpha_{\phi_2}=5\mathrm{e}{-4}$, $\alpha_{\phi_3}=50$, $\alpha_{\phi_4}=5\mathrm{e}{-6}$, $\alpha_{\phi_5}=1\mathrm{e}{-5}$ with the same learning rate decay schedule as in the simulation study. We set the temperature parameter $\theta=0.1$.

We consider four maturities ranging from 1 to 4 months for the put option. For a standard S\&P500 index option traded on CBOE, in practice its strike is quoted in points commensurate with the scale of the index level and converted to dollar amount by a multiplier of 100. That is, the dollar strike $\mathcal{K}=100\mathcal{K}_p$, where $\mathcal{K}_p$ is the strike in points.\footnote{For example, for an at-the-money option contract where the current index is 5000 points, we have $\mathcal{K}_p=5000$. The contract's notional size is 500000 dollars.} Let $S_t$ be the index level at time $t$ converted to dollars by the multiplier of 100. We sample $S_0$ uniformly from the range $[0.7\mathcal{K}, 1.3\mathcal{K}]$ to generate training and test episodes. This allows us to include scenarios where the option is deep out of money or in the money. We train our algorithm for 100 iterations and then test it on 5000 index price episodes. Finally, although we use a stochastic policy to interact with the environment for learning/updating the parameters during training, in the test stage we apply a deterministic policy which is the \emph{mean} of the learned stochastic policy. This can reduce the variance of the final hedging error.

To measure the test performance, we consider the mean squared hedging error (denominated in squared dollars) divided by $\mathcal{K}_p^2$. This effectively normalizes the contract's scale; so we can just set $\mathcal{K}_p=1$ in our implementation. The test result can be interpreted as the average mean squared hedging error in squared
 dollars for every point of strike price.

We compare the performances of the policy learned by our RL algorithm and the one obtained by plugging the MLE estimates from the training period data to \eqref{eq:mean-policy-optimal-MV-hedge}. Both policies are tested on the same index price episodes. We check the statistical significance of the difference in the squared hedging errors of the two policies by applying the t-test. Table \ref{tab:hedge-test-perf} reports the mean squared hedging error for each policy as well as the p-value of the t-test. Not surprisingly, the hedging error of each policy increases as the maturity becomes longer due to greater uncertainty in the problem. In every case, the RL policy achieves notable reduction in the hedging error compared with the MLE-based policy. The p-value further shows the outperformance of the former is statistically significant.

\begin{table}[htbp!]
	\centering
	\begin{tabular}{lccc}
		\hline
		Maturity & RL policy & MLE policy & p-value \\ \hline
		1 month & 0.3513 (0.034) & 0.4006 (0.043) & 3.2e-7 \\
		2 months & 0.5668 (0.029) & 0.6796 (0.050) & 1.2e-5 \\
		3 months & 0.8283 (0.031) & 1.0080 (0.060) & 1.6e-5 \\
		4 months & 0.8649 (0.047) & 1.1544 (0.068) & 4.2e-13 \\ \hline
	\end{tabular}
	\caption{Mean squared hedging errors of the two policies in the empirical test with their standard errors shown in parenthesis}\label{tab:hedge-test-perf}
\end{table}

\section{Conclusions} \label{sec:conclusion}

Fluctuations in data or time series coming from nonlinear, complex dynamic systems are characterized by two types: slow changes and sudden jumps, the latter occurring much rarely than the former. Hence jump-diffusions capture the key {\it structural}  characteristics of many data generating processes in areas such as physics, astrophysics, earth science,
engineering, finance, and medicine. As a result, RL for jump-diffusions is important both theoretically and practically. This paper endeavors to lay a theoretical foundation for the study. A key insight from this research is that temporal--difference algorithms designed for diffusions can work seamlessly for jump-diffusions. However,
unless using general neural networks, policy parameterization does need to respond to the presence of jumps if one is to take advantage of any special structure of an underlying problem.

There are plenty of open questions starting from here, including \rev{convergence of the exploratory HJB equation as the temperature parameter decays to zero}, convergence of the RL algorithms, regret bounds, decaying rates of the temperature parameters, and learning rates of the gradient decent and/or stochastic approximation procedures involved.

\bibliographystyle{chicago}
\bibliography{RL2,new}

\appendix

\section{Proofs}\label{appendix:proofs}

\begin{proof}[\bf Proof of Lemma~\ref{lem:expl-jump-LG}]
	We observe that for each $k=1,\cdots,\ell$,
	\begin{equation}
		\int_{\mathbb{R}\times[0,1]^n}\left|\gamma_k\left(t,x,G^{\bpi_{t,x}}(u),z\right)\right|^p\nu_k(dz)du = \int_{\mathcal{A}}\int_{\mathbb{R}}|\gamma_k(t,x,a,z)|^p\nu_k(dz)\bpi(a|t,x)da.
	\end{equation}
	From Assumption \ref{assump:SDE}-(iii), we have \rev{$\sum_{k=1}^\ell \int_{\mathbb{R}}|\gamma_k(t,x,a,z)|^p\nu_k(dz)\le C_p(1+|x|^p+|a|^p)$} for any $(t, x, a) \in[0, T] \times \mathbb{R}^d \times \mathcal{A}$. It follows from condition (iii) of admissibility (see Definition \ref{def:policy}) that the desired result holds.
\end{proof}


\begin{proof}[\bf Proof of Lemma~\ref{lem:expl-jump-LC}]
We consider
\begin{equation}
	\int_{\mathbb{R}\times[0,1]^n}\left|\gamma_k\left(t,x,G^{\bpi_{t,x}}(u),z\right)-\gamma_k\left(t,x',G^{\bpi_{t,x'}}(u),z\right)\right|^p\nu_k(dz)du,
\end{equation}
which is bounded by
\begin{align}
	&C_p\bigg(\int_{\mathbb{R}\times[0,1]^n}\left|\gamma_k\left(t,x,G^{\bpi_{t,x}}(u),z\right)-\gamma_k\left(t,x,G^{\bpi_{t,x'}}(u),z\right)\right|^p\nu_k(dz)du\\
	&\phantom{C_p~~}+\int_{\mathbb{R}\times[0,1]^n}\left|\gamma_k\left(t,x,G^{\bpi_{t,x'}}(u),z\right)-\gamma_k\left(t,x',G^{\bpi_{t,x'}}(u),z\right)\right|^p\nu_k(dz)du\bigg)
\end{align}
for some constant $C_p>0$.

For the first integral, using Assumption \ref{assump:jump} we obtain
\begin{align}
	&\int_{\mathbb{R}\times[0,1]^n}\left|\gamma_k\left(t,x,G^{\bpi_{t,x}}(u),z\right)-\gamma_k\left(t,x,G^{\bpi_{t,x'}}(u),z\right)\right|^p\nu_k(dz)du \\
	\leq&C_{K,p}\int_{[0,1]^n}\left|G^{\bpi}(t,x,u)-G^{\bpi}(t,x',u)\right|^pdu
	\leq C'_{K,p}|x-x'|^p.
\end{align}

For the second integral, we have
\begin{align}
	&\int_{\mathbb{R}\times[0,1]^n}\left|\gamma_k\left(t,x,G^{\bpi_{t,x'}}(u),z\right)-\gamma_k\left(t,x',G^{\bpi_{t,x'}}(u),z\right)\right|^p\nu_k(dz)du\\
	=&\int_{\mathcal{A}}\int_{\mathbb{R}}\left|\gamma_k\left(t,x,a,z\right)-\gamma_k\left(t,x',a,z\right)\right|^p\nu_k(dz)\bpi(a|t,x')da\\
	\leq&\int_{\mathcal{A}} C''_{K,p}|x-x'|^p\bpi(a|t,x')da\\
	=&C''_{K,p}|x-x'|^p,
\end{align}
where we used Assumption \ref{assump:SDE}-(ii) and $C''_{K,p}$ is the constant there. The desired claim is obtained by combining these results.
\end{proof}


\begin{proof}[\bf Proof of Lemma~\ref{lem:PK}]
	Fix $t\in[0,T)$ and suppose $\tilde X^{\pi}_t=x$. Let $a_s^{\bpi}(u)=G^{\bpi}(s,\tilde X^{\pi}_{s-},u)$. Define a sequence of stopping times $\tau_n = \inf \{s \ge t: |\tilde X^{\pi}_s| \ge n\}$ for $n\in\mathbb{N}$.
	Applying It\^{o}'s formula \eqref{eq:Ito-JD} to the process $e^{-\beta s}\phi(s,  \tilde X^{\pi}_s)$, where $\tilde X^{\pi}_s$ follows the exploratory SDE \eqref{eq:SDE-expl}, we obtain $\forall t'\in[t,T]$,
	\begin{align}
		&e^{-\beta(t'\wedge\tau_n)}\phi(t'\wedge\tau_n,  \tilde X^{\pi}_{r\wedge\tau_n}) - e^{-\beta t}\phi(t, x)\\
		=& \int_t^{t'\wedge\tau_n} e^{-\beta s}\left(\mathcal{L}^{\bpi} \phi(s, \tilde X^{\pi}_{s-})-\beta \phi(s, \tilde X^{\pi}_{s-}) \right)ds \\
		+& \int_t^{t'\wedge\tau_n} e^{-\beta s}\partial_x\phi(s, \tilde X^{\pi}_{s-})\circ\tilde{\sigma}(s, \tilde X^{\pi}_{s-}, \pi(\cdot|s, \tilde X^{\pi}_{s-})) dW_s  \label{eq:ito1W} \\
		+& \int_t^{t'\wedge\tau_n} e^{-\beta s}\sum_{k=1}^{\ell}\int_{\mathbb{R}\times[0,1]^n} \left( \phi(s, \tilde X^{\pi}_{s-}+ \gamma_k(s, \tilde X^{\pi}_{s-}, a_s^{\bpi}(u), z)) -  \phi(s, \tilde X^{\pi}_{s-})  \right)  \widetilde{N}'_k(ds, dz, du), \\
		\label{eq:ito1N}
	\end{align}
	where $ \mathcal{L}^{\bpi}$ is  given in \eqref{eq:expl-gen-2}. We next show that the expectations of \eqref{eq:ito1W} and \eqref{eq:ito1N} are zero.
	
	Note that for $s\in [t,t'\wedge\tau_n]$, $|\tilde X^{\pi}_{s-}|\leq n$. Thus, $\partial_x\phi(s, \tilde X^{\pi}_{s-})$ is also bounded. Using the linear growth of $\tilde{\sigma}(s, \tilde X^{\pi}_{s-}, \pi(\cdot|s, \tilde X^{\pi}_{s-}))$ in Lemma \ref{lem:expl-diff-coeffs} and the moment estimate \eqref{eq:moment-expl}, we can see that \eqref{eq:ito1W} is a square-integrable martingale and hence its expectation is zero.
	
	Next, we analyze the following stochastic integral:
	\begin{equation}\label{eq:FK-jump-integral}
		\int_t^{t'\wedge\tau_n} e^{-\beta s}\int_{\mathbb{R}\times[0,1]^n} \left( \phi (s, \tilde X^{\pi}_{s-}+ \gamma_k(s, \tilde X^{\pi}_{s-}, a_s^{\bpi}(u),z)) -  \phi(s, \tilde X^{\pi}_{s-})  \right)  \widetilde{N}_k(ds, dz, du).
	\end{equation}
	We consider the finite and infinite jump activity cases separately.
	
	Case 1: $\int_{\mathbb{R}}\nu_k(dz)<\infty$. In this case, both of the processes
	\begin{align}
		&\int_t^{t'\wedge\tau_n} e^{-\beta s}\int_{\mathbb{R}\times[0,1]^n} \phi (s, \tilde X^{\pi}_{s-}+ \gamma_k(s, \tilde X^{\pi}_{s-}, a_s^{\bpi}(u), z)) \widetilde{N}_k(ds, dz, du), \label{eq:FK-FA-int1}\\
		&\int_t^{t'\wedge\tau_n} e^{-\beta s}\int_{\mathbb{R}\times[0,1]^n} \phi (s, \tilde X^{\pi}_{s-}) \widetilde{N}_k(ds, dz, du) \label{eq:FK-FA-int2},
	\end{align}
	are square-integrable martingales and hence have zero expectations. We prove this claim for \eqref{eq:FK-FA-int1} and analyzing \eqref{eq:FK-FA-int2} is entirely analogous.
	
	Using the polynomial growth of $\phi$, Lemma \ref{lem:expl-jump-LG}, and $|\tilde X^{\pi}_{s-}|\leq n$ for $s\in[t, t'\wedge\tau_n]$, we obtain
	\begin{align}
		&\mathbb{E}_{t,x}\left[\int_t^{t'\wedge\tau_n} e^{-2\beta s}\int_{\mathbb{R}\times[0,1]^n} \left|\phi (s, \tilde X^{\pi}_{s-}+ \gamma_k(s, \tilde X^{\pi}_{s-}, a_s^{\bpi}(u), z))\right|^2 \nu_k(dz)duds\right]\\
		\leq& C_p\cdot \mathbb{E}_{t,x}\left[\int_t^{t'\wedge\tau_n} \int_{\mathbb{R}\times[0,1]^n} \left(1 + |\tilde{X}^{\pi}_{s-}|^{2p}  + |\gamma_k(s, \tilde X^{\pi}_{s-}, a_s^{\bpi}(u),z))|^{2p}\right)\nu_k(dz)duds\right]\\
		\leq& C_p'\cdot\mathbb{E}_{t,x}\left[\int_t^{t'\wedge\tau_n} (1+|\tilde X^{\pi}_{s-}|^{2p})ds\right]<\infty.
	\end{align}
	This implies the process \eqref{eq:FK-FA-int1} is a square-integrable martingale; see e.g. \cite[Section 1.9]{rong2006theory}  for square-integrability of stochastic integrals with respect to compensated Poisson random measures.
	
	Case 2: $\int_{\mathbb{R}}\nu_k(dz)=\infty$. Let $B_{1}=\{|z|\leq 1\}\times[0,1]^n$ and $B_{1}^c=\{|z|> 1\}\times[0,1]^n$. The stochastic integral \eqref{eq:FK-jump-integral} can be written as the sum of two integrals:
	\begin{align}
	&\int_t^{t'\wedge\tau_n} e^{-\beta s}\int_{B_1} \big( \phi (s, \tilde X^{\pi}_{s-}+ \gamma_k(s, \tilde X^{\pi}_{s-}, a_s^{\bpi}(u), z)) -  \phi(s, \tilde X^{\pi}_{s-})  \big)  \widetilde{N}_k(ds, dz,du)\label{eq:ito-N-small}\\
	+&\int_t^{t'\wedge\tau_n} e^{-\beta s}\int_{B_1^c} \big( \phi(s, \tilde X^{\pi}_{s-}+ \gamma_k(s, \tilde X^{\pi}_{s-}, a_s^{\bpi}(u), z)) -  \phi(s, \tilde X^{\pi}_{s-})  \big)  \widetilde{N}_k(ds, dz, du).\label{eq:ito-N-big}
	\end{align}
	Using the mean-value theorem, the stochastic integral \eqref{eq:ito-N-small} is equal to
	\begin{equation}
		\int_t^{t'\wedge\tau_n} e^{-\beta s} \int_{B_1} \partial_x\phi(s, \tilde X^{\pi}_{s-}+\alpha_s\gamma_k(s, \tilde X^{\pi}_{s-}, a_s^{\bpi}(u),z))\circ\gamma_k(s, \tilde X^{\pi}_{s-}, a_s^{\bpi}(u),z)\widetilde{N}_k(ds, dz, du)
	\end{equation}
	for some $\alpha_s\in[0,1]$. For $s\in [t,t'\wedge\tau_n]$, $|\tilde X^{\pi}_{s-}|\leq n$. By Assumption \ref{assump:SDE-1}-(v),  $|\gamma_k(s, \tilde X^{\pi}_{s-} ,a_s^{\bpi}(u),z)|$ is bounded for any $|z|\leq 1$, $s\in [t,t'\wedge\tau_n]$ and $u\in[0,1]^n$, which further implies the boundedness of $|\phi_x(s, \tilde X^{\pi}_{s-}+\alpha_s\gamma_k(s, \tilde X^{\pi}_{s-}, a_s^{\bpi}(u), z))|$. Now, for each $j=1,\cdots,d$,
	\begin{equation}
		\int_t^{t'\wedge\tau_n} e^{-\beta s}\int_{B_1} \partial_{x_j}\phi(s, \tilde X^{\pi}_{s-}+\alpha_s\gamma_k(s, \tilde X^{\pi}_{s-}, a_s^{\bpi}(u), z))\gamma_{jk}(s, \tilde X^{\pi}_{s-}, a_s^{\bpi}(u), z)\widetilde{N}_k(ds, dz, du)
	\end{equation}
	is a square-integrable martingale because
	\begin{align}
		&\mathbb{E}_{t,x}\left[\int_t^{t'\wedge\tau_n}\int_{B_1} \big(\partial_{x_j}\phi(s, \tilde X^{\pi}_{s-}+\alpha_s\gamma_k(s, \tilde X^{\pi}_{s-} , a_s^{\bpi}(u), z))\big)^2\gamma_{jk}^2(s, \tilde X^{\pi}_{s-}, a_s^{\bpi}(u), z)\nu_k(dz)duds\right]\\
		\leq& C\cdot \mathbb{E}_{t,x}\left[\int_t^{t'\wedge\tau_n} \int_{B_1} \gamma_{jk}^2(s, \tilde X^{\pi}_{s-}, a_s^{\bpi}(u), z)\nu_k(dz)duds\right]\\
		\leq& C'\cdot\mathbb{E}_{t,x}\left[\int_t^{t'\wedge\tau_n} (1+|\tilde X^{\pi}_{s-}|^2)ds\right]<\infty,
	\end{align}
	where we used Lemma \ref{lem:expl-jump-LG} and boundedness of $|\tilde X^{\pi}_{s-}|$ in the above. Thus, \eqref{eq:ito-N-small} is a square-integrable martingale with mean zero.
	
	For \eqref{eq:ito-N-big}, we can use the same arguments as in the finite activity case by observing $\int_{|z|>1} \nu_k(dz)<\infty$
	to show that each of the two processes in \eqref{eq:ito-N-big} is a square-integrable martingale with mean zero.
		
	Combining the results above, setting $t'=T$, and taking expectation, we obtain
	\begin{align*}
	\mathbb{E}_{t,x}[e^{-\beta(T\wedge\tau_n)}\phi( \tilde X^{\bpi}_{T\wedge\tau_n})] - e^{-\beta t}\phi(t, x)
	&= \mathbb{E}^{\bar{\mathbb{P}}}_{t,x} \bigg[ \int_t^{T\wedge\tau_n} e^{-\beta s}\big(\mathcal{L}^{\bpi} \phi(s, \tilde X^{\pi}_{s-})-\beta\phi(s, \tilde X^{\pi}_{s-})\big)ds\bigg].
	\end{align*}
	As $\phi(s,x)$ satisfies Eq.~\eqref{eq:FK}, it follows from \eqref{eq:Hal-1} that
\begin{align}
	\phi(t,x)
	&= \mathbb{E}_{t,x} \bigg[ \int_t^{T\wedge \tau_n} e^{-\beta (s-t)}\int_{\mathcal{A}}  \big(r(s, \tilde X^{\bpi}_{s-}, a) - \theta\log{\bpi}(a|s, \tilde X^{\pi}_{s-})\big)\bpi(a|s,\tilde X^{\bpi}_{s-})dads \\
	&\phantom{\mathbb{E}_{t,x} \bigg[\bigg]~~} + e^{-\beta (T\wedge\tau_n-t)}\phi( \tilde X^{\bpi}_{T\wedge\tau_n})\bigg]. \label{eq:prelimit}
\end{align}
It follows from Assumption \ref{assump:SDE}-(iv), Definition \ref{def:policy}-(iii), and Eq.~\eqref{eq:VF-PolyG} that the term on the right-hand side of \eqref{eq:prelimit} is dominated by $C(1 + \sup_{t \le s \le T} |\tilde X_{s}^{\bpi}|^p)$ for some $p\geq 2$, which has finite expectation from the moment estimate \eqref{eq:moment-expl}.
Therefore, sending $n$ to infinity in \eqref{eq:prelimit} and applying the dominated convergence theorem, we obtain $\phi(t,x)=J(t, x, \bpi)$.\end{proof}

\begin{proof}[\bf Proof of Lemma~\ref{lem:expHJB}]
We first maximize the integral in (\ref{eq:HJB-explore}). Applying \cite[Lemma 13]{jia2022q}, we deduce that  $\bpi^*$ given by \eqref{eq:opt-pol} is the unique maximizer. Next, we  show that $\psi(t,x)$ is the optimal value function.
	
On one hand, given any admissible stochastic policy $\bpi\in\bPi$, from \eqref{eq:HJB-explore} we have
\begin{align}
	\partial_t\psi(t, x) + \int_{\mathcal{A}} \big( H(t,x, a, \partial_x\psi,  \partial_{x}^2\psi, \psi)   - \theta \log \bpi(a|t,x)\big)\bpi(a|t,x)da - \beta\psi(t, x) \leq 0.
\end{align}
Using similar arguments as in the proof of Lemma~\ref{lem:PK}, we obtain $J(t, x, \boldsymbol{\pi}) \leq \psi(t,x)$ for any $\bpi\in\bPi$. Thus, $J^*(t, x) \leq \psi(t,x)$.

On the other hand, Eq.~\eqref{eq:HJB-explore} becomes
\begin{align} \label{eq:HJB2}
	\partial_t\psi(t, x) +   \int_{\mathcal{A}} \big( H(t,x, a, \partial_x\psi,  \partial_x^2\psi, \psi)  - \theta \log \boldsymbol{\pi}^*(a)\big) \boldsymbol{\pi}^*(a) da  - \beta  \psi(t, x) &= 0,
\end{align}
with $\psi (T, x) = h (x).$
By Lemma~\ref{lem:PK}, we obtain that $J(t, x, \boldsymbol{\pi}^*) = \psi(t,x)$. It follows that $J^*(t, x) \geq \psi(t,x)$.

Combining these results, we conclude that $J^*(t, x)=\psi(t,x)$ and $\bpi^*$ is the optimal stochastic policy.
\end{proof}

\begin{proof}[\bf Proof of Proposition \ref{prop:SDE-sampled}]
	In this proof, $C_p$ is a generic positive constant that depends on $p$, which may take different values each time it occurs.
	Consider the grid $\mathbb{S}= \{t = t_0 < t_1 < \ldots < t_N = T\}$. We prove the well-posedness by induction. Suppose for $i=0,\cdots,n$, the SDE \eqref{eq:samplestate} admits a unique strong solution on $[t_0,t_i]$ satisfying $\mathbb{E}_{t,x}[\sup_{t \le s\le t_i}|X_s^{\bpi, \mathbb{S}}|^p]<\infty$ for any $p\geq 2$. Now consider the equation on $(t_n,t_{n+1}]$. Note that the random action $a_{t_n}^{\bpi,\mathbb{S}}\coloneqq G^{\bpi}(t_n, X^{\bpi, \mathbb{S}}_{t_n}, U_{n+1})$ is applied throughout $(t_n,t_{n+1}]$. For any $p\geq 2$, we have
	\begin{align}
		\mathbb{E}_{t,x}\big[|a_{t_n}^{\bpi,\mathbb{S}}|^p\big] &= \mathbb{E}_{t,x}\big[\mathbb{E}_{t,x}\big[|G^{\bpi}(t_n, X^{\bpi, \mathbb{S}}_{t_n}, U_{n+1})|^p|X^{\bpi, \mathbb{S}}_{t_n}\big]\big]\\
		&=\mathbb{E}_{t,x}\Big[\int_{\mathcal{A}}|a|^p\bpi(a | t_n, X^{\bpi, \mathbb{S}}_{t_n})da\Big]\\
		&\leq C_p(1+\mathbb{E}_{t,x}[|X^{\bpi, \mathbb{S}}_{t_n}|^p]), \quad (\text{by Definition \ref{def:policy}-(iii)})
		\label{eq:action-moment-grid}
	\end{align}
	which is finite. This result together with Assumption \ref{assump:SDE} is sufficient for us to use the arguments in \cite{kunita2004stochastic} and \cite{rong2006theory} to show the existence and uniqueness of the strong solution on $(t_n,t_{n+1}]$ that satisfies $\mathbb{E}_{t,x}[\sup_{t_{n}\le s\le t_{n+1}}|X_s^{\bpi, \mathbb{S}}|^p]<\infty$ for any $p\geq 2$.
	
	Recall the action process $a^{\bpi,\mathbb{S}}$ satisfies $a^{\bpi,\mathbb{S}}_{\tau}=a_{t_n}^{\bpi,\mathbb{S}}$ for $\tau\in (t_n,t_{n+1}]$. The SDE can be written as
	\begin{equation}
	 X^{\bpi, \mathbb{S}}_s = x + \int_t^s b(\tau, X^{\bpi, \mathbb{S}}_{\tau-}, a^{\bpi,\mathbb{S}}_{\tau}) d\tau + \int_t^s\sigma (\tau, X^{\bpi, \mathbb{S}}_{\tau-}, a^{\bpi,\mathbb{S}}_{\tau}) dW_{\tau} + \int_t^s\int_{\mathbb{R}^\ell} \gamma(\tau, X^{\bpi, \mathbb{S}}_{\tau-}, a^{\bpi,\mathbb{S}}_{\tau}, z) \widetilde N(d\tau, dz).
 	\end{equation}
	Using \cite[Theorem 2.11]{kunita2004stochastic}, for any $p\geq 2$ we obtain
	\begin{align}
		&\phantom{\leq~}\mathbb{E}_{t,x}\Big[\sup_{t\le \tau\le s}|X_{\tau}^{\bpi, \mathbb{S}}|^p\Big]\\
		&\leq C_p\Big(|x|^p + \mathbb{E}\Big[\Big(\int_t^s |b(\tau, X^{\bpi, \mathbb{S}}_{\tau-}, a^{\bpi,\mathbb{S}}_{\tau})| d\tau\Big)^p\Big] + \mathbb{E}\Big[\Big(\int_t^s |\sigma(\tau, X^{\bpi, \mathbb{S}}_{\tau-}, a^{\bpi,\mathbb{S}}_{\tau})|^2 d\tau\Big)^{p/2}\Big]\\
		&\phantom{\leq~} +  \mathbb{E}\Big[\Big(\int_t^s \int_{\mathbb{R}^\ell} |\gamma(\tau, X^{\bpi, \mathbb{S}}_{\tau-}, a^{\bpi,\mathbb{S}}_{\tau}, z)|^2 \nu(dz) d\tau\Big)^{p/2}\Big] + \mathbb{E}\Big[\int_t^s \int_{\mathbb{R}^\ell} |\gamma(\tau, X^{\bpi, \mathbb{S}}_{\tau-}, a^{\bpi,\mathbb{S}}_{\tau}, z)|^p \nu(dz) d\tau\Big]\Big).
	\end{align}
	Then
	\begin{align}
		\mathbb{E}\Big[\Big(\int_t^s |b(\tau, X^{\bpi, \mathbb{S}}_{\tau-}, a^{\bpi,\mathbb{S}}_{\tau}| d\tau\Big)^p\Big]&\leq C_p \mathbb{E}\Big[\int_t^s |b(\tau, X^{\bpi, \mathbb{S}}_{\tau-}, a^{\bpi,\mathbb{S}}_{\tau})|^p d\tau\Big]\\
		 &\leq C_p \int_t^s \mathbb{E}\Big[|b(\tau, X^{\bpi, \mathbb{S}}_{\tau-}, a^{\bpi,\mathbb{S}}_{\tau})|^p \Big]d\tau\\
		 &\leq C_p \int_t^s \big(1+ \mathbb{E}\big[|X^{\bpi, \mathbb{S}}_{\tau-}|^p\big] + \mathbb{E}\big[|X^{\bpi, \mathbb{S}}_{\delta(\tau)}|^p\big]\big) d\tau\\
		 &\leq C_p \int_t^s\Big(1+ 2\mathbb{E}\Big[\sup_{t\le u\le \tau}|X_{u}^{\bpi, \mathbb{S}}|^p\Big]\Big)d\tau,
	\end{align}
	where $\delta(\tau)\coloneqq t_n$ for $\tau\in (t_n,t_{n+1}]$, and we have used the linear growth property in Assumption \ref{assump:SDE}-(iii) and \eqref{eq:action-moment-grid}. We can also obtain similar estimates for the other three terms. Together, we have
	\begin{equation}
		\mathbb{E}_{t,x}\Big[\sup_{t\le \tau\le s}|X_{\tau}^{\bpi, \mathbb{S}}|^p\Big]\leq C_p(1+|x|^p) + C_p\int_t^s\mathbb{E}\Big[\sup_{t\le u\le \tau}|X_{u}^{\bpi, \mathbb{S}}|^p\Big]d\tau.
	\end{equation}
	The estimate \eqref{eq:moment-grid} then follows from Gronwall's inequality. \color{black}
	\end{proof}


\begin{proof}[\bf Proof of Theorem~\ref{thm:mart1}]
\rev{In this proof, $C$ is a generic constant whose values may vary from line to line. Consider part (i). Applying It\^{o}'s formula to the value function \eqref{eq:J2} of policy $\boldsymbol{\pi}$ over the grid sample state process \eqref{eq:samplestate}
and
using the definition of q-function, we obtain that for $0 \le t < s \le T$,
\begin{align}
	& e^{-\beta s}  J(s,  X_{s}^{\boldsymbol{\pi}, \mathbb{S}}; \boldsymbol{\pi}) - e^{-\beta t}   J(t,  x; \boldsymbol{\pi}) +  \int_t^s e^{-\beta \tau}  \big( r(\tau, X_{\tau-}^{\boldsymbol{\pi}, \mathbb{S}}, a_\tau^{\boldsymbol{\pi}, \mathbb{S}}) -  \hat q(\tau,  X_{\tau-}^{\boldsymbol{\pi}, \mathbb{S}}, a_\tau^{\boldsymbol{\pi}, \mathbb{S}}) \big)d\tau \nonumber  \\
	& =    \int_t^s e^{-\beta \tau}   \big(  q(\tau,  X_{\tau-}^{\boldsymbol{\pi}, \mathbb{S}}, a_\tau^{\boldsymbol{\pi}, \mathbb{S}}; \boldsymbol{\pi}) -  \hat  q(\tau,  X_{\tau-}^{\boldsymbol{\pi}, \mathbb{S}}, a_\tau^{\boldsymbol{\pi}, \mathbb{S}}) \big)  d\tau +   \int_t^s  e^{-\beta \tau} \partial_xJ(\tau,  X_{\tau-}^{\boldsymbol{\pi}, \mathbb{S}};   \boldsymbol{\pi} ) \circ {\sigma}(\tau,  X^{\boldsymbol{\pi}, \mathbb{S} }_{\tau-}, a^{\boldsymbol{\pi}, \mathbb{S} }_\tau ) dW_\tau \nonumber  \\
	& \phantom{=~} +  \sum_{k=1}^\ell \int_t^s e^{-\beta \tau} \int_{\mathbb{R}} \big( J (\tau,   X^{\boldsymbol{\pi}, \mathbb{S}}_{\tau-} + {\gamma}_k( \tau,X^{\boldsymbol{\pi}, \mathbb{S} }_{\tau-} , a_\tau^{\boldsymbol{\pi}, \mathbb{S}},z)) -  J(\tau,X^{\boldsymbol{\pi}, \mathbb{S}}_{\tau-}; \boldsymbol{\pi})  \big)  \widetilde{ N}_k(d\tau, dz). \label{eq:J-q}
\end{align}
Suppose $\hat q(t,x, a) = q(t,x, a; \boldsymbol{\pi})$ for all $(t,x,a)$. Then the first term on the right-hand side of \eqref{eq:J-q} is zero. We verify the following two conditions:
\begin{align}
		& \mathbb{E}_{t,x}  \bigg[ \int_t^T  e^{-2\beta \tau}  |  J_x(\tau,  X_{\tau-}^{\boldsymbol{\pi}, \mathbb{S}};   \boldsymbol{\pi} )\circ{\sigma}(\tau,  X^{\boldsymbol{\pi}, \mathbb{S} }_{\tau-}, a^{\boldsymbol{\pi}, \mathbb{S} }_\tau ) |^2 d\tau \bigg] < \infty, \label{eq:M1-W} \\
		& \mathbb{E}_{t,x} \bigg[ \int_t^T  e^{-2\beta \tau}  \int_{\mathbb{R}} | J (\tau, X^{\boldsymbol{\pi}, \mathbb{S}}_{\tau-} + {\gamma}_k( \tau, X^{\boldsymbol{\pi}, \mathbb{S}}_{\tau-} , a_\tau^{\boldsymbol{\pi}, \mathbb{S}}, z) ;   \boldsymbol{\pi}) -  J(\tau, X^{\boldsymbol{\pi}, \mathbb{S}}_{\tau-}; \boldsymbol{\pi})|^2 \nu_k(dz) d\tau \bigg] < \infty. \label{eq:M2-J}
	\end{align}
Equality \eqref{eq:M1-W} follows from Assumption~\ref{assump:SDE}-(iii), the polynomial growth of $\partial_xJ$ in $x$, and the moment estimate \eqref{eq:moment-grid}. For \eqref{eq:M2-J}, we apply the mean-value theorem to turn  the integral to
\begin{equation}
	\int_t^T  e^{-2\beta \tau}  \int_{\mathbb{R}} | \partial_xJ(\tau, X^{\boldsymbol{\pi}, \mathbb{S}}_{\tau-} + \alpha_{\tau}{\gamma}_k( \tau, X^{\boldsymbol{\pi}, \mathbb{S}}_{\tau-} , a_\tau^{\boldsymbol{\pi}, \mathbb{S}}, z) ;   \boldsymbol{\pi})\circ
	{\gamma}_k( \tau, X^{\boldsymbol{\pi}, \mathbb{S}}_{\tau-} , a_\tau^{\boldsymbol{\pi}, \mathbb{S}}, z)|^2 \nu_k(dz) d\tau
\end{equation}
for some $\alpha_{\tau}\in[0,1]$. Using the polynomial growth of $\partial_xJ$ in $x$, we can bound the integral by
\begin{align}
	&\phantom{\leq~}\int_t^T  e^{-2\beta \tau}  \int_{\mathbb{R}} |\partial_xJ(\tau, X^{\boldsymbol{\pi}, \mathbb{S}}_{\tau-} + \alpha_{\tau}{\gamma}_k( \tau, X^{\boldsymbol{\pi}, \mathbb{S}}_{\tau-} , a_\tau^{\boldsymbol{\pi}, \mathbb{S}}, z) ;   \boldsymbol{\pi})|^2\cdot
	|{\gamma}_k( \tau, X^{\boldsymbol{\pi}, \mathbb{S}}_{\tau-} , a_\tau^{\boldsymbol{\pi}, \mathbb{S}}, z)|^2 \nu_k(dz) d\tau\\
	&\leq C\int_t^T  e^{-2\beta \tau}  \int_{\mathbb{R}} (1+ |X^{\boldsymbol{\pi}, \mathbb{S}}_{\tau-} + \alpha_{\tau}{\gamma}_k( \tau, X^{\boldsymbol{\pi}, \mathbb{S}}_{\tau-} , a_\tau^{\boldsymbol{\pi}, \mathbb{S}}, z)|^p)^2\cdot
	|{\gamma}_k( \tau, X^{\boldsymbol{\pi}, \mathbb{S}}_{\tau-} , a_\tau^{\boldsymbol{\pi}, \mathbb{S}}, z)|^2 \nu_k(dz) d\tau\\
	&\leq C\int_t^T  e^{-2\beta \tau}  \int_{\mathbb{R}} (1+ |X^{\boldsymbol{\pi}, \mathbb{S}}_{\tau-}|^p + |{\gamma}_k( \tau, X^{\boldsymbol{\pi}, \mathbb{S}}_{\tau-} , a_\tau^{\boldsymbol{\pi}, \mathbb{S}}, z)|^p)^2\cdot
	|{\gamma}_k( \tau, X^{\boldsymbol{\pi}, \mathbb{S}}_{\tau-} , a_\tau^{\boldsymbol{\pi}, \mathbb{S}}, z)|^2 \nu_k(dz) d\tau\\
	&\leq C\int_t^T (1+|X^{\boldsymbol{\pi}, \mathbb{S}}_{\tau-}|^p)^2\int_{\mathbb{R}}|{\gamma}_k( \tau, X^{\boldsymbol{\pi}, \mathbb{S}}_{\tau-} , a_\tau^{\boldsymbol{\pi}, \mathbb{S}}, z)|^2 \nu_k(dz)d\tau\\
	&\phantom{\leq~} + C\int_t^T\int_{\mathbb{R}}|{\gamma}_k( \tau, X^{\boldsymbol{\pi}, \mathbb{S}}_{\tau-} , a_\tau^{\boldsymbol{\pi}, \mathbb{S}}, z)|^{2p+2} \nu_k(dz)d\tau \\
	&\phantom{\leq~} + 2C \int_t^T(1+ |X^{\boldsymbol{\pi}, \mathbb{S}}_{\tau-}|^p)\int_{\mathbb{R}}|{\gamma}_k( \tau, X^{\boldsymbol{\pi}, \mathbb{S}}_{\tau-} , a_\tau^{\boldsymbol{\pi}, \mathbb{S}}, z)|^{p+2} \nu_k(dz)d\tau.
\end{align}
Using Assumption~\ref{assump:SDE}-(iii) and the moment estimate \eqref{eq:moment-grid}, we obtain \eqref{eq:M2-J}.  \rev{It follows that the second and third processes on the right-hand side of \eqref{eq:J-q} are $\mathcal{G}^{\mathbb{S}}$-martingales.}
Thus, we have the martingale property of the process given by \eqref{eq:mart1}.
}

\rev{
Conversely, suppose \eqref{eq:mart1} is a $\mathcal{G}^{\mathbb{S}}$-martingale for any $(t,x)$ and any grid $\mathbb{S}$ of $[t,T]$.
We can infer from \eqref{eq:J-q} that the process
\begin{equation}
\int_t^s e^{-\beta \tau}   \big(  q(\tau,  X_{\tau-}^{\boldsymbol{\pi}, \mathbb{S}}, a_\tau^{\boldsymbol{\pi}, \mathbb{S}}; \boldsymbol{\pi}) -  \hat  q(\tau,  X_{\tau-}^{\boldsymbol{\pi}, \mathbb{S}}, a_\tau^{\boldsymbol{\pi}, \mathbb{S}}) \big)  d\tau	
\end{equation}
is also a $\mathcal{G}^{\mathbb{S}}$-martingale. Furthermore, it has continuous sample paths and finite variation and thus is zero almost surely. That is, for any $t<s,$
\begin{equation}\label{eq:f-zero}
\int_t^s e^{-\beta \tau}   f(\tau,  X_{\tau-}^{\boldsymbol{\pi}, \mathbb{S}}, a_\tau^{\boldsymbol{\pi}, \mathbb{S}})  d\tau =0, \quad \text{almost surely,}	
\end{equation}
where $f(t,x,a) := q(t,x,a; \pi) - \hat q(t,x,a)$, and it is a continuous function by our assumption.
We need to show that $f(t,x,a)=0$ for every $(t,x,a)$, and we prove this by contradiction. Suppose that there exists some time-state action $(t^*, x^*, a^*)$ such that $f(t^*, x^*, a^*)$ is nonzero. Without loss of generality, we suppose $f(t^*, x^*, a^*)> \epsilon$ for some $\epsilon>0.$ Then there exists some $\delta^*>0$ such that $f(t, x, a)> \epsilon/2$ for
all $(u, x', a')$ with $ |u - t^*| \vee | x'-x^*| \vee |a' -a^*| \le \delta^*$.
}

\rev{
Fix $\delta \in (0, \delta^*)$ and choose a grid $\mathbb{S}$ of $[t^*, T]$ such that $[t^*, t^* + \delta]$ is within the first subinterval of the grid.
Consider the grid sample state process $X^{\boldsymbol{\pi}, \mathbb{S}}$ starting from $(t^*, x^*)$, with $X^{\boldsymbol{\pi}, \mathbb{S}}_{t^*} = x^*$. By the definition of the action process, the action at time $s \in [t^*, t^* + \delta]$ is given by $a_{s}^{\bpi, \mathbb{S}} = a_{t^*}^{\bpi, \mathbb{S}} = G^{\bpi}(t^*, x^*, U) $, where $U$ is an uniform random vector on $[0,1]^n$.
Define
\begin{align}
	T_{\delta, \mathbb{S}} = \inf \{t' \ge t^*:  |X_{t'}^{\boldsymbol{\pi}, \mathbb{S}} - x^*| > \delta\}  \wedge (t^* + \delta).
\end{align}
We can see that $T_{\delta, \mathbb{S}} > t^*$, $\mathbb{P}$-almost surely. That is, $T_{\delta, \mathbb{S}}(\omega) > t^*$ for all $\omega \in \Omega \backslash \Omega_0$ where $\mathbb{P}(\Omega_0)=0.$
This is because the L\'evy processes that drive the controlled grid sample state $X^{\bpi, \mathbb{S}}$
are stochastic continuous, i.e., the probability of having a jump at the fixed time $t^*$ is zero.
Now for $\omega \in \Omega \backslash \Omega_0$ and $s \in [t^*, T_{\delta, \mathbb{S}}(\omega)] \subset [t^*, t^* + \delta]$, the action is given by
$a_{s}^{\bpi, \mathbb{S}}(\omega)  = a_{t^*}^{\bpi, \mathbb{S}}(\omega) = G^{\bpi}(t^*, x^*, U(\omega))$ .
By Definition~\ref{def:policy}-(i), if we let $\mathcal{Z} = \{\omega \in \Omega: |a_{t^*}^{\bpi, \mathbb{S}}(\omega) - a^*| \le \delta \}$, then $\mathbb{P}(\mathcal{Z}) >0$.
This implies that for $\omega \in \mathcal{Z} \backslash \Omega_0$ and $\tau \in [t^*, T_{\delta, \mathbb{S}}(\omega)]$, we have $f(\tau,  X_{\tau-}^{\boldsymbol{\pi}, \mathbb{S}}(\omega), a_\tau^{\boldsymbol{\pi}, \mathbb{S}} (\omega)) > \epsilon/2$. 
 This leads to a contradiction to \eqref{eq:f-zero}. }

\rev{ Part (ii) directly follows from applying It\^{o}'s formula to $J(s,  X_{s}^{\boldsymbol{\pi'}, \mathbb{S}}; \boldsymbol{\pi})$ and together with a similar argument  in proving part (i).
For part (iii), again a similar argument in proving part (i) deduces  that
the process
\begin{equation}
\int_t^s e^{-\beta \tau}   \big(   q(\tau,  X_{\tau-}^{\boldsymbol{\pi'}, \mathbb{S}}, a_\tau^{\boldsymbol{\pi'}, \mathbb{S}}; \boldsymbol{\pi}) -  \hat  q(\tau,  X_{\tau-}^{\boldsymbol{\pi'}, \mathbb{S}}, a_\tau^{\boldsymbol{\pi'}, \mathbb{S}}) \big)  d\tau	
\end{equation}
is a $\mathcal{G}^{\mathbb{S}}$-martingale. Then the same argument in part (i) applies to obtain the desired conclusion.
}
\end{proof}


\begin{proof}[\bf Proof of Theorem~\ref{thm:mart2}]

(i) The ``if" part directly follows from the argument in the proof of Theorem~\ref{thm:mart1}. In the following, we prove the ``only if" part.

\rev{Assume for any
	$(t,x) \in [0, T] \times \mathbb{R}^d $ and any grid $\mathbb{S}$ of $[t,T]$, the following process
	\begin{align}
		e^{-\beta s} \hat J(s, X_s^{ \boldsymbol{\pi} , \mathbb{S}}) +      \int_t^s e^{-\beta \tau } \big( r( \tau, X_{\tau}^{\boldsymbol{\pi} , \mathbb{S}}, a_{\tau}^{\boldsymbol{\pi},  \mathbb{S}}) - \hat q ( \tau, X_{\tau}^{\boldsymbol{\pi}, \mathbb{S}}, a_{\tau}^{\boldsymbol{\pi, \mathbb{S}}}) \big) d\tau
	\end{align}
	is a $\mathcal{G}^{\mathbb{S}}$-martingale. Applying It\^{o}'s formula to $\hat J$ over the grid sample state process $(X_s^{ \boldsymbol{\pi}, \mathbb{S}})$  similarly as in the proof of part (i) of Theorem~\ref{thm:mart1}, and letting $g(t,x, a):= \mathcal{L}^a \hat J(t,x) - \beta \hat J(t,x)$ where the operator $\mathcal{L}^a$ is given in \eqref{eq:gene1}, we obtain
\begin{align}
e^{-\beta s} \hat J(s, X_s^{ \boldsymbol{\pi}, \mathbb{S}}) -    \int_t^s e^{-\beta \tau}   g(\tau, X_{\tau}^{\boldsymbol{\pi}, \mathbb{S}}, a_{\tau}^{\boldsymbol{\pi}, \mathbb{S}}) d\tau
\end{align}
is a $\mathcal{G}^{\mathbb{S}}$-martingale. Therefore,
\begin{align}
 \int_t^s e^{-\beta \tau}  (r- \hat q + g)(\tau, X_{\tau}^{\boldsymbol{\pi}, \mathbb{S}}, a_{\tau}^{\boldsymbol{\pi}, \mathbb{S}}) d\tau
\end{align}
is a $\mathcal{G}^{\mathbb{S}}$-martingale, which further implies that $(r- \hat q + g)(t,x,a)=0$ for all $(t,x,a)$ by a similar argument as in the proof of Theorem~\ref{thm:mart1}. Hence we have $$\hat q(t,x,a) = r(t,x,a) + \mathcal{L}^a \hat J(t,x) - \beta \hat J(t,x) = \partial_t\hat J(t,x) + H(t,x, a, \partial_x{\hat J},  \partial_x^2{\hat J}, {\hat J}) - \beta \hat J(t,x).$$ The constraint for $\hat q$ in \eqref{eq:constraints} then implies that
\begin{align}
	\int_{\mathcal{A}} \left\{ \partial_t\hat J(t,x) + H(t,x, a, \partial_x{\hat J},  \partial_x^2{\hat J}, {\hat J}) - \beta \hat J(t,x) - \theta \log \boldsymbol{\pi}(a| t,x)\right\} \boldsymbol{\pi}(a|t,x) da =0,
\end{align}
for all $(t,x) \in [0, T] \times \mathbb{R}^d.$ Because $\hat J(T,x) =h(x)$, by Lemma~\ref{lem:PK}, we deduce that $\hat J(t,x) = J(t,x; \boldsymbol{\pi}) $ for all $(t,x)$. Then it follows from Theorem~\ref{thm:mart1}-(iii) that
$\hat q(t,x, a) = q(t,x, a; \boldsymbol{\pi}) $ for all $(t,x,a)$.
}

\rev{Part (ii) follows directly from part (ii) of Theorem~\ref{thm:mart1}. Part (iii) follows from a similar argument as in the proof of part (i), by applying It\^{o}'s formula to $\hat J$ over the grid sample state process $(X_s^{ \boldsymbol{\pi}', \mathbb{S}})$ under $\boldsymbol{\pi}'$. Finally,
if $\boldsymbol{\pi}(a|t,x) = \frac{\exp(\frac{1}{\theta} \hat q(t,x,a))} {\int_{\mathcal{A} } {\exp(\frac{1}{\theta} \hat q(t,x,a)) da}}$, then it follows from Theorem~\ref{thm:PI} that $\boldsymbol{\pi}$ is the optimal policy and $\hat J$ is the optimal value function.}
\end{proof}


\begin{proof}[\bf Proof of Theorem~\ref{thm:opt-mart}]
\rev{Part (i) directly follows from Part (ii) of Theorem~\ref{thm:mart2}. The proof of part (ii) follows the same argument as in the proof of part (ii) of Theorem 9 in \cite{jia2025erratum}.}
\end{proof}

\begin{proof}[\bf Proof of Lemma~\ref{lem:pde-MV-hedge}]
	(1) Consider the function $h'(t,S)$ defined by the RHS of \eqref{eq:h-rep}. Proposition 12.1 in \cite{tankov2004financial} shows that $h'\in C^{1,2}([0,T)\times\mathbb{R}_{+})\cap C([0,T]\times\mathbb{R}_{+})$ and it satisfies the PIDE \eqref{eq:h-pde-MV-hedge}. Furthermore, $h'(t,S)$ is Lipschitz continuous in $S$. The Lipschitz continuity of function $\hat{G}$ also implies $|\hat{G}(\hat{S}_T)|\leq C(1+\hat{S}_T)$ for some constant $C>0$. It follows that
	\begin{equation}
		|h'(t,S)|\leq C(1+ \mathbb{E}^{\mathbb{Q}}[\hat{S}_T|\hat{S}_t=S]) \leq C(1+S),
	\end{equation}
	where we used the martingale property of $\hat{S}$ under $\mathbb{Q}$. The Feynman-Kac Theorem (see \cite{zhu2015feynman}) implies uniqueness of classical solutions satisfying the linear growth condition.
	
	(2) To study the PIDE \eqref{eq:g-pde-MV-hedge}, we consider the function $g'(t,S)$ defined by the RHS of \eqref{eq:g-rep}.
	Under the assumption of Lemma \ref{lem:pde-MV-hedge}, the arguments of Proposition 12.1 in \cite{tankov2004financial} can be used to show that $g'\in C^{1,2}([0,T)\times\mathbb{R}_{+})\cap C([0,T]\times\mathbb{R}_{+})$ and it satisfies the PIDE \eqref{eq:g-pde-MV-hedge}. The Lipschits continuity of $h(t,S)$ in $S$ implies that $\partial_Sh(t,S)$ is bounded and $|h(t,S\exp(z))-h(t,S)|\leq C(\exp(z)-1)$. We also have $\mathbb{E}[\sup_{t\leq s\leq T}\hat{S}_s^2|\hat{S}_t=S]\leq C(1+S^2)$ \cite[Theorem 3.2]{kunita2004stochastic}. Combining these results, we see that $g'(t,S)$ shows quadratic growth in $S$. Again the Feynman-Kac Theorem implies uniqueness of classical solutions satisfying the quadratic growth condition.
\end{proof}


\section{Convergence of value functions with grid sample state processes}
In this section, we show that, under appropriate conditions, the value functions with grid sample state processes converge to the value function of the exploratory state process, and the convergence rate is of the same order in the mesh size of the time grid.

Consider the time grid $\mathbb{S}= \{0 = t_0 < t_1 < \ldots < t_N = T\}$ with its mesh size $|\mathbb{S}|:=\max_{1\leq n\leq N}|t_{n}-t_{n-1}|$. Define
\begin{equation}
	g(t,x,a):=r(t,x,a) - \theta\log\bpi(a|t,x),\ \ \tilde{g}(t,x):=\int_{\mathcal{A}}g(t,x,a)\bpi(a|t,x)da.
\end{equation}
At $t=0$, the value function associated with  the grid sample state process following SDE \eqref{eq:samplestate} is
\begin{equation}
J^{\bpi,\mathbb{S}}(0,x) \coloneqq \mathbb{E}_{0,x}\left[  \int_0^T  e^{-\beta t} g(t, X_{t-}^{\bpi, \mathbb{S}}, a_{t}^{\bpi, \mathbb{S}}) dt +  e^{-\beta T } h(X_T^{\bpi, \mathbb{S}} )\right].
\end{equation}
For the exploratory state process following SDE \eqref{eq:SDE-expl}, the corresponding (initial) value function is
\begin{equation}
	J^{\bpi}(0,x) \coloneqq \mathbb{E}_{0,x}\left[  \int_0^T  e^{-\beta t} \tilde{g}(t, \tilde{X}_{t-}^{\bpi})dt +  e^{-\beta T } h( \tilde X_T^{\bpi} )\right].
\end{equation}
In the above, $\mathbb{E}_{0,x}$ denotes taking expectation given the initial state is $x$.

For further analysis, we introduce some terminology and function spaces. When we say a function $f(t,x)$ shows polynomial growth in $x$, we mean $\sup_{0\leq t\leq T}|f(t,x)|\leq C(1+|x|^p)$ for some $p\geq 2$. Additionally, a function $f(t,x,a)$ has polynomial growth in $x$ and $a$ if $\sup_{0\leq t\leq T}|f(t,x,a)|\leq C(1+|x|^p+|a|^p)$ for some $p\geq 2$. Let $C^{k}_{p}([0,T]\times\mathbb{R}^d)$ be the space of functions $f: [0,T]\times\mathbb{R}^d\mapsto \mathbb{R}$ such that for nonnegative integer $i$ and multi-index $j$ with $2i+|j|\leq k$, $\partial_t^i\partial_x^j f(t,x)$ is continuous and has polynomial growth in $x$. Let $C^{k,0}_p([0,T]\times\mathbb{R}^d\times\mathcal{A})$ be the space of functions $f: [0,T]\times\mathbb{R}^d\times\mathcal{A}\mapsto \mathbb{R}$ such that for nonnegative integer $i$ and multi-index $j$ with $2i+|j|\leq k$, $\partial_t^i\partial_x^j f(t,x,a)$ is continuous and shows polynomial growth in $x$ and $a$. We impose the following assumption.

\begin{assumption}\label{assumption-converg}
We assume
\begin{itemize}
\item [(i)] $b, \sigma^2\in C^{2,0}_p([0,T]\times\mathbb{R}^d\times\mathcal{A})$;
\item [(ii)] for any $t\in(0,T]$, the PIDE
\begin{align}\label{eq:PIDE}
 \mathcal{L}^{\bpi} \phi(s,x) &= 0, \quad (s,x) \in [0,t) \times \mathbb{R}^d,  \\
\phi(t,x) &= f(x), \quad x \in \mathbb{R}^d,
\end{align}
has a unique solution $\phi \in C^{4}_p([0,t]\times\mathbb{R}^d)$ for $f=h, \tilde{g}(t,\cdot)$, where $\mathcal{L}^{\bpi}$ is given by \eqref{eq:expl-gen-2};
\item[(iii)] for $k=1,\cdots,\ell$ and $p\in \{0,1\}\cup[2,\infty)$, $\int_{\mathbb{R}}\gamma_k(t,x,a,z)^p\varphi(t,x,a,z)\nu_k(dz)$ has polynomial growth in $x$ and $a$ for $\varphi=\partial_t\gamma_k, \partial_x\gamma_k, (\partial_x\gamma_k)^2, \partial_{x}^2\gamma_k$;
\item[(iv)] $g\in C^{2,0}_p([0,T]\times\mathbb{R}^d\times\mathcal{A})$ and $\tilde{g}\in C^{2}_p([0,T]\times\mathbb{R}^d)$.
\end{itemize}
\end{assumption}

In the two applications considered in Sections 5 and 6, conditions (i) and (iii) clearly hold and condition (iv) can be verified for the Gaussian policies studied. Condition (ii) can be easily verified in the MV portfolio selection problem for the Gaussian policies due to the LQ structure of the problem (the solution is quadratic in $x$ and has smooth dependence on $t$). It is more difficult to check this condition in the MV hedging problem, but for the optimal stochastic policy given by \eqref{eq:policy-optimal-MV-hedge} it should be possible to verify the required polynomial growth of the derivatives for calls and puts based on the representations \eqref{eq:h-rep} and \eqref{eq:g-rep} under  additional assumptions on the derivatives of the transition density of the L\'evy process that models the log stock price.

The following result shows the convergence of the value functions with grid sample state processes to that of the exploratory SDE, which extends Theorem 4.1 and Proposition 5.1 of \cite{jia2025accuracy} from diffusions to jump-diffusions.

\begin{theorem}\label{thm:value-converge}
Suppose Assumptions~\ref{assump:SDE}--\ref{assumption-converg} hold. For any admissible stochastic policy $\bpi$, we have
\begin{equation}
	|J^{\bpi,\mathbb{S}}(0,x) - J^{\bpi}(0,x)|\leq C|\mathbb{S}|
\end{equation}
for some positive constant $C$ which may depend on $x$.
\end{theorem}

We need the following lemma for proving the theorem.
\begin{lemma}\label{lem:value-converge-lemma}
	For $\phi \in C^{4}_p([0,t]\times\mathbb{R}^d)$ and $k=1,\cdots,\ell$, let
	\begin{equation}
	I^\phi_k(\tau,x,a)\coloneqq \int_\mathbb{R}\big(\phi(\tau,x+\gamma_k(\tau,x,a,z))-\phi(\tau,x)-\gamma_k(\tau,x,a,z)\circ\partial_x\phi(\tau,x)\big)\nu_k(dz),\quad \tau\in[0,t].	\label{eq:I-phi}
	\end{equation}
Under Assumption \ref{assump:SDE}-(iii) and Assumption \ref{assumption-converg}-(iii), $\partial_\tau I^\phi_k, \partial_xI^\phi_k, \partial_{x}^2I^\phi_k$ have polynomial growth in $x$ and $a$.
\end{lemma}
\begin{proof}[\bf Proof]
Without loss of generality we only consider $d=1$ for notational simplicity. Consider
\begin{align}
	\psi(\tau,x,a)&\coloneqq\int_{\mathbb{R}}\big(\partial_{\tau}\phi(\tau,x+\gamma_k(\tau,x,a,z)) + \partial_x\phi(\tau,x+\gamma_k(\tau,x,a,z))\partial_\tau\gamma_k(\tau,x,a,z)\\
	&\phantom{\int_{\mathbb{R}}\quad~~} - \partial_\tau\phi(\tau,x) - \partial_\tau\gamma_k(\tau,x,a,z)\partial_x\phi(\tau,x) - \gamma_k(\tau,x,a,z)\partial_{\tau}\partial_x\phi(\tau,x)\big)\nu_k(dz)\\
	&=\int_{\mathbb{R}}\big(\partial_{\tau}\partial_x\phi(\tau,x+\alpha\gamma_k(\tau,x,a,z))\gamma_k(\tau,x,a,z)- \partial_{\tau}\partial_x\phi(\tau,x)\gamma_k(\tau,x,a,z) \\
	&\phantom{\int_{\mathbb{R}}\quad~~} + \partial_x^2\phi(\tau,x+\alpha'\gamma_k(\tau,x,a,z))\gamma_k(\tau,x,a,z)\partial_\tau\gamma_k(\tau,x,a,z) \big)\nu_k(dz)
\end{align}
for some $\alpha,\alpha'\in(0,1)$ by the mean-value theorem. Using $\phi \in C^{4}_p([0,t]\times\mathbb{R}^d)$, Assumption \ref{assump:SDE}-(iii) and Assumption \ref{assumption-converg}-(iii), we obtain that $\psi(\tau,x,a)$ has polynomial growth in $x$ and $a$ (and hence has finite values). The interchangeability  of the differentiation and integration thus follows leading to $\partial_\tau I^\phi_k = \psi(\tau,x,a)$.

Using the same argument, we obtain
\begin{align}
	\partial_xI^\phi_k(\tau,x,a)&=\int_{\mathbb{R}}\big(\partial_x\phi(\tau,x+\gamma_k(\tau,x,a,z))(1+\partial_x\gamma_k(\tau,x,a,z)) - \partial_x\phi(\tau,x)\\
	&\phantom{\int_{\mathbb{R}}\quad~~}  - \partial_x\gamma_k(\tau,x,a,z)\partial_x\phi(\tau,x) - \gamma_k(\tau,x,a,z)\partial_x^2\phi(\tau,x)\big)\nu_k(dz),\\
	\partial_x^2I^\phi_k(\tau,x,a)&=\int_{\mathbb{R}}\big(\partial_x^2\phi(\tau,x+\gamma_k(\tau,x,a,z))(1+\partial_x\gamma_k(\tau,x,a,z))^2 - \partial_x^2\phi(\tau,x)\\
	&\phantom{\int_{\mathbb{R}}\quad~~} + \partial_x\phi(\tau,x+\gamma_k(\tau,x,a,z))\partial_x^2\gamma_k(\tau,x,a,z) - \partial_x\phi(\tau,x)\partial_x^2\gamma_k(\tau,x,a,z) \\
	&\phantom{\int_{\mathbb{R}}\quad~~}-2\partial_x\gamma_k(\tau,x,a,z)\partial_x^2\phi(\tau,x)-\gamma_k(\tau,x,a,z)\partial_{x}^3\phi(\tau,x)\big)\nu_k(dz),
\end{align}
which show polynomial growth in $x$ and $a$.
\end{proof}

\begin{proof}[\bf Proof of Theorem~\ref{thm:value-converge}]
In this proof, $C$ is a generic constant independent of $n$, whose values may vary from line to line. We consider $\beta=0$ for simplicity, but the proof can be easily adapted to accommodate discounting.

\noindent(1) Suppose $t=t_i\in\mathbb{S}$. We first show that for $f=h, \tilde{g}(t,\cdot)$,
\begin{equation} \label{eq:hT-limit}
|\mathbb{E}_{0,x}[f(X_{t}^{\bpi, \mathbb{S}})] - \mathbb{E}_{0,x}[f(\tilde{X}_{t}^{\bpi} )]|\leq C|\mathbb{S}|.
\end{equation}

To this end, consider the PIDE \eqref{eq:PIDE}. By the Feynman--Kac formula, we have $\phi(0, x) = \mathbb{E}_{0,x}[f(\tilde X_t^{\bpi})]$ . We also have
\begin{align}\label{eq:terminal-gap}
\mathbb{E}_{0,x}[f( X_t^{\bpi,\mathbb{S}})] - \mathbb{E}_{0,x}[f(\tilde X_T^{\bpi})]
&= \mathbb{E}_{0,x}[\phi(t, X_t^{\bpi, \mathbb{S}})] - \phi(0, x) \\
& =  \sum_{j=1}^{i} \mathbb{E}_{0,x}[\phi(t_j, X_{t_j}^{\bpi, \mathbb{S}}) - \phi(t_{j-1}, X_{t_{j-1}}^{\bpi, \mathbb{S}})].
\end{align}
It\^o's formula yileds
\begin{align*}
&\phi(t_j, X_{t_j}^{\bpi,\mathbb{S}}) - \phi(t_{j-1}, X_{t_{j-1}}^{\bpi,\mathbb{S}}) \\
&= \int_{t_{j-1}}^{t_j}\mathcal{L}^{a_j} \phi(\tau, X_{\tau-}^{\bpi,\mathbb{S}})d\tau +   \int_{t_{j-1}}^{t_j}\partial_x\phi(\tau, X_{\tau-}^{\bpi, \mathbb{S}}) \circ {\sigma}(\tau,  X^{\bpi, \mathbb{S}}_{\tau-}, a_j) dW_\tau \\
& \quad +  \sum_{k=1}^\ell \int_{t_{j-1}}^{t_j}  \int_{\mathbb{R}} \big( \phi(\tau, X^{\bpi, \mathbb{S}}_{\tau-} + \gamma_k(\tau, X^{\bpi, \mathbb{S}}_{\tau-}, a_j,z)) -  \phi(\tau,X^{\bpi, \mathbb{S}}_{\tau-})\big) \widetilde{N}_k(d\tau, dz),
\end{align*}
where $a_j=G^{\bpi}(t_{j-1}, X^{\bpi, \mathbb{S}}_{t_{j-1}}, U_{j})$ and the operator $\mathcal{L}^{a_j}$ is given by \eqref{eq:gene1}. Assumption \ref{assump:SDE} implies that the two stochastic integrals driven by the Brownian motion and compensated Poisson random measure are martingales. It follows that
\begin{align*}
 \mathbb{E}_{0,x}[\phi(t_j, X_{t_j}^{\bpi,\mathbb{S}}) - \phi(t_{j-1}, X_{t_{j-1}}^{\bpi, \mathbb{S}})] = \mathbb{E}_{0,x}\Big[ \int_{t_{j-1}}^{t_j}\mathcal{L}^{a_j}\phi(\tau, X_{\tau-}^{\bpi, \mathbb{S}})d\tau\Big].
\end{align*}
As $a_j\sim\bpi(\cdot|t_{j-1}, X^{\bpi, \mathbb{S}}_{t_{j-1}})$, we have
\begin{align}
	\mathbb{E}_{0,x}\Big[\int_{t_{j-1}}^{t_j}\mathcal{L}^{a_j}\phi(t_{j-1}, X_{t_{j-1}}^{\bpi, \mathbb{S}})d\tau\Big] &= \mathbb{E}_{0,x}\Big[\mathbb{E}_{0,x}\Big[\int_{t_{j-1}}^{t_j}\mathcal{L}^{a_j}\phi(t_{j-1}, X_{t_{j-1}}^{\bpi, \mathbb{S}})d\tau\Big|X_{t_{j-1}}^{\bpi, \mathbb{S}}\Big]\Big]\\
	&= \mathbb{E}_{0,x}\Big[\int_{t_{j-1}}^{t_j}\mathcal{L}^{\bpi} \phi(t_{j-1}, X_{t_{j-1}}^{\bpi,\mathbb{S}})d\tau\Big]\\
	&=0,
\end{align}
where the last equality follows from \eqref{eq:PIDE}. Hence, we can write
\begin{equation}
 \mathbb{E}_{0,x}[\phi(t_j, X_{t_j}^{\bpi , \mathbb{S}}) - \phi(t_{j-1}, X_{t_{j-1}}^{\bpi, \mathbb{S}})] =  \mathbb{E}_{0,x}[e_j],\  e_j\coloneqq\int_{t_{j-1}}^{t_j}
 \big(\mathcal{L}^{a_j} \phi(\tau , X_{\tau-}^{\bpi, \mathbb{S}}) -  \mathcal{L}^{a_j} \phi(t_{j-1}, X_{t_{j-1}}^{\bpi, \mathbb{S}})\big)d\tau.
\end{equation}
Using the expression of $\mathcal{L}^{a_j}$, we obtain
\begin{align}
\mathbb{E}_{0,x}[e_j]&=\mathbb{E}_{0,x}\Big[\int_{t_{j-1}}^{t_j}\big(\partial_{\tau}\phi(\tau, X_{\tau-}^{\bpi, \mathbb{S}})-\partial_{\tau}\phi(t_{j-1}, X_{t_{j-1}}^{\bpi, \mathbb{S}})\big)d\tau\Big]\\
&\phantom{==}+\mathbb{E}_{0,x}\Big[\int_{t_{j-1}}^{t_j}b(\tau, X_{\tau-}^{\bpi, \mathbb{S}}, a_{j})\circ\big(\partial_{x}\phi(\tau, X_{\tau-}^{\bpi, \mathbb{S}})-\partial_{x}\phi(t_{j-1}, X_{t_{j-1}}^{\bpi, \mathbb{S}})\big)d\tau\Big]\\
&\phantom{==}+\mathbb{E}_{0,x}\Big[\int_{t_{j-1}}^{t_j}\frac{1}{2}\sigma^2(\tau, X_{\tau-}^{\bpi, \mathbb{S}}, a_{j})\circ\big(\partial_x^2\phi(\tau, X_{\tau-}^{\bpi, \mathbb{S}})-\partial_x^2\phi(t_{j-1}, X_{t_{j-1}}^{\bpi, \mathbb{S}})\big)d\tau\Big]\\
&\phantom{==}+\mathbb{E}_{0,x}\Big[\int_{t_{j-1}}^{t_j}\big(b(\tau, X_{\tau-}^{\bpi, \mathbb{S}}, a_{j})-b(t_{j-1}, X_{t_{j-1}}^{\bpi, \mathbb{S}}, a_{j})\big)\circ\partial_{x}\phi(t_{j-1}, X_{t_{j-1}}^{\bpi, \mathbb{S}})d\tau\Big]\\
&\phantom{==}+\mathbb{E}_{0,x}\Big[\int_{t_{j-1}}^{t_j}\frac{1}{2}\big(\sigma^2(\tau, X_{\tau-}^{\bpi, \mathbb{S}}, a_{j})-\sigma^2(t_{j-1}, X_{t_{j-1}}^{\bpi, \mathbb{S}}, a_{j})\big)\circ \partial_x^2\phi(t_{j-1}, X_{t_{j-1}}^{\bpi, \mathbb{S}})d\tau\Big]\\
&\phantom{==}+\sum_{k=1}^\ell\mathbb{E}_{0,x}\Big[\int_{t_{j-1}}^{t_j}\big(I_k^\phi(\tau, X_{\tau-}^{\bpi, \mathbb{S}},a_{j})-I_k^\phi(t_{j-1}, X_{t_{j-1}}^{\bpi, \mathbb{S}},a_{j})\big)d\tau\Big],
\end{align}
where $I_k^{\phi}$ is defined by \eqref{eq:I-phi}. The first five terms can be analyzed using the arguments in \cite{jia2025accuracy}, which show that their sum is bounded by $C(t_j-t_{j-1})^2$. For the last term, we obtain from It\^o's formula that
\begin{align}
	&\phantom{=~}I_k^\phi(\tau, X_{\tau-}^{\bpi, \mathbb{S}},a_{j})-I_k^\phi(t_{j-1}, X_{t_{j-1}}^{\bpi, \mathbb{S}},a_{j})\\
	&=\int_{t_{j-1}}^{\tau}\big(\partial_sI_k^\phi(s, X_{s-}^{\bpi, \mathbb{S}},a_{j}) + b(s, X_{s-}^{\bpi, \mathbb{S}},a_{j})\circ\partial_xI_k^\phi(s, X_{s-}^{\bpi, \mathbb{S}},a_{j}) \big)ds \\
	&\phantom{=~} + \int_{t_{j-1}}^{\tau}\frac{1}{2}\sigma^2(s, X_{s-}^{\bpi, \mathbb{S}},a_{j})\circ\partial_{x}^2I_k^\phi(s, X_{s-}^{\bpi, \mathbb{S}},a_{j}) \big)ds\\
	&\phantom{=~} + \int_{t_{j-1}}^{\tau}\int_\mathbb{R}\big(I_k^\phi(s,X_{s-}^{\bpi, \mathbb{S}}+\gamma_k(s,X_{s-}^{\bpi, \mathbb{S}},a_j,z),a_j)-I_k^\phi(s,X_{s-}^{\bpi, \mathbb{S}},a_j)\\
	&\phantom{=~\int_{t_{j-1}}^{\tau}\int_\mathbb{R}()~} -\gamma_k(s,X_{s-}^{\bpi, \mathbb{S}},a_j,z)\circ\partial_xI_k^\phi(s,X_{s-}^{\bpi, \mathbb{S}},a_j)\big)\nu_k(dz)ds\\
	&\phantom{=~} +  \int_{t_{j-1}}^{\tau}\partial_xI_k^\phi(s, X_{s-}^{\bpi, \mathbb{S}},a_{j}) \circ {\sigma}(s,  X^{\bpi, \mathbb{S}}_{s-}, a_j) dW_s \\
	& \quad +  \sum_{k=1}^\ell \int_{t_{j-1}}^{\tau} \int_{\mathbb{R}}\big(I_k^\phi(s, X^{\bpi, \mathbb{S}}_{s-} + \gamma_k(s, X^{\bpi, \mathbb{S}}_{s-}, a_j, z), a_j)-  I_k^\phi(s, X^{\bpi, \mathbb{S}}_{s-}, a_j)\big) \widetilde{N}_k(ds, dz).
\end{align}
By Lemma \ref{lem:value-converge-lemma} and Proposition \ref{prop:SDE-sampled}, the last two terms are martingales and hence their expectations are zero. By Assumptions \ref{assump:SDE} and \ref{assumption-converg}, the admissibility of $\bpi$, Lemma \ref{lem:value-converge-lemma} and Proposition \ref{prop:SDE-sampled}, the expectations of the first three terms are finite and bounded by $C(\tau-t_{j-1})$. It follows that
\begin{align}
	&\phantom{\leq~}\Big|\mathbb{E}_{0,x}\Big[\int_{t_{j-1}}^{t_j}\big(I_k^\phi(\tau, X_{\tau-}^{\bpi, \mathbb{S}}, a_j)-I_k^\phi(t_{j-1}, X_{t_{j-1}}^{\bpi, \mathbb{S}}, a_j)\big)d\tau\Big]\Big|\\
	&\leq\int_{t_{j-1}}^{t_j}\mathbb{E}_{0,x}\Big[\big|I_k^\phi(\tau, X_{\tau-}^{\bpi, \mathbb{S}}, a_j)-I_k^\phi(t_{j-1}, X_{t_{j-1}}^{\bpi, \mathbb{S}}, a_j)\big|\Big]d\tau\\
	&\leq C\int_{t_{j-1}}^{t_j}\int_{t_{j-1}}^{\tau}(\tau-t_{j-1})dsd\tau\\
	&\leq C(t_j-t_{j-1})^2.
\end{align}
Putting these estimates together, we have $|E[e_j]|\leq C(t_j-t_{j-1})^2$. Finally, we obtain
\begin{align}
	\big|\mathbb{E}_{0,x}[f( X_t^{\bpi,\mathbb{S}})] - \mathbb{E}_{0,x}[f(\tilde X_t^{\bpi})]\big|
	&\leq \sum_{j=1}^{i} \big|\mathbb{E}_{0,x}[e_j]\big|\leq C|\mathbb{S}|.
\end{align}

\noindent (2) We estimate the difference between $J^{\bpi,\mathbb{S}}(0,x)$ and $J^{\bpi}(0,x)$. Note that
\begin{equation}
	J^{\bpi,\mathbb{S}}(0,x) - J^{\bpi}(0,x) = \big(\mathbb{E}_{0,x}[h(X_T^{\bpi,\mathbb{S}})] - \mathbb{E}_{0,x}[h(\tilde{X}_T^{\bpi})]\big) + \sum_{i=1}^N \mathbb{E}\Big[\int_{t_{i-1}}^{t_i} \big(g(\tau, X_\tau^{\bpi, \mathbb{S}}, a_i) - \tilde{g}(\tau, \tilde{X}_{\tau}^{\bpi})\big)d\tau\Big],
\end{equation}
where $a_i=G^{\bpi}(t_{i-1}, X^{\bpi, \mathbb{S}}_{t_{i-1}}, U_{i})$. We have already analyzed the first term in part (1), which is bounded by $C|\mathbb{S}|$. For the second term, we have for $\tau\in(t_{i-1},t_i]$,
\begin{align}
	&\phantom{~~~~}\mathbb{E}_{0,x}[g(\tau, X_\tau^{\bpi, \mathbb{S}}, a_i) - \tilde{g}(\tau, \tilde{X}_{\tau}^{\bpi})]\\
	&= \mathbb{E}_{0,x}[g(\tau, X_\tau^{\bpi, \mathbb{S}}, a_i) - g(t_{i-1}, X_{t_{i-1}}^{\bpi, \mathbb{S}}, a_i)] + \mathbb{E}_{0,x}[g(t_{i-1}, X_{t_{i-1}}^{\bpi, \mathbb{S}}, a_i) - \tilde{g}(\tau, \tilde{X}_{\tau}^{\bpi})]\\
	&= \mathbb{E}_{0,x}[g(\tau, X_\tau^{\bpi, \mathbb{S}}, a_i) - g(t_{i-1}, X_{t_{i-1}}^{\bpi, \mathbb{S}}, a_i)] + \mathbb{E}_{0,x}[\tilde{g}(t_{i-1}, X_{t_{i-1}}^{\bpi, \mathbb{S}}) - \tilde{g}(t_{i-1}, \tilde{X}_{t_{i-1}}^{\bpi})]\\
	&\phantom{=~}+ \mathbb{E}_{0,x}[\tilde{g}(t_{i-1}, \tilde{X}_{t_{i-1}}^{\bpi}) - \tilde{g}(\tau, \tilde{X}_{\tau}^{\bpi})]
\end{align}
where we noted $\mathbb{E}_{0,x}[g(t_{i-1}, X_{t_{i-1}}^{\bpi, \mathbb{S}}, a_i)]=\mathbb{E}_{0,x}[\mathbb{E}_{0,x}[g(t_{i-1}, X_{t_{i-1}}^{\bpi, \mathbb{S}}, a_i)|X_{t_{i-1}}^{\bpi, \mathbb{S}}]]=\mathbb{E}_{0,x}[\tilde{g}(t_{i-1}, X_{t_{i-1}}^{\bpi, \mathbb{S}})]$. As $g\in C^{2,0}_p([0,T]\times\mathbb{R}^d\times\mathcal{A})$, it follows from It\^o's formula, Assumption \ref{assump:SDE} and Proposition \ref{prop:SDE-sampled} that
\begin{equation}
	\big|\mathbb{E}_{0,x}[g(\tau, X_\tau^{\bpi, \mathbb{S}}, a_i) - g(t_{i-1}, X_{t_{i-1}}^{\bpi, \mathbb{S}}, a_i)]\big| \leq C|\mathbb{S}|.
\end{equation}
Part (1) shows
\begin{equation}
	\big|\mathbb{E}_{0,x}[\tilde{g}(t_{i-1}, X_{t_{i-1}}^{\bpi, \mathbb{S}}) - \tilde{g}(t_{i-1}, \tilde{X}_{t_{i-1}}^{\bpi})]\big| \leq C|\mathbb{S}|.
\end{equation}
As $\tilde{g}\in C^{2}_p([0,T]\times\mathbb{R}^d)$, using It\^o's formula, Lemma \ref{lem:expl-diff-coeffs}, Lemma \ref{lem:expl-jump-LG} and Proposition \ref{prop:expl-SDE-wellposed} yields
\begin{equation}
	\big|\mathbb{E}_{0,x}[\tilde{g}(t_{i-1}, \tilde{X}_{t_{i-1}}^{\bpi}) - \tilde{g}(\tau, \tilde{X}_{\tau}^{\bpi})]\big| \leq C|\mathbb{S}|.
\end{equation}
Combining these results, we obtain
\begin{equation}
	\Big|\sum_{i=1}^N \mathbb{E}_{0,x}\Big[\int_{t_{i-1}}^{t_i} \big(g(\tau, X_\tau^{\bpi, \mathbb{S}}, a_i) - \tilde{g}(\tau, \tilde{X}_{\tau}^{\bpi})\big)d\tau\Big]\Big| \leq C|\mathbb{S}|.
\end{equation}
The proof of the theorem is now complete.  \color{black}
\end{proof}

\end{document}